\documentclass{article} % For LaTeX2e
\usepackage{iclr2027_conference,times}

\usepackage{amsmath,amsfonts,bm}

\def\figref#1{figure~\ref{#1}}
\def\Figref#1{Figure~\ref{#1}}

\def\secref#1{section~\ref{#1}}
\def\eqref#1{equation~\ref{#1}}
\def\Eqref#1{Equation~\ref{#1}}
\def\1{\bm{1}}

\DeclareMathAlphabet{\mathsfit}{\encodingdefault}{\sfdefault}{m}{sl}
\SetMathAlphabet{\mathsfit}{bold}{\encodingdefault}{\sfdefault}{bx}{n}

\newcommand{\parents}{Pa} % See usage in notation.tex. Chosen to match Daphne's book.

\usepackage{hyperref}
\usepackage{url}

\usepackage{graphicx}

\usepackage{amsmath,amssymb,amsthm}

\newtheorem{lemma}{Lemma}
\newtheorem{proposition}{Proposition}

\usepackage{booktabs}
\usepackage{placeins}

\title{From Redundancy to Minimality: Fixed-Point-Guided Hierarchical Reduction of Learned Piecewise-Linear Dynamics}

\author{
Hiroto Tamura\textsuperscript{1,3}
\qquad
Gouhei Tanaka\textsuperscript{2,3} \\
\textsuperscript{1}Graduate School of Informatics,
Osaka Metropolitan University, Osaka, Japan \\
\textsuperscript{2}Graduate School of Engineering,
Nagoya Institute of Technology, Nagoya, Japan \\
\textsuperscript{3}International Research Center for Neurointelligence (IRCN),
The University of Tokyo, Tokyo, Japan \\
\texttt{h.tamura@omu.ac.jp}
}

\iclrfinalcopy % Uncomment for camera-ready version, but NOT for submission.
\begin{document}

\maketitle
%% 会議名のヘッダーを消す
\fancyhead{}
\renewcommand{\headrulewidth}{0pt}

\begin{abstract}
Understanding a nonlinear dynamical system from time series requires not only reproducing its trajectories, but also identifying a simple representation that preserves its essential dynamical structure. 
Almost-linear recurrent neural networks (AL-RNNs) are piecewise-linear RNNs in which only a subset of units use ReLU nonlinearities, while the remaining units are linear, so that nonlinear capacity is explicitly
controlled by the number of ReLU units.
Their activation patterns define linear regions, represented as symbols, whose observed transitions form a symbolic transition graph. However, directly training AL-RNNs with few ReLU units to realize minimal dynamical representations can be unreliable.
We ask whether an AL-RNN with more ReLU units can instead be trained first and systematically reduced to a minimal dynamical representation. We introduce a fixed-point-guided hierarchical reduction procedure that progressively linearizes selected ReLU units, thereby merging neighboring linear regions and their corresponding graph nodes while preserving distinct symbols containing fixed points (FPs). 
The resulting reduction tree defines a hierarchy of progressively simpler candidate models.
% Each reduced candidate is initialized from the parent parameters and retrained under guidance from the parent dynamics, providing a practical route from structural reduction to dynamical realization.
Each reduced candidate is initialized from the parent parameters and retrained by gradually changing the selected ReLU activations to linear functions, with guidance from the parent dynamics, providing a practical route from structural reduction to dynamical realization.
We also prove that, under certain assumptions, reproducing $Q$ distinct fixed points requires at least $Q$ FP-containing symbols, providing a certificate of symbol-level minimality when this bound is attained.
% On the 3-scroll Chua system, direct training with the theoretical minimum of three ReLU units achieves high-fidelity minimal realizations in only \(20\%\) of seeds, whereas our learn--reduce--retrain strategy increases the seed-macro success rate to approximately \(71\%\) at the same final
% nonlinear capacity.
We evaluate the framework on attractor-reconstruction tasks and flip-flop memory tasks, chosen to examine minimal dynamical representations involving multiple fixed points.
% We evaluate the framework on three-, four-, and five-scroll Chua systems, the Dadras four-wing system, and two- and three-bit flip-flop memory, chosen to examine minimal dynamical representations involving multiple fixed points. 
Across these tasks, our learn--reduce--retrain strategy substantially improves the recovery of minimal models, including settings where direct training rarely succeeds. 
These results show that redundant nonlinear capacity can serve as a scaffold for discovering and realizing minimal dynamical representations.
\end{abstract}

\section{Introduction}
\label{sec:introduction}

%% ============================================================
%% Paragraph 1: Broad problem / Why parsimony matters
Learning dynamical systems from time-series observations is a central problem in scientific machine learning.
When the governing equations of a system are unknown, recurrent neural networks (RNNs) can serve as flexible generative surrogates that reproduce complex nonlinear dynamics directly from data
\citep{durstewitz2023reconstructing}.
For scientific understanding, however, trajectory reconstruction alone is often insufficient:
many different models may reproduce similar observations while providing very different descriptions of the underlying dynamics.
A complementary goal is therefore to identify a \emph{parsimonious} representation, one that explains the observed dynamics using as few essential dynamical components as possible
\citep{brunton2016discovering,champion2019data}.
Such representations are attractive not only for interpretability,
but also for subsequent dynamical analysis and control, where reduced
structural complexity can be advantageous.

%% ============================================================
%% Paragraph 2: Fixed-point analysis and its limitations
% Fixed points provide a natural starting point for such a mechanistic description.
% Stable fixed points characterize attracting states, while unstable fixed points and saddles can organize transitions and shape trajectories far beyond their immediate neighborhoods
% \citep{sussillo2013opening}.
% Accordingly, fixed-point analysis has become an important tool for reverse-engineering trained RNNs.
% For generic nonlinear RNNs, fixed and slow points are commonly identified by numerically minimizing the residual of the recurrent dynamics
% \citep{sussillo2013opening,golub2018fixedpointfinder,smith2021reverse},
% while more systematic search strategies based on directional fibers have also been proposed
% \citep{katz2018directional}.
% However, locating fixed points does not by itself provide a global dynamical description:
% a collection of fixed points gives no direct correspondence between an arbitrary state and the fixed point around which its dynamics should be linearized, and Jacobian linearizations around fixed points are inherently local, with approximation errors accumulating over multi-step rollouts
% \citep{smith2021reverse}.

Fixed points provide a natural starting point for mechanistic analysis
of learned dynamics and have been widely used to reverse-engineer RNNs
\citep{sussillo2013opening,golub2018fixedpointfinder,smith2021reverse}.
Stable fixed points characterize attracting states, while unstable fixed points and saddles can organize transitions and shape trajectories far beyond their immediate neighborhoods \citep{sussillo2013opening}.
However, fixed-point analysis alone remains local and does not describe transitions between different local dynamics.

%% piecewise-linear models
Piecewise-linear recurrent neural networks (PL-RNNs) avoid these
ambiguities by using piecewise-linear nonlinearities, such as ReLU
units, to explicitly partition the state space into linear regions
\citep{durstewitz2017state,monfared2020existence}.
Each state is unambiguously assigned to a region whose affine dynamics are exact throughout that region, while candidate fixed points can be obtained analytically by solving a linear equation within each region.
Moreover, transitions across region boundaries induce a discrete sequence of visited linear regions, providing a symbolic description of how the learned dynamics evolve globally across the state space.
This combination of region-wise analytical tractability and explicit inter-region transitions makes PL-RNNs attractive for turning time-series reconstruction into an interpretable dynamical representation.

%% ============================================================
%% Paragraph 3: Closest predecessor -- AL-RNN
A key obstacle, however, is that a conventional PL-RNN with many nonlinear units can partition its state space into an enormous number of linear regions, making the resulting symbolic dynamics difficult to interpret.
Almost-linear RNNs (AL-RNNs) address this problem by restricting the ReLU nonlinearity to only \(P\) units while keeping the remaining recurrent state variables linear and the latent-state dimension fixed \citep{brenner2024almostlinear}.
The parameter \(P\) therefore provides explicit control over the nonlinear complexity of the model.
By training AL-RNNs over different values of \(P\), previous work showed that surprisingly few nonlinear units can suffice to reconstruct complex chaotic systems and recover compact piecewise-linear descriptions of their dynamics
\citep{brenner2024almostlinear}.
More recently, AL-RNNs have also been used as a general probe for systematically identifying when recurrence nonlinearity is computationally necessary and what mechanisms it enables across sequence-modeling tasks
\citep{brenner2026nonlinearity}.
Together, these results establish AL-RNNs as a framework not only for reconstructing interpretable piecewise-linear dynamics, but also for controlling and dissecting the functional role of nonlinearity in recurrent computation.
% A compact dynamical model can therefore be sought by varying \(P\) and selecting an operating point from the resulting performance--complexity curve.

%% short ver.
% Almost-linear RNNs (AL-RNNs) address this problem by restricting the
% ReLU nonlinearity to only \(P\) units while keeping the remaining
% recurrent state variables linear \citep{brenner2024almostlinear}.
% Thus, \(P\) explicitly controls nonlinear complexity, and previous work
% showed that surprisingly few nonlinear units can suffice to reconstruct
% complex chaotic systems.

%% ============================================================
%% Paragraph 4: Gap -- structural relation + optimization barrier
Yet controlling the amount of nonlinearity does not by itself reveal how dynamical representations at different levels of complexity are structurally related, nor does it guarantee that a minimally nonlinear model can be reliably learned.
Existing AL-RNN approaches primarily address \emph{how much} nonlinearity is needed by independently training models with different values of \(P\)
\citep{brenner2024almostlinear,brenner2026nonlinearity}.
Consequently, there is no explicit correspondence between the linear regions or transitions of a successful higher-capacity model and those of a simpler one.
More importantly, reducing \(P\) also makes optimization increasingly difficult near the representational capacity boundary: a compact architecture may be theoretically sufficient to represent the target dynamics yet rarely be discovered by direct training.
Thus, simply searching over progressively smaller models conflates two distinct problems: determining the minimal dynamical structure and successfully optimizing a model that realizes it.
Rather than discarding an accurately trained higher-capacity model and repeatedly restarting optimization at smaller \(P\), we ask whether the large model itself can serve as a scaffold from which simpler representations can be identified and subsequently realized.
Specifically, \emph{can a higher-capacity AL-RNN reveal a hierarchy of increasingly compact dynamical representations, and can this hierarchy guide the reliable training of a minimal realization?}

%% ============================================================
%% Paragraph 5: Key insight + theory
We show that fixed points provide a principled constraint for this reduction from redundancy to minimality.
Each visited linear region can be classified according to whether its affine dynamics contain a fixed point; we refer to these two classes as \emph{fixed-point (FP) symbols} and \emph{fixed-point-free (FP-free) symbols}, respectively.
Under standard nondegeneracy conditions, distinct isolated fixed points of the reconstructed system require distinct FP symbols.
Thus, for a system with \(Q\) fixed points, any faithful piecewise-linear representation requires at least \(Q\) FP symbols.
These FP symbols form an irreducible fixed-point skeleton: merging regions associated with distinct fixed points necessarily destroys the fixed-point structure.
FP-free symbols, in contrast, are not protected by this lower bound and therefore constitute candidates for coarse-graining, although some may still be required to reproduce global transitions.
% Consequently, reaching a representation with exactly \(Q\) visited symbols, all of which are FP symbols, certifies minimality in the number of visited linear regions.
% When FP-free symbols remain, the same construction still yields a reduced representation, but the fixed-point lower bound alone is insufficient to certify its minimality.

% Importantly, a reduced graph specifies a candidate low-complexity dynamical structure, but does not by itself establish that the corresponding reduced AL-RNN can reproduce the target dynamics.
% The reduction tree records which ReLU units must be linearized to realize each candidate structure, allowing us to construct the corresponding lower-complexity AL-RNN directly from its higher-capacity parent.
% We then retrain these reduced models from the parent parameters and test whether the graph-derived minimal structures can be realized as faithful dynamical models.
% This turns the higher-capacity AL-RNN into a scaffold not only for identifying hierarchical and minimal symbolic representations, but also for reliably realizing compact models that can be difficult to learn directly.
% % \Figref{fig:overview} summarizes the full pipeline from a higher-capacity AL-RNN through hierarchical graph reduction to reduction-guided realization of compact dynamics.

%% ============================================================
%% Paragraph 6: Hierarchical reduction + reduction-guided retraining

Building on this structure, we develop a fixed-point-guided hierarchical
reduction procedure for trained AL-RNNs.
We represent the learned dynamics as a transition graph whose nodes are
visited ReLU activation patterns and whose edges record transitions
between the corresponding linear regions.
For sufficiently resolved trajectories, transitions between adjacent
regions are associated with individual ReLU switches.
Linearizing one ReLU removes the boundaries induced by that unit and
therefore merges the corresponding regions in the symbolic graph.
We systematically enumerate distinct sequences of such linearizations,
while terminating any branch that would merge distinct FP symbols.
The result is a \emph{reduction tree}: its root is the original
high-resolution transition graph, and its descendants represent
progressively coarser dynamical descriptions.

Importantly, a reduced graph specifies a candidate low-complexity
dynamical structure, but does not by itself establish that the
corresponding reduced AL-RNN can reproduce the target dynamics.
The reduction tree records which ReLU units must be linearized to
realize each candidate structure, allowing us to construct the
corresponding reduced AL-RNN directly from its higher-capacity parent.
We then retrain these models from the parent parameters under guidance from the parent dynamics and test whether the graph-derived structures can be realized as faithful dynamics.
This learn--reduce--retrain procedure turns the higher-capacity
AL-RNN into a scaffold not only for identifying hierarchical and
minimal symbolic representations, but also for realizing compact models that can be difficult to learn directly.

Prior AL-RNN work \citep{brenner2024almostlinear} demonstrated that direct low-capacity training can
robustly recover topologically minimal representations of the Lorenz-63 \citep{lorenz1963deterministic} and R\"ossler \citep{rossler1976equation} systems, both of which involve relatively few fixed points.
% Here, we focus on the more challenging 3-scroll Chua system \citep{tang2001generation}, whose five target fixed points imply lower bounds of five visited symbols and three ReLU units.
% At this theoretical minimum of three ReLU units, direct training produces high-fidelity minimal realizations in only \(20\%\) of seeds.
In this study, we consider tasks involving more fixed points: attractor-reconstruction using 3-, 4-, and 5-scroll Chua systems \citep{tang2001generation} and the Dadras 4-wing system \citep{dadras2009novel}, and 2- and 3-bit flip-flop memory tasks adapted from \citep{sussillo2013opening}, which impose stronger constraints on compact dynamical representations and can make direct low-capacity training more challenging.
In contrast, hierarchical reduction of reliably learned
higher-capacity models frequently reaches both the symbol-level and nonlinear-capacity lower bounds.
% After reduction-guided retraining, the complete pipeline achieves approximately \(71\%\) seed-macro success for high-fidelity minimal realization at the same final nonlinear capacity, a \(3.5\times\) improvement over direct minimal-capacity training.
After reduction-guided retraining, the learn--reduce--retrain strategy improves the recovery of minimal models, including settings in which direct training rarely succeeds. 
These results show that redundant nonlinear capacity can serve as both a structural and an optimization scaffold for discovering and realizing minimal dynamical representations.

Our main contributions are:
\begin{itemize}
    \item We derive a fixed-point lower bound on symbolic complexity:
    reproducing \(Q\) distinct target fixed points requires at least
    \(Q\) FP symbols, yielding a certificate of symbol-level minimality when all target fixed points are preserved and exactly
    this many symbols are visited.

    \item We introduce a fixed-point-guided hierarchical reduction
    procedure that reveals progressively coarser representations while preserving distinct FP symbols.

    % \item We show that higher-capacity AL-RNNs can serve as scaffolds
    % for minimal-model learning: the learn--reduce--retrain pipeline
    % substantially improves the reliability of high-fidelity minimal
    % realization over direct minimal-capacity training.
    \item We show that higher-capacity AL-RNNs can serve as scaffolds for minimal-model learning in both autonomous reconstruction and input-driven memory, including settings in which direct low-capacity training rarely succeeds.
\end{itemize}

% ============================================================
% 2. PRELIMINARIES
% Target: ~0.6 page
% ============================================================

\section{AL-RNNs and Fixed-Point Structure}
\label{sec:preliminaries}

We first review the AL-RNN formulation and its symbolic
representation following \citet{brenner2024almostlinear}, and then
derive the fixed-point constraints that underlie our reduction method.

\subsection{AL-RNNs and symbolic dynamics}
\label{sec:alrnn}

We consider the Almost-Linear Recurrent Neural Network (AL-RNN)
\citep{brenner2024almostlinear},
\begin{equation}
    \mathbf{z}_{t}
    =
    \mathbf{A}\mathbf{z}_{t-1}
    +
    \mathbf{W}\boldsymbol{\phi}_{P}(\mathbf{z}_{t-1})
    +
    \mathbf{h},
    \label{eq:alrnn}
\end{equation}
where
\(\mathbf{z}_t\in\mathbb{R}^{M}\),
\(\operatorname{ReLU}(u):=\max(0,u)\), and
\(\boldsymbol{\phi}_{P}\) applies ReLU nonlinearities only to the last
\(P\) coordinates,
\begin{equation}
    \boldsymbol{\phi}_{P}(\mathbf{z})
    =
    [z_1,\ldots,z_{M-P},
      \operatorname{ReLU}(z_{M-P+1}),\ldots,
      \operatorname{ReLU}(z_M)]^\top .
    \label{eq:alrnn_activation}
\end{equation}
Thus, \(P\) controls the number of nonlinear switching variables while
the latent dimension \(M\) remains fixed.

Following the identity teacher-forcing formulation of \citet{brenner2024almostlinear}, the first \(N\) linear coordinates serve directly as readout units for the \(N\)-dimensional observed system. The model output is therefore
\begin{equation}
\hat{\mathbf{x}}_t = \mathcal{I}\mathbf{z}_t, \qquad \mathcal{I} =
\begin{bmatrix}
\mathbf{I}_N & \mathbf{0} 
\end{bmatrix} \in \mathbb{R}^{N\times M}, \label{eq:readout} 
\end{equation}
so that the readout units are part of the recurrent state rather than a separate output layer. The latent state therefore consists of \(N\) linear readout units, \(M-N-P\) additional linear units, and \(P\) nonlinear ReLU units.

The activation state of the nonlinear units defines a binary symbol
\begin{equation}
    \boldsymbol{\sigma}(\mathbf{z})
    =
    \bigl(
    \mathbb{I}[z_{M-P+1}>0],\ldots,
    \mathbb{I}[z_M>0]
    \bigr)
    \in\{0,1\}^{P},
    \label{eq:activation_pattern}
\end{equation}
which defines a linear region
\(\Omega_{\boldsymbol{\sigma}}\).
Within this region, the dynamics are exactly affine,
\begin{equation}
    \mathbf{z}_{t}
    =
    \mathbf{F}_{\boldsymbol{\sigma}}\mathbf{z}_{t-1}
    +
    \mathbf{h},
    \qquad
    \mathbf{F}_{\boldsymbol{\sigma}}
    =
    \mathbf{A}
    +
    \mathbf{W}\mathbf{D}_{\boldsymbol{\sigma}},
    \label{eq:regional_dynamics}
\end{equation}
where
\[
\mathbf{D}_{\boldsymbol{\sigma}}
=
\operatorname{diag}
(
\underbrace{1,\ldots,1}_{M-P},
\sigma_1,\ldots,\sigma_P
)
\in\mathbb{R}^{M\times M}
\]
is the diagonal activation matrix, with the first \(M-P\) entries fixed
to one for the linear coordinates.

For a generated trajectory of length \(T\), let
\begin{equation}
    \Sigma
    =
    \{
    \boldsymbol{\sigma}(\mathbf{z}_t)
    :1\leq t\leq T
    \}
    \label{eq:visited_symbols} \nonumber
\end{equation}
denote the set of visited linear regions, which we refer to as
\emph{symbols}.
\Figref{fig:fp_reduction_concept}(a) schematically illustrates this
piecewise-linear partition of the latent state space and a trajectory
moving across the resulting regions.

The symbolic trajectory induces a directed transition graph
\(G=(\Sigma,\mathcal{E})\), with an edge for each observed transition between
symbols, as illustrated in \figref{fig:fp_reduction_concept}(b).

We measure symbolic complexity by the number of visited symbols,
\(|\Sigma|\).
We refer to a representation attaining the minimum possible
\(|\Sigma|\) while reproducing the
target dynamics as a \emph{minimal piecewise-linear representation}.
This notion of minimality concerns the number of visited linear regions
rather than the number of nonlinear units \(P\).

% ============================================================
% 3. FIXED-POINT THEORY
% Target: ~0.6--0.7 page
% ============================================================

\subsection{Fixed-Point Theory}
\label{sec:fixed_point_theory}

For a linear region \(\Omega_{\boldsymbol{\sigma}}\), if
\(\mathbf{I}-\mathbf{F}_{\boldsymbol{\sigma}}\) is nonsingular, its
unique candidate fixed point is
\begin{equation}
    \mathbf{z}^{*}_{\boldsymbol{\sigma}}
    =
    \left(
        \mathbf{I}
        -
        \mathbf{F}_{\boldsymbol{\sigma}}
    \right)^{-1}
    \mathbf{h}. \label{eq:regional_fixed_point}
\end{equation}
We therefore define
\begin{align}
    \Sigma_{\mathrm{FP}}
    &:=
    \left\{
        \boldsymbol{\sigma}\in\Sigma
        \;\middle|\;
        \mathbf{z}^{*}_{\boldsymbol{\sigma}}
        \in
        \Omega_{\boldsymbol{\sigma}}
    \right\},
    \\
    \Sigma_{\mathrm{free}}
    &:=
    \Sigma\setminus\Sigma_{\mathrm{FP}}.
\end{align}
Symbols in \(\Sigma_{\mathrm{FP}}\) are called \emph{FP symbols}, whereas
those in \(\Sigma_{\mathrm{free}}\) are called \emph{FP-free symbols}.
% \Figref{fig:fp_reduction_concept}(a,b) illustrates this distinction
% and the corresponding symbolic transition graph.
\Figref{fig:fp_reduction_concept}(a) illustrates the distinction
between FP and FP-free symbols according to whether the corresponding
linear region contains a valid fixed point.

Let \(Q\) denote the number of distinct fixed points of the target
dynamical system that the learned piecewise-linear representation is
required to reproduce.
Here, ``preserving'' these fixed points means that the reconstructed
dynamics contain \(Q\) distinct fixed points corresponding to those of
the target system.

Assuming that these fixed points are isolated and lie in the interiors
of their corresponding regions, together with the nonsingularity
condition above, each linear region contains at most one valid fixed
point.
It follows that:

\begin{proposition}[Fixed-point lower bound]
\label{prop:fp_lower_bound}
% Any piecewise-linear representation reproducing the \(Q\) target fixed
% points satisfies
% %
% \begin{equation}
%     |\Sigma|
%     \geq
%     |\Sigma_{\mathrm{FP}}|
%     \geq
%     Q.
%     \label{eq:symbol_lower_bound}
% \end{equation}
Suppose that a piecewise-linear representation reproduces \(Q\)
distinct target fixed points and that each of these fixed points lies
in a visited region, so that its corresponding symbol belongs to
\(\Sigma_{\rm FP}\).
Then
\begin{equation}
    |\Sigma| \ge |\Sigma_{\rm FP}| \ge Q.\label{eq:symbol_lower_bound}
\end{equation}
\end{proposition}

% The proof is given in Appendix~\ref{app:fp_proofs}.
The proof and further discussion of the consequences and scope of this bound are provided in Appendix\ref{app:theory}.

Therefore, if a representation reproducing the same \(Q\) target fixed
points attains
\begin{equation}
|\Sigma|=Q,
\label{eq:minimality_certificate}
\end{equation}
then it is minimal in the number of visited linear regions among all
piecewise-linear representations preserving this fixed-point
structure.
The lower bound need not be tight, however: additional FP-free regions
may still be required to organize the global flow.
Our reduction procedure therefore preserves the target fixed-point
structure while testing constructively which remaining symbolic
regions can be removed.

The fixed-point count also yields a necessary lower bound on the number
of nonlinear units $P$.
Because \(P\) ReLUs can generate at most \(2^P\) activation patterns,
\begin{equation}
P
\geq
\left\lceil
\log_2 Q
\right\rceil .
\label{eq:p_lower_bound}
\end{equation}
This condition is necessary but not sufficient: having enough
activation patterns does not guarantee that training will realize the
required fixed points or reproduce the target dynamics.

% ============================================================
% 3. METHOD
% Target: ~1.3--1.5 pages
% ============================================================

\section{Methods: Fixed-Point-Guided Hierarchical Reduction and Retraining}
\label{sec:method}

% \begin{figure}[t]
%     \centering
%     \includegraphics[width=\linewidth]
%     {fig_fp_guided_reduction_concept.png}
%     \caption{
%     Conceptual overview of fixed-point (FP) symbols and fixed-point-guided
%     hierarchical reduction.
%     (a) Each visited ReLU activation pattern defines a linear region.
%     Regions containing valid fixed points are FP symbols, whereas visited
%     regions without fixed points are FP-free symbols.
%     (b) The visited regions and their observed transitions form a symbolic
%     transition graph.
%     (c) Linearizing a ReLU removes its switching boundary and merges symbols
%     that become equivalent under the resulting projection.
%     (d) Repeating this operation yields a reduction tree.
%     Branches that would merge distinct FP symbols are rejected, while valid
%     branches produce progressively simpler symbolic representations.
%     }
%     \label{fig:fp_reduction_concept}
% \end{figure}

\begin{figure}[t]
    \centering
    \includegraphics[width=\linewidth]
    {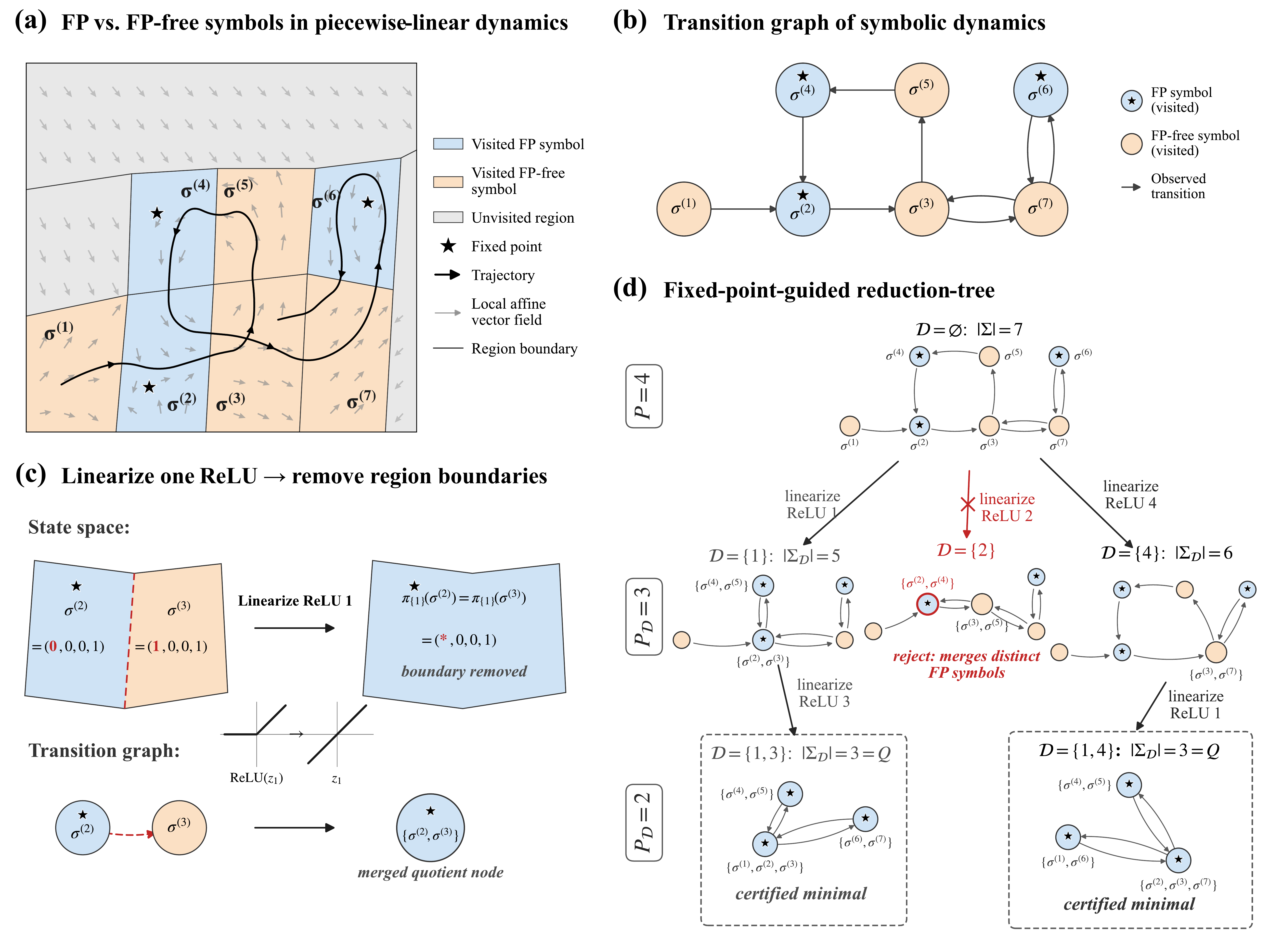}
    \caption{
    \textbf{Proposed fixed-point-guided hierarchical reduction.}
    (a) FP and FP-free symbols in the piecewise-linear state-space partition.
    (b) Symbolic transition graph constructed from visited ReLU activation patterns.
    (c) Region merging induced by ReLU linearization.
    (d) Reduction-tree construction under the fixed-point-preservation constraint.
    }
    \label{fig:fp_reduction_concept}
\end{figure}

% Role:
% Give the entire method in one paragraph before technical details.

Our framework follows a simple strategy:
\emph{learn large, reduce structurally, and retrain compactly}.
We first train a higher-capacity AL-RNN that reliably reconstructs
the target dynamics.
From its symbolic transition graph, we then construct a hierarchy of
progressively reduced representations by linearizing ReLU units while
preventing distinct FP symbols from being merged.
Finally, each selected reduced graph is mapped back to an AL-RNN and
retrained from the parent parameters.
Thus, graph reduction identifies structurally admissible compact
representations, while retraining determines whether they can be
realized as actual dynamical systems.

\subsection{Fixed-point-guided graph reduction}
\label{sec:graph_reduction}

% Role:
% Define the one core graph operation and the reduction tree.
The reduction does not assume that fixed points are provided with the
observed time series.
Once a parent AL-RNN has been trained, candidate fixed points can be
computed analytically from its piecewise-affine dynamics using \eqref{eq:regional_fixed_point}.
We therefore construct each reduction hierarchy directly from the
fixed-point structure recovered by the parent model, rather than from
ground-truth fixed-point annotations or a prescribed fixed-point count.

Let \(G=(\Sigma,\mathcal{E})\) be the symbolic transition graph of a trained
parent AL-RNN, as defined in Sec.~\ref{sec:alrnn}, and let
\(P_{\mathrm{parent}}\) denote the number of nonlinear ReLU units in
the parent model.
Each node is a visited activation pattern
$\boldsymbol{\sigma}
    \in
    \{0,1\}^{P_{\mathrm{parent}}},
$
and is classified as an FP or FP-free symbol according to
$
    \boldsymbol{\sigma}\in\Sigma_{\mathrm{FP}} $ or 
$
    \boldsymbol{\sigma}\in\Sigma_{\mathrm{free}},
$
respectively (\figref{fig:fp_reduction_concept}(a,b)).

% For interpreting graph edges in terms of individual ReLU switches, we assume that the trajectory is sufficiently resolved such that each
% transition between adjacent visited regions crosses a single switching
% boundary, and therefore changes only one ReLU activation bit.
% Simultaneous crossings of multiple switching boundaries are excluded as
% nongeneric cases.

For interpreting graph edges in terms of individual ReLU switches, we assume
that the trajectory is sufficiently resolved such that transitions typically
cross a single switching boundary.
Empirically, $97.1\%$ of observed symbol transitions
($7{,}024/7{,}233$) across the 30 $P_{\mathrm{parent}}=10$ Chua parents
changed exactly one ReLU activation bit, consistent with this assumption.

We reduce the representation by replacing selected ReLU activations by
the identity,
\begin{equation}
    \operatorname{ReLU}(z_j)\rightarrow z_j.
    \label{eq:linearization_method}
\end{equation}
This operation leaves the latent dimension unchanged but removes the
switching boundary associated with ReLU unit \(j\), thereby allowing
the corresponding symbolic regions to merge
(\figref{fig:fp_reduction_concept}(c)).

Let \(\mathcal{D}\) denote the index set of ReLU units that have been linearized.
Removing the activation bits associated with \(\mathcal{D}\) defines a projection

$$
    \pi_{\mathcal{D}}:
    \{0,1\}^{P_{\mathrm{parent}}}
    \to
    \{0,1\}^{P_{\mathrm{parent}}-|\mathcal{D}|}.
$$

% The projected symbols
% \(\Sigma_{\mathcal{D}}:=\pi_{\mathcal{D}}(\Sigma)\)
% form the node set of the quotient graph
% \(G_{\mathcal{D}}=(\Sigma_{\mathcal{D}},E_{\mathcal{D}})\).
The projected symbols
\(\Sigma_{\mathcal D}:=\pi_{\mathcal D}(\Sigma)\)
form the node set of the quotient graph
\(G_{\mathcal D}=(\Sigma_{\mathcal D},\mathcal E_{\mathcal D})\),
whose transitions are obtained by projecting both endpoints of each
parent transition in \(\mathcal E\) through \(\pi_{\mathcal D}\).

For two original symbols
\(\boldsymbol{\sigma},\boldsymbol{\sigma}'\in\Sigma\),
if

$$
    \pi_{\mathcal{D}}(\boldsymbol{\sigma})
    =
    \pi_{\mathcal{D}}(\boldsymbol{\sigma}')
    =
    \tilde{\boldsymbol{\sigma}},
$$

they are merged into the same node
\(\tilde{\boldsymbol{\sigma}}\in\Sigma_{\mathcal{D}}\)
of \(G_{\mathcal{D}}\).

The fixed-point theory in Sec.~\ref{sec:fixed_point_theory} provides the
reduction constraint:
for all distinct FP symbols
\(\boldsymbol{\sigma},\boldsymbol{\sigma}' \in \Sigma_{\mathrm{FP}}\),
we require
\begin{equation}
    \pi_{\mathcal{D}}(\boldsymbol{\sigma})
    \neq
    \pi_{\mathcal{D}}(\boldsymbol{\sigma}').
    \label{eq:fp_constraint_short}
\end{equation}

Starting from \(\mathcal{D}=\varnothing\), we enumerate admissible single-ReLU
linearizations and terminate branches violating
\eqref{eq:fp_constraint_short}.
Duplicate reductions are removed.
The resulting \emph{reduction tree} connects the original
high-resolution graph to progressively coarser symbolic
representations while recording the exact ReLU units linearized along
each path (\figref{fig:fp_reduction_concept}(d)).

% We characterize each reduction by the effective number of nonlinear
% units,
% \(P_{\mathcal{D}}=P_{\mathrm{parent}}-|\mathcal{D}|\),
% and by its symbolic complexity, \(|\Sigma_{\mathcal{D}}|\).
% These quantities need not coincide.
% For retraining, we select reduced graphs attaining the smallest observed
% \(|\Sigma_{\mathcal{D}}|\).
% When \(|\Sigma_{\mathcal{D}}|=Q\), symbol-level minimality of the reduced graph is
% certified by the fixed-point lower bound in
% \secref{sec:fixed_point_theory}.

We characterize each reduction by the remaining nonlinear capacity
\(P_{\mathcal D}\) and symbolic complexity
\(|\Sigma_{\mathcal D}|\).
Candidate-selection criteria are fixed before retraining and depend
only on structural information.
For the 3-scroll Chua experiments, we focus on candidates attaining
both theoretical lower bounds in \secref{sec:fixed_point_theory},
\(|\Sigma_{\mathcal D}|=5\) and \(P_{\mathcal D}=3\).
% Additional benchmark-specific criteria are described in Appendix~\ref{app:candidate_selection}.
% Further details of the reduction procedure, candidate selection, and consensus analysis are provided in Appendix~\ref{app:reduction_algorithm}.
Further details on reduction-tree enumeration, computational cost, and
candidate selection are provided in Appendix~\ref{app:reduction_algorithm}.

\subsection{Reduction-guided retraining}
\label{sec:retraining}

% A quotient graph specifies which nonlinear switching boundaries are
% removed, but does not guarantee that the corresponding reduced AL-RNN
% will realize the target dynamics.
% For each selected reduction with linearization set \(\mathcal D\), we
% initialize the corresponding reduced AL-RNN from the trained parent model.
% and permanently replace the ReLU units in \(\mathcal D\) by
% % linear activations before retraining.

% During retraining, the autonomous parent trajectory serves as the
% teacher.
% Rather than supervising only the observed readout coordinates, we use
% the full latent state of the parent as the teacher signal, with equal
% weight per coordinate.
% This transfers the dynamical realization learned by the
% higher-capacity parent while the linearization mask fixes the reduced
% nonlinear structure.

A quotient graph specifies which nonlinear switching boundaries are
removed, but does not guarantee that the corresponding reduced AL-RNN
will realize the target dynamics.
For each selected reduction with linearization set \(\mathcal D\), we
initialize the corresponding reduced AL-RNN from the trained parent model.
The set $\mathcal D$ remains fixed during retraining.

% We compare two ways of applying the linearization. In the \textit{abrupt} condition, each ReLU unit in $\mathcal D$ is replaced by the identity unit before retraining begins. In the \textit{abrupt} condition, the activation of each ReLU unit in
% $\mathcal D$ is replaced by the identity function before retraining begins. At epoch $e$, the activation of a unit in $\mathcal D$ is
% \begin{equation}
% \label{eq:alrnn-rev-annealing}
%  \phi_{a_e}(x)=(1-a_e)\max(0,x)+a_e x,\qquad
%  a_e=\min(1,e/E).
% \end{equation}
% Here $E>0$ specifies the duration of the transition. We refer to this gradual change as \textit{continuation}. For abrupt linearization, we set $a_e=1$ from the start. Units outside $\mathcal D$ retain their original activation functions. We select the final reduced model from the parameters obtained at epochs
% with $a_e=1$, when the reduced model has exactly $P_{\mathcal D}$ ReLU units.

To allow the reduced model to adapt as switching boundaries are removed, we gradually replace the ReLU activations of the units in $\mathcal D$
with the identity function during retraining. 
% At epoch $e$,
At retraining epoch or optimization step \(e\)
the activation of each such unit is
\begin{equation}
\label{eq:alrnn-rev-annealing}
 \phi_{a_e}(x)=(1-a_e)\max(0,x)+a_e x,\qquad
 a_e=\min(1,e/E).
\end{equation}
% We use $E=200$, so these activations become fully linear after the first 200 retraining epochs. Units outside $\mathcal D$ retain their original ReLU activation.
We use $E=200$, measured in retraining epochs for attractor-reconstruction tasks and optimization steps for memory tasks. The selected activations are fully linear thereafter. Units outside $\mathcal D$ retain their original activation functions.

During retraining, the parent serves as the teacher. We combine two forms of supervision. First, matching the reduced model to the full parent state encourages it to follow the latent trajectories learned by the higher-capacity model. Second, a loss on the signs of the retained ReLU coordinates encourages agreement between the reduced activation pattern $\boldsymbol{\sigma}$ and the parent activation pattern restricted to those units at corresponding times. We refer to these components as \textit{full-preactivation guidance} and \textit{symbol guidance}, respectively, and use their combination for retraining. The task-specific objectives and training settings are provided in Appendices~\ref{app:retraining_detail}-\ref{app:flipflop_details}. 

For attractor-reconstruction tasks, we periodically synchronize the reduced-model state with the corresponding parent state, analogous to the sparse teacher-forcing scheme used in standard AL-RNN training \citep{brenner2024almostlinear}. For memory tasks, both models receive the same initial state and input sequence without subsequent state synchronization, and the task loss is included alongside the two guidance terms.

% Complete loss definitions, synchronization details, and comparisons
% with alternative guidance and weighting schemes are provided in
% Appendix~\ref{app:retraining_detail}.
Implementation details on reduced-model initialization, teacher-trajectory construction, phase alignment, synchronization, and retraining objectives are provided in Appendices~\ref{app:student_initialization}–\ref{app:flipflop_details}.

\subsection{Evaluation}
\label{sec:evaluation_short}

For attractor-reconstruction tasks, each model is evaluated using an autonomous rollout without teacher forcing or state synchronization. We assess fixed-point structure and symbolic complexity through the number of visited admissible fixed points and the number of visited symbols. We additionally assess continuous-state fidelity using the state-space distance $E_{\mathrm{stsp}}$, which measures discrepancy between the generated and target attractor distributions. As a complementary diagnostic of temporal fidelity, we evaluate the spectral Hellinger distance $E_{\mathrm H}$. Task-specific success criteria are stated with the experiments, while rollout settings, count definitions, and metric calculations are detailed in Appendix~\ref{app:evaluation}.

For memory tasks, we evaluate the model's responses to input sequences and its fixed-point structure at zero input. Bit accuracy measures whether the outputs retain the prescribed memory values, while the locations and stability of the zero-input fixed points assess how these values are represented by the learned dynamics. The task and success criterion are described in Sec.~\ref{sec:rev-flipflop}, with evaluation details in Appendix~\ref{app:evaluation}.

We use \(30\) independent random seeds for each training configuration.
Because different parent models can yield different numbers of reduced graphs, the source parent seed is treated as the primary statistical unit for the reduction pipeline.
Full metric definitions and candidate-level statistics are provided in Appendix~\ref{app:evaluation}.

\section{Experiments and Results}
\label{sec:results}

% We evaluate attractor-reconstruction and flip-flop memory tasks using $M=20$ latent coordinates. All main-text reduced models use combined full-preactivation and symbol guidance with continuation. Task-specific training and evaluation details are provided in Appendix~C and Appendix~\ref{app:evaluation}.

We evaluate the framework on attractor-reconstruction tasks and flip-flop memory tasks. Detailed task definitions, training settings, and evaluation procedures are provided in Appendices~\ref{app:training_detail} and \ref{app:evaluation}.

\begin{figure*}[t]
    \centering
    \includegraphics[width=\textwidth]{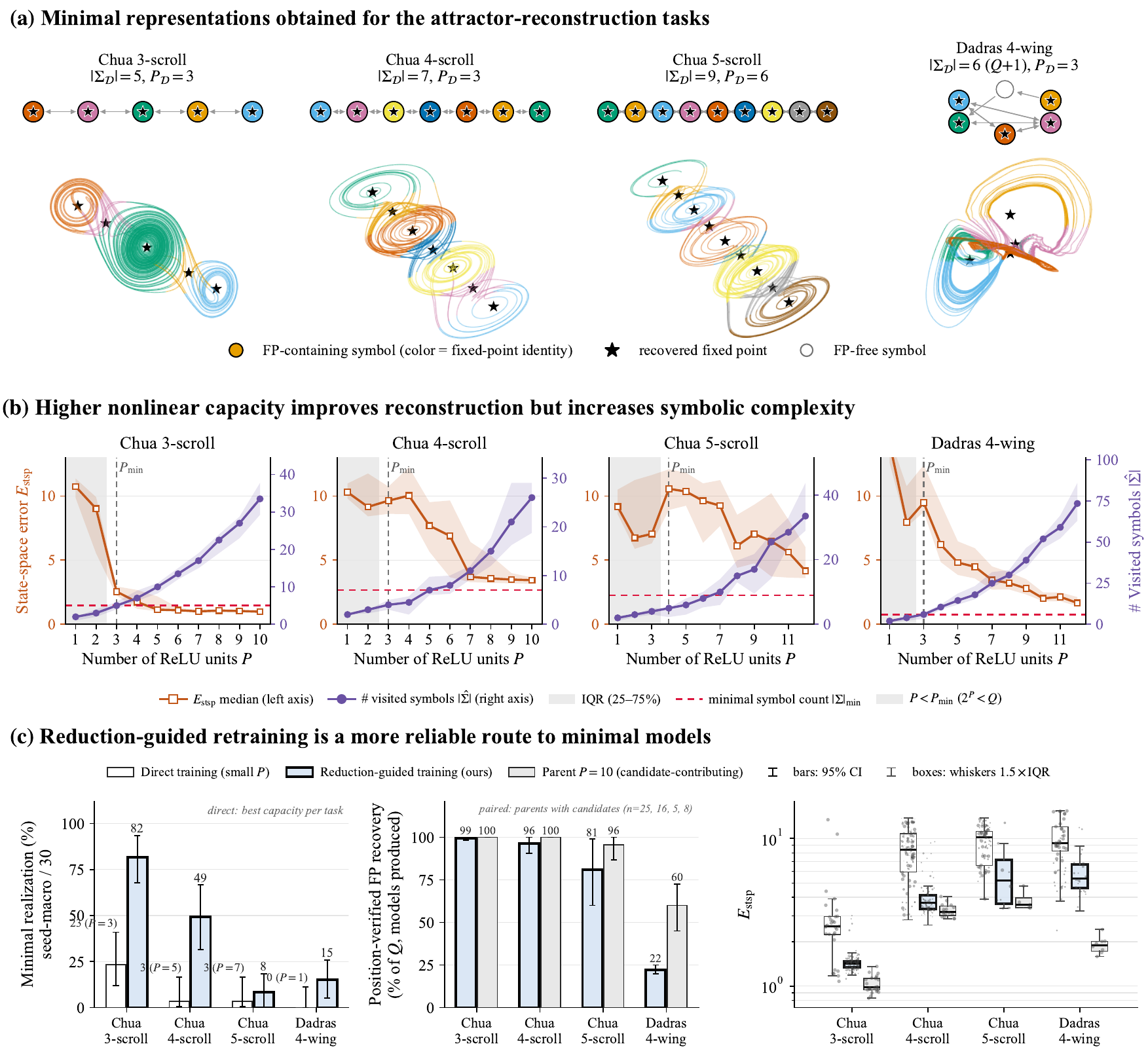}
\caption{\textbf{Main results for the attractor-reconstruction tasks} (Chua 3-/4-/5-scroll and Dadras 4-wing; 30 initial seeds per condition). Parents use $P=10$.
(a) Minimal symbolic representations obtained by hierarchical reduction and the corresponding attractor partitions. 
% The Chua graphs attain the fixed-point-based symbol bounds of 5, 7, and 9; the Dadras example contains 5 FP symbols and one additional FP-free symbol.
(b) Direct-training reconstruction error $E_{\mathrm{stsp}}$ and visited-symbol count $|\widehat{\Sigma}|$ as functions of nonlinear capacity $P$. All 30 seeds are included, regardless of success.
% Higher capacity generally improves reconstruction while increasing symbolic complexity.
% Horizontal dashed lines indicate the minimal symbol counts illustrated in (a); counts below these references do not imply successful minimal reconstruction.
% Vertical dashed lines and gray shading indicate the fixed-point-based capacity bounds and $2^P<Q$, respectively.
Horizontal dashed lines mark the minimal symbol counts in (a); vertical dashed lines mark $P_{\min}$, with $P<P_{\min}$ shaded.
(c) Left: minimal-realization success for direct training at its best tested low capacity and reduction-guided retraining. For retraining, success rates are averaged within each parent and then across all 30 parents, assigning zero to parents without candidates. Middle: position-verified fixed-point recovery over produced models, using one-to-one matching within $0.25\,d_{\min}$; parent bars use the corresponding candidate-contributing parents. Right: state-space error $E_{\mathrm{stsp}}$ over produced models. 
% Our retraining improves minimal-realization success and reconstruction accuracy relative to direct training, while maintaining high positional recovery on 3- and 4-scroll Chua.
}
\label{fig:attractor_results}
    \label{fig:attractor_results}
\end{figure*}

\subsection{Attractor-reconstruction tasks}
\label{sec:rev-autonomous}

We consider 3-, 4-, and 5-scroll Chua systems \citep{tang2001generation} and the Dadras system \citep{dadras2009novel}, with 5, 7, 9, and 5 target fixed points, respectively. For all four systems, we observe and reconstruct all 3 state coordinates. \Figref{fig:attractor_results}(a) illustrates minimal representations obtained by hierarchical reduction. The Chua graphs contain 5, 7, and 9 symbols, attaining the corresponding fixed-point lower bounds, and partition the attractors into distinct FP symbols. The Dadras example contains five FP symbols and one additional FP-free symbol.

Direct training reveals a trade-off between reconstruction accuracy and symbolic complexity (\Figref{fig:attractor_results}(b)). Increasing the number of ReLU units $P$ generally reduces state-space error $E_{\mathrm{stsp}}$ but increases the number of visited symbols $|\widehat{\Sigma}|$. This motivates learning a higher-capacity parent and then reducing its symbolic complexity.

The structural candidates provide effective starting points for learning the reduced dynamics (\Figref{fig:attractor_results}(c), left). Reduction-guided retraining improves minimal-realization success across all four tasks: from approximately $23\%$ to $82\%$ for 3-scroll Chua, $3\%$ to $49\%$ for 4-scroll Chua, $3\%$ to $8\%$ for 5-scroll Chua, and $0\%$ to $15\%$ for Dadras. For each task, the direct baseline uses the tested low-capacity setting with the highest success rate. For retraining, we average the fraction of successful candidates within each parent over all 30 parents, assigning zero to parents without candidates. Our approach therefore enables minimal dynamical realizations that direct training rarely obtains, although success rates remain low for 5-scroll Chua and Dadras.

Fixed point positional matching reveals a more task-dependent outcome: fixed-point recovery remains approximately $99\%$ and $96\%$ for the 3- and 4-scroll models, but falls to $81\%$ for 5-scroll Chua and $22\%$ for Dadras (\Figref{fig:attractor_results}(c), middle). Dadras parents themselves show incomplete positional recovery. The reduced models also achieve lower state-space errors than low-capacity direct training, generally lying between the direct models and their higher-capacity parents (\Figref{fig:attractor_results}(c), right). Detailed fixed-point position analyses, retraining ablations, and controls on the choice of ReLU units to linearize are provided in Appendix~\ref{app:additional_attractor_results}.

\begin{figure*}[t]
    \centering
    \includegraphics[width=\textwidth]{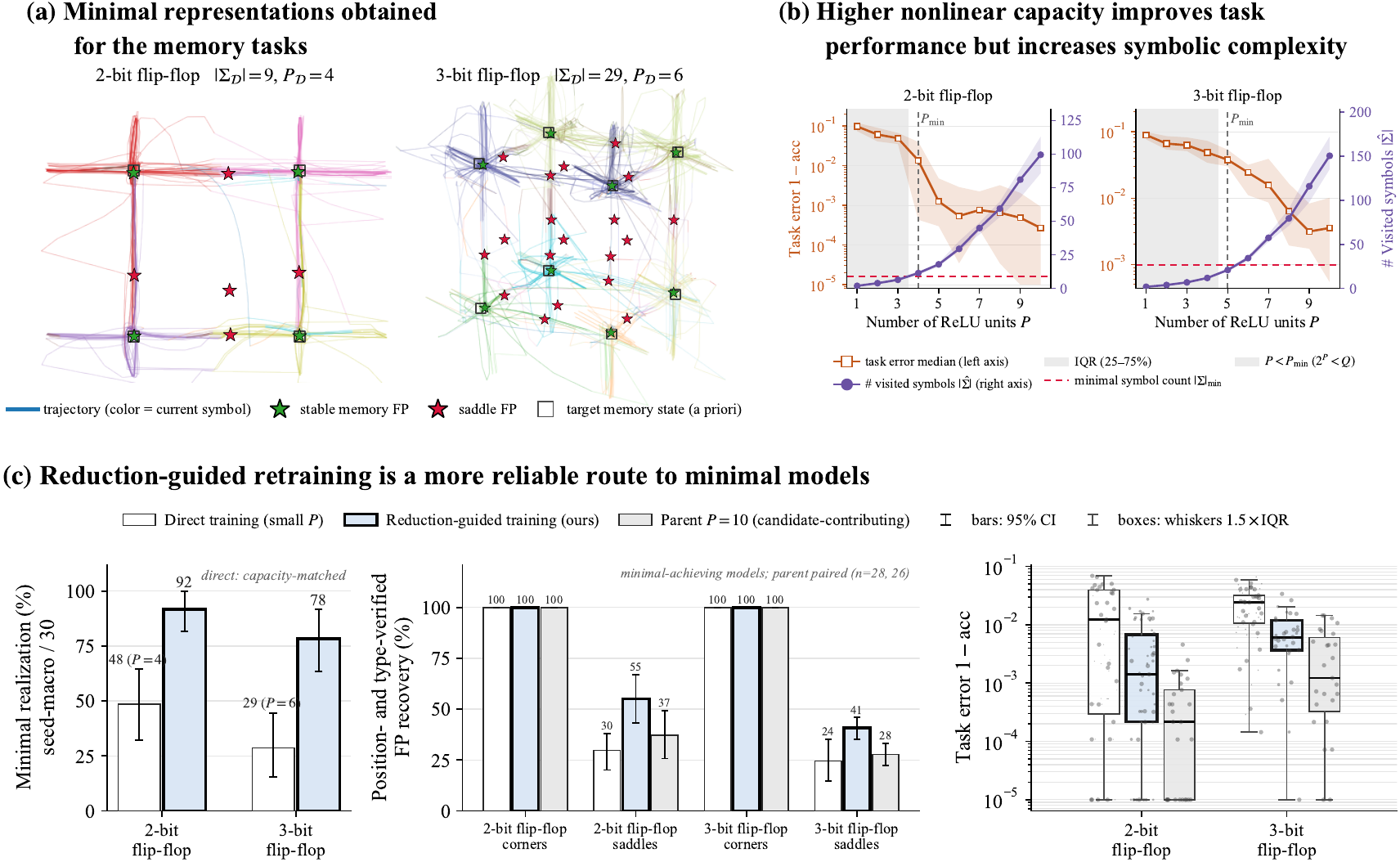}
\caption{
\textbf{Main results for the flip-flop memory tasks} (2- and 3-bit memory; 30 initial seeds per condition; parent $P=10$; retraining condition as in \Figref{fig:attractor_results}).
(a) Minimal symbolic representations and input-driven trajectories of retrained models.
% Green stars mark stable memory fixed points, red stars mark saddles, and open squares mark the prescribed memory states.
Panel labels report quotient-symbol counts recorded by the reduction tree.
(b) Direct-training task error $1-\mathrm{acc}$ and visited-symbol count $|\widehat{\Sigma}|$ as functions of nonlinear capacity $P$. All 30 seeds are included, regardless of success. Symbols are counted using the same input sequence for all models within each bit count.
% Horizontal dashed lines mark the reference FP-symbol counts $3^n$; vertical dashed lines and shading indicate the corresponding capacity bounds.
(c) Left: minimal-realization success for capacity-matched direct training and reduction-guided retraining. For retraining, success rates are averaged within each parent and then across all 30 parents, assigning zero to parents without candidates. Middle: position- and type-verified fixed-point recovery relative to the reference grid $\{-1,0,+1\}^n$, using one-to-one matching within $0.3$, shown separately for memory corners and saddles. Direct and retrained models are restricted to minimal-achieving models; parent bars use the corresponding candidate-contributing parents. Right: task error $1-\mathrm{acc}$ over produced models, with parents restricted to those contributing candidates.}
\label{fig:memory_results}
\end{figure*}

\subsection{Flip-flop memory tasks}
\label{sec:rev-flipflop}

We next test whether reduction preserves input-driven behavior. In the $n$-bit flip-flop task, a pulse of $+1$ or $-1$ writes a value to the corresponding memory channel, which must retain that value until the next pulse \citep{sussillo2013opening}. For 2- and 3-bit memory tasks, the retrained models represent stable memory states and intervening saddles as distinct FP symbols, with input pulses driving transitions between memory states through saddle-associated symbols (\Figref{fig:memory_results}(a)).

As in the attractor-reconstruction tasks, direct training reveals a trade-off: increasing nonlinear capacity $P$ generally reduces task error $1-\mathrm{acc}$ but increases the number of visited symbols $|\widehat{\Sigma}|$ (\Figref{fig:memory_results}(b)). Reduction-guided retraining offers a more reliable route to minimal representations at the same final nonlinear capacity (\Figref{fig:memory_results}(c), left). Minimal-realization success increases from approximately $48\%$ to $92\%$ for 2-bit memory with 4 ReLUs and from $29\%$ to $78\%$ for 3-bit memory with 6 ReLUs. As above, retraining success is averaged within each parent and then across all 30 initial parents, assigning zero to parents without candidates.

Among minimal-achieving models and the corresponding parents, all methods recover the memory corners within the position-and-type tolerance (\Figref{fig:memory_results}(c), middle). Differences arise in the saddles: relative to the reference grid $\{-1,0,+1\}^n$, retrained models recover approximately $55\%$ and $41\%$ for 2- and 3-bit memory, compared with $37\%$ and $28\%$ for parents and $30\%$ and $24\%$ for direct training. Task errors generally remain above those of the parents (\Figref{fig:memory_results}(c), right). Thus, reducing nonlinear capacity can improve recovery of the reference saddle structure even when task error remains higher than in the parents. Detailed fixed-point position analyses, retraining ablations, and controls on the choice of ReLU units to linearize are provided in Appendix~\ref{app:additional_memory_results}.

\section{Discussion and Conclusion}
\label{sec:discussion}

Our results show that higher-capacity AL-RNNs can serve as scaffolds for symbolically minimal dynamical modeling. By learning a higher-capacity parent before reducing its symbolic structure and retraining the resulting model, the learn--reduce--retrain strategy separates structural simplification from optimization near the capacity boundary. In several evaluated settings, this strategy produces symbolically minimal models that direct training rarely obtains.

The reduction tree provides an explicit hierarchy of progressively coarser representations by linearizing ReLU units and merging symbolic regions while preserving distinct FP symbols. Successful retraining need not reproduce the quotient transition graph exactly: in the 3- and 4-scroll Chua examples, non-path-like candidates reorganize into path-like symbolic transition graphs (Appendix~\ref{app:transition_topology}). Reduction therefore provides a structural scaffold without fixing the final transition topology.

The off-tree controls show that the choice of ReLU units to linearize contributes to successful realization of the reduced dynamics. At the same remaining nonlinear capacity and under matched retraining conditions, tree-selected reductions succeed substantially more often than sampled off-tree selections (Appendices~\ref{app:offtree-attractor} and \ref{app:offtree-memory}).

The retraining ablations reveal different effects across tasks. On 3- and 4-scroll Chua, gradual linearization substantially improves success even without auxiliary guidance, and combined guidance with gradual linearization achieves the highest minimal-realization rates among the tested conditions. On memory, retraining without auxiliary guidance already achieves high structural success, whereas symbol guidance improves recovery of the reference fixed-point structure at the cost of higher task error (Appendices~\ref{app:chua_guidance_ablation} and \ref{app:flipflop_ablation}). Combined guidance thus provides a common protocol without being uniformly optimal for every metric.

The fixed-point criterion provides a lower bound rather than a complete characterization of minimal dynamics: FP-free regions may still be required to organize the global flow. Because reduction preserves the fixed-point structure recovered by the parent, its effectiveness depends on the quality of that structure. Even successful parents do not always admit reductions attaining both the symbolic-complexity and nonlinear-capacity lower bounds.

% Exhaustive reduction scales combinatorially with $P_{\mathrm{parent}}$, although tree construction remains inexpensive in the tested regime: median times were $0.2$~s at $P_{\mathrm{parent}}=10$ and $10.2$~s at $P_{\mathrm{parent}}=14$ (Appendix~\ref{app:candidate_selection}). 
% The combinatorial dependence is on the number of nonlinear switching variables rather than directly on the latent-state dimension.
% Learning a parent provides concrete reduction candidates, avoiding a separate capacity sweep that may extend beyond the fixed-point-based lower bound before direct training succeeds. This can offset the additional training expense:
% including parent training and reduction, our method reduces the estimated expected first-success cost by a factor of approximately $2.7$--$4.3$ relative to parent-informed direct search on 4- and 5-scroll Chua and both memory tasks (Appendix~\ref{app:computational_cost}).

Exhaustive reduction scales combinatorially with the number of nonlinear switching variables $P_{\mathrm{parent}}$, rather than directly with the latent-state dimension. Nevertheless, tree construction remains inexpensive in the tested regime, with median times of $0.2$~s at $P_{\mathrm{parent}}=10$ and $10.2$~s at $P_{\mathrm{parent}}=14$ (Appendix~\ref{app:candidate_selection}). 
% The resulting tree provides concrete reduction candidates for retraining, avoiding the separate capacity sweep that direct training may require even after estimating a fixed-point-based lower bound. 
% In the evaluated cohorts, parent-informed direct search first achieved minimal realization above its estimated starting capacity on 4- and 5-scroll Chua and both memory tasks.
% Avoiding these repeated training runs can offset the additional parent-training expense: including parent training and reduction, our method reduces the estimated expected first-success cost by a factor of approximately $2.7$--$4.3$ relative to parent-informed direct search on 4- and 5-scroll Chua and both memory tasks (Appendix~\ref{app:computational_cost}).
The resulting tree provides concrete candidates for retraining without a separate capacity sweep that direct training may require even after estimating a fixed-point-based lower bound. In the evaluated cohorts, parent-informed direct search first achieved minimal realization above its estimated starting capacity on 4- and 5-scroll Chua and both memory tasks. On these tasks, our method reduces the estimated expected first-success cost by a factor of approximately $2.7$--$4.3$ relative to this direct search, including parent training and reduction (Appendix~\ref{app:computational_cost}).

% Our empirical study covers both autonomous dynamics and input-driven memory involving multiple fixed points. Extending the framework to noisy, partially observed, high-dimensional real-world data remains an important direction for future work.

\subsection*{AI use statement}

We used generative AI tools, including OpenAI ChatGPT and Anthropic
Claude, to assist with manuscript drafting and editing, literature
retrieval and discovery, discussion and refinement of theoretical and methodological formulations, experimental design, and software implementation and debugging.
We also used these tools to assist with the formulation and refinement
of mathematical claims and derivations, and with the preparation and
refinement of scientific figures.

The authors critically reviewed all AI-assisted content.
Mathematical claims and derivations were independently checked by the
authors, and AI-assisted code was executed, tested, and validated
against the intended methodology.
Experimental results reported in the paper were obtained from the
authors' implementations and experiments rather than generated by AI,
and cited literature was checked against the original sources.
The authors take full responsibility for the final content of this work,
including all text, claims, figures, code, and artifacts produced with
the aid of generative AI.

\subsection*{Reproducibility statement}

We provide detailed definitions of the AL-RNN architecture, the
fixed-point-guided reduction procedure, reduction-guided retraining, and
the evaluation criteria in the main text and appendix.
The appendix documents the experimental settings, hyperparameters,
random-seed protocol, fixed-point computation and evaluation procedures,
and other implementation details needed to reproduce the reported
experiments, together with additional analyses and derivations
supporting the theoretical claims.
Source code and experiment scripts for reproducing all
reported results, covering training, hierarchical reduction,
retraining, evaluation, and figure generation, are available at
% \url{https://anonymous.4open.science/r/alrnn-fp-guided-reduction-DB32/}.
\url{https://github.com/hiroto0324/alrnn-fp-guided-reduction}.
The repository additionally contains the dataset and all trained model
checkpoints, so that every reported figure and statistic can be
regenerated directly from the released models without retraining.

% \subsubsection*{Acknowledgments}
% Omitted for anonymous review.

\subsubsection*{Acknowledgments}
The authors would like to thank Professor Kazuyuki Aihara for valuable discussions, guidance, and encouragement on this research.
This work was partially supported by JSPS KAKENHI Grant Numbers JP24K20855 (HT), JP23K28154 (GT), and JP25K24744 (GT); JST CREST Grant Number JPMJCR24R2 (GT); JST ACT-X Grant Number JPMJAX26LG (HT); JST Moonshot R\&D Grant Number JPMJMS2021 (GT); a project, JPNP14004, commissioned by the New Energy and Industrial Technology Development Organization (NEDO) (HT, GT); and the Institute of AI and Beyond, The University of Tokyo (HT).

\bibliography{iclr2027_conference}
\bibliographystyle{iclr2027_conference}

\clearpage
\appendix

\renewcommand{\thefigure}{A\arabic{figure}}
\renewcommand{\thetable}{A\arabic{table}}
\setcounter{figure}{0}
\setcounter{table}{0}
\setcounter{equation}{0}
\renewcommand{\theequation}{A\arabic{equation}}

% Give appendix references unique hyperref anchors.
\renewcommand*{\theHfigure}{appendix.\arabic{figure}}
\renewcommand*{\theHtable}{appendix.\arabic{table}}
\renewcommand*{\theHequation}{appendix.\arabic{equation}}

% ============================================================
% A. ADDITIONAL THEORY AND PROOFS
% ============================================================

\section{Additional Theory and Proofs}
\label{app:theory}

% \subsection{Assumptions and proof of the fixed-point lower bound}
% \label{app:fp_proofs}

% We assume that the target fixed points considered in the analysis are
% isolated and lie in the interiors of their corresponding linear
% regions.
% We further assume that
% \(\mathbf I-\mathbf F_{\boldsymbol{\sigma}}\)
% is nonsingular for every relevant region.
% These conditions exclude fixed points lying exactly on switching
% boundaries and affine subsystems containing continua of fixed points.

% \begin{lemma}[At most one fixed point per region]
% \label{lemma:one_fp_per_region_appendix}
% Under these assumptions, each linear region
% \(\mathcal R_{\boldsymbol{\sigma}}\)
% contains at most one fixed point.
% \end{lemma}

\subsection{Assumptions and proof of the fixed-point lower bound}
\label{app:fp_proofs}

We assume that the target fixed points considered in the analysis are
isolated and lie in the interiors of their corresponding linear
regions.
We further assume that
\(\mathbf I-\mathbf F_{\boldsymbol{\sigma}}\)
is nonsingular for every relevant region.
These conditions exclude fixed points lying exactly on switching
boundaries and affine subsystems containing continua of fixed points.

\begin{lemma}[At most one fixed point per region]
\label{lemma:one_fp_per_region_appendix}
Under these assumptions, each linear region
\(\Omega_{\boldsymbol{\sigma}}\)
contains at most one fixed point.
\end{lemma}

\begin{proof}
Within a fixed region, a fixed point $\mathbf z^*$ satisfies
\[
    (\mathbf I-\mathbf F_{\boldsymbol{\sigma}})
    \mathbf z^*
    =
    \mathbf h.
\]
Because
\(\mathbf I-\mathbf F_{\boldsymbol{\sigma}}\)
is nonsingular, this linear system has a unique solution.
Hence the region contains at most one valid fixed point.
\end{proof}

% \begin{proof}[Proof of Proposition~\ref{prop:fp_lower_bound}]
% Suppose a piecewise-linear representation reproduces \(Q\) distinct
% target fixed points.
% By Lemma~\ref{lemma:one_fp_per_region_appendix}, no two of these fixed
% points can occupy the same linear region.
% They therefore require at least \(Q\) distinct FP symbols:
% \[
%     |\Sigma_{\mathrm{FP}}|
%     \geq Q.
% \]
% Since
% \(\Sigma_{\mathrm{FP}}\subseteq\Sigma\),
% \[
%     |\Sigma|
%     \geq
%     |\Sigma_{\mathrm{FP}}|
%     \geq Q.
% \]
% \end{proof}

\begin{proof}[Proof of Proposition~\ref{prop:fp_lower_bound}]
By assumption, each of the \(Q\) target fixed points lies in a visited
region, and hence its corresponding symbol belongs to
\(\Sigma_{\rm FP}\).
By Lemma~\ref{lemma:one_fp_per_region_appendix}, no two distinct target fixed
points can occupy the same linear region.
They therefore require at least \(Q\) distinct FP symbols:
\[
    |\Sigma_{\rm FP}| \geq Q.
\]
Since
\(\Sigma_{\rm FP}\subseteq\Sigma\),
\[
    |\Sigma|
    \geq
    |\Sigma_{\rm FP}|
    \geq
    Q.
\]
\end{proof}

\subsection{Consequences for symbolic and nonlinear complexity}
\label{app:minimality_consequences}

If a representation reproducing the same \(Q\) target fixed points
satisfies
\[
    |\Sigma|=Q,
\]
then no representation preserving these fixed points can contain fewer
visited linear regions.
It is therefore minimal in symbolic complexity under the stated
fixed-point-preservation requirement.

This certificate concerns the number of visited regions rather than
the number of nonlinear units.
Because \(P\) ReLU units admit at most \(2^P\) activation patterns,
preserving \(Q\) distinct FP symbols also requires
\[
    P
    \geq
    \left\lceil\log_2 Q\right\rceil.
\]
% This is a necessary but not sufficient condition: sufficient
% activation-pattern capacity does not guarantee that the required
% fixed points or global dynamics can be realized or learned.
This is a necessary but not sufficient condition:
meeting this bound does not guarantee that the required fixed points
or global dynamics can be realized or learned.

\subsection{Scope and possible non-tightness of the certificate}
\label{app:certificate_scope}

The fixed-point lower bound specifies a minimum number of symbolic
regions required by the fixed-point structure, but it does not imply
that this number is sufficient to reproduce the full attractor
dynamics.
Additional FP-free regions may still be required to organize the
global flow.

Consequently, the absence of a fixed point does not imply that a
symbol can be removed.
Our reduction procedure therefore preserves the fixed-point skeleton
while exploring which additional FP-free regions can be merged.
The corresponding reduced models are then retrained and evaluated to
determine whether they can realize the target dynamics faithfully.

Importantly, a quotient graph accepted by the fixed-point-preservation
constraint is only a structurally admissible candidate.
Retraining may introduce or eliminate visited symbolic regions, so the
resulting autonomous dynamics must be evaluated separately.
Whether the bound is tight must therefore be determined through
retraining and autonomous dynamical evaluation.
% The R\"ossler benchmark in Appendix~\ref{app:rossler} illustrates a case in which the fixed-point lower bound is not tight.

\clearpage
% ============================================================
% B. HIERARCHICAL REDUCTION DETAILS
% ============================================================

\section{Hierarchical Reduction Details}
\label{app:reduction_algorithm}

\subsection{ReLU linearization and symbol projection}
\label{app:symbol_projection}

Let
\(\boldsymbol{\sigma}\in
\{0,1\}^{P_{\mathrm{parent}}}\)
denote an original activation pattern and let \(\mathcal{D}\) be the index set of linearized ReLU units.
For each \(j\in \mathcal{D}\), we replace
\begin{equation}
    \operatorname{ReLU}(z_j)
    \longrightarrow
    z_j.
\end{equation}
This operation does not remove a latent-state dimension.
Instead, it removes the switching boundary induced by the $j$th ReLU while preserving the full latent dimension \(M\).

The set \(\mathcal{D}\) induces the projection
\begin{equation}
    \pi_{\mathcal{D}}:
    \{0,1\}^{P_{\mathrm{parent}}}
    \rightarrow
    \{0,1\}^{P_{\mathrm{parent}}-|\mathcal{D}|},
\end{equation}
which removes all activation-pattern coordinates indexed by \(\mathcal{D}\).
Two original symbols are therefore equivalent after the linearization
if
\begin{equation}
    \boldsymbol{\sigma}
    \sim_{\mathcal{D}}
    \boldsymbol{\sigma}'
    \iff
    \pi_{\mathcal{D}}(\boldsymbol{\sigma})
    =
    \pi_{\mathcal{D}}(\boldsymbol{\sigma}').
\end{equation}
The nodes of the reduced quotient graph are the resulting equivalence
classes.

The reduced transition set is
\begin{equation}
    \mathcal E_{\mathcal D}
    :=
    \left\{
        \left(
            \pi_{\mathcal D}(\boldsymbol{\sigma}),
            \pi_{\mathcal D}(\boldsymbol{\sigma}')
        \right)
        :
        (\boldsymbol{\sigma},\boldsymbol{\sigma}')\in\mathcal E
    \right\}.
\end{equation}
Thus, the reduced transition graph is obtained directly by projecting
both endpoints of each parent transition onto the retained activation
bits.
Self-transitions created by quotienting may be omitted from graph
visualizations.

\subsection{FP-symbol separation constraint}
\label{app:fp_preservation}

A reduction is structurally admissible only if no two distinct FP
symbols are projected onto the same quotient node:
\begin{equation}
    \forall\,
    \boldsymbol{\sigma},\boldsymbol{\sigma}'
    \in\Sigma_{\mathrm{FP}},
    \quad
    \boldsymbol{\sigma}\neq\boldsymbol{\sigma}',
    \qquad
    \pi_{\mathcal{D}}(\boldsymbol{\sigma})
    \neq
    \pi_{\mathcal{D}}(\boldsymbol{\sigma}').
    \label{eq:fp_preservation_appendix}
\end{equation}
This condition allows an FP-free symbol to merge with another FP-free
symbol or with a single FP symbol, but rejects any reduction that would
merge two distinct FP symbols.

The condition protects only the fixed-point skeleton.
It does not imply that every accepted quotient graph can be realized
as a faithful reduced dynamical model.
Dynamical realization is assessed separately through
reduction-guided retraining and autonomous evaluation.

% \paragraph{Protected FP symbols.}
% For attractor-reconstruction tasks, we protect admissible FP symbols visited during parent evaluation. Any such symbol missing from the graph-construction trajectory is checked against the selected parent parameters and added to the graph, as an isolated node if no transition involving it was observed. For memory tasks, we protect all admissible zero-input FP symbols of the parent, including those containing unstable fixed points (Appendix~\ref{app:flipflop_details}).

\subsection{Reduction-tree enumeration}
\label{app:reduction_tree}

Starting from
\(\mathcal{D}=\varnothing\),
we systematically consider each ReLU unit that has not yet been
linearized.
For a candidate ReLU unit \(j\), we form
$
    \mathcal{D}'
    =
    \mathcal{D}\cup\{j\}
$
and compute the induced symbol projection
\(\pi_{\mathcal{D}'}\).
If \eqref{eq:fp_preservation_appendix} is violated, the branch is
terminated.
Because further linearization can only remove additional activation
bits, once two distinct FP symbols collide they cannot be separated
again by any descendant reduction.
Otherwise, the corresponding quotient graph is added to the reduction
tree and further linearizations are considered.

Because different linearization orders can lead to the same set
\(\mathcal{D}\), each linearization set is stored only once.
A concise pseudocode description is given below.

\[
\begin{array}{ll}
\textbf{Input:}
& G=(\Sigma,\mathcal E),\ \Sigma_{\mathrm{FP}},
  \ P_{\mathrm{parent}} \\[2mm]

\textbf{Initialize:}
& \mathcal D\leftarrow\varnothing,\quad
  \mathrm{Queue}\leftarrow\{\mathcal D\} \\[2mm]

\textbf{while } \mathrm{Queue}\neq\varnothing: &
\\
& \text{pop } \mathcal D \\[1mm]
& \textbf{for each }
  j\in\{1,\ldots,P_{\mathrm{parent}}\}\setminus\mathcal D: \\[1mm]
& \qquad \mathcal D'\leftarrow\mathcal D\cup\{j\} \\[1mm]
& \qquad \text{project all symbols with }
  \pi_{\mathcal D'} \\[1mm]
& \qquad \textbf{if distinct FP symbols collide:} \\[1mm]
& \qquad\qquad \text{reject }\mathcal D' \\[1mm]
& \qquad \textbf{else:} \\[1mm]
& \qquad\qquad \text{construct }G_{\mathcal D'} \\[1mm]
& \qquad\qquad \textbf{if }\mathcal D'\text{ has not been visited:} \\[1mm]
& \qquad\qquad\qquad
  \text{add }(\mathcal D',G_{\mathcal D'})
  \text{ as a child of }(\mathcal D,G_{\mathcal D}) \\[1mm]
& \qquad\qquad\qquad
  \text{push }\mathcal D' \\[2mm]

\textbf{Output:}
& \text{tree of admissible }(\mathcal D,G_{\mathcal D})\text{ pairs,} \\
& \text{with edges corresponding to single-ReLU linearizations.}
\end{array}
\]

\subsection{Reduction complexity and candidate selection}
\label{app:candidate_selection}

For each linearization set \(\mathcal D\), we characterize the reduced
representation by two complementary quantities:
\[
    P_{\mathcal D}
    =
    P_{\mathrm{parent}}-|\mathcal D|,
    \qquad
    |\Sigma_{\mathcal D}|.
\]
The first is the number of ReLU units remaining after linearization,
whereas the second is the symbolic complexity of the quotient graph.
These quantities need not attain their minima simultaneously.

Candidate selection is fixed before retraining and never uses
post-retraining reconstruction performance.
For the 3-scroll Chua system, the fixed-point theory gives a
five-symbol lower bound and a 3-ReLU lower bound.
The principal capacity-boundary analysis therefore uses candidates
satisfying
\[
    |\Sigma_{\mathcal D}|=5,
    \qquad
    P_{\mathcal D}=3.
\]
Five-symbol candidates with larger \(P_{\mathcal D}\) still attain the
certified symbol-count lower bound, but do not attain the
nonlinear-capacity lower bound.

\paragraph{Computational cost of reduction.}
Each reduction is uniquely specified by a subset \(\mathcal D\) of the
\(P_{\mathrm{parent}}\) nonlinear units, so the number of distinct
reduction masks is at most \(2^{P_{\mathrm{parent}}}\).
The exhaustive search therefore has a combinatorial worst-case
dependence on \(P_{\mathrm{parent}}\), although branches are pruned once
the fixed-point-separation constraint is violated.
In practice, each reduction step consists only of symbolic projection,
fixed-point-collision checks, and quotient-graph construction, and
requires no gradient-based optimization.

To characterize the empirical scaling of the structural search, we
measured exhaustive reduction-tree construction for
\(P_{\mathrm{parent}}=10,\ldots,14\), using 30 independently trained
parent models for each value of \(P_{\mathrm{parent}}\).
All timings were measured single-threaded on one thread of an
Intel Core i9-10980XE (3.0\,GHz) desktop CPU.
The resulting tree sizes and wall-clock times are summarized in
Table~\ref{tab:reduction_runtime}.

\begin{table}[t]
    \centering
    \caption{
    Empirical scaling of exhaustive reduction-tree construction.
    Values are medians over 30 parent models, with interquartile ranges
    in brackets.
    }
    \label{tab:reduction_runtime}
    \begin{tabular}{rrrr}
        \toprule
        \(P_{\mathrm{parent}}\)
        & \(2^{P_{\mathrm{parent}}}\)
        & Admissible nodes
        & Tree time [s] \\
        \midrule
        10 & 1{,}024
           & 476 [382--586]
           & 0.2 [0.1--0.3] \\
        11 & 2{,}048
           & 1{,}133 [861--1{,}270]
           & 0.5 [0.4--0.7] \\
        12 & 4{,}096
           & 2{,}324 [1{,}936--2{,}685]
           & 1.4 [1.2--1.7] \\
        13 & 8{,}192
           & 5{,}145 [3{,}554--5{,}498]
           & 3.3 [2.7--4.3] \\
        14 & 16{,}384
           & 10{,}616 [8{,}543--11{,}779]
           & 10.2 [8.1--11.5] \\
        \bottomrule
    \end{tabular}
\end{table}

For the \(P_{\mathrm{parent}}=10\) setting used in the main experiments,
exhaustive tree construction required only a median of \(0.2\)~s per
parent.
As expected from the combinatorial growth of the mask space, the
runtime increased rapidly with \(P_{\mathrm{parent}}\).
Nevertheless, the absolute cost remained modest over the range examined:
even for \(P_{\mathrm{parent}}=14\), corresponding to
\(2^{14}=16{,}384\) possible masks, exhaustive construction required a
median of only \(10.2\)~s per parent.
Thus, for the nonlinear capacities considered here, the structural
reduction stage is computationally inexpensive despite its combinatorial
worst-case scaling.

The exhaustive traversal used in this study is primarily an analysis
choice rather than a requirement of the learn--reduce--retrain
framework.
We enumerate the full admissible hierarchy to characterize variability
across parent representations and to support the consensus analysis in
Appendix~\ref{app:consensus_granularity}.
If the objective is simply to obtain a compact realization, the search
can instead terminate after one or several candidates satisfying a
prescribed structural criterion have been identified.
Such targeted search avoids exploring the full reduction tree, although
it does not provide the exhaustive candidate distribution required by
the consensus estimator.

\subsection{Parent and candidate selection in the experiments}
\label{app:selection_experiments}

\begin{table*}[t]
    \centering
    \small
    \setlength{\tabcolsep}{4pt}
    \caption{Reduction-candidate selection by task. Each task starts from 30 parent seeds; the screened-parent column reports the number retained after parent screening. Candidate counts are broken down by remaining nonlinear capacity $P_{\mathcal D}$. The final column gives the total number of selected candidates and the number of parents contributing at least one candidate.}
    \label{tab:selection}
    \begin{tabular}{@{}ccp{0.32\textwidth}p{0.14\textwidth}r@{}}
        \toprule
        Task & \shortstack{Screened\\parents} & Selection rule & \shortstack{Candidates \\ by $P_{\mathcal D}$} & \shortstack{Candidates\\/ parents} \\
        \midrule
        Chua 3-scroll & 30/30 & Global minimum $|\Sigma_{\mathcal D}|=5$. & $3:126$\newline $4:15$ & 141/26 \\
        \addlinespace
        Chua 4-scroll & 30/30 & Global minimum $|\Sigma_{\mathcal D}|=7$. & $3:1,\;4:17$\newline $5:74,\;6:10$ & 102/16 \\
        \addlinespace
        Chua 5-scroll & 14/30 & Global minimum $|\Sigma_{\mathcal D}|=9$. & $4:1,\;5:2$\newline $6:7$ & 10/5 \\
        \addlinespace
        Dadras 4-wing & 29/30 & At $P_{\mathcal D}=3$, retain the two smallest layers, $|\Sigma_{\mathcal D}|\in\{5,6\}$. At $P_{\mathcal D}=4$, additionally retain the smallest layer, $|\Sigma_{\mathcal D}|=6$. & $3:23$\newline $4:3$ & 26/8 \\
        \midrule
        2-bit flip-flop & 30/30 & Per parent: minimum $|\Sigma_{\mathcal D}|$, then minimum $P_{\mathcal D}$; up to 4 distinct linearization sets. & $4:49$\newline $6:3$ & 52/30 \\
        \addlinespace
        3-bit flip-flop & 30/30 & Per parent: minimum $|\Sigma_{\mathcal D}|$, then minimum $P_{\mathcal D}$; up to 4 distinct linearization sets. & $6:34$\newline $7:3$ & 37/30 \\
        \bottomrule
    \end{tabular}
\end{table*}

Each admissible reduction is specified by its remaining nonlinear capacity $P_{\mathcal D}$ and quotient-symbol count $|\Sigma_{\mathcal D}|$. Table~\ref{tab:selection} summarizes the candidate sets used for retraining. Selection differs across tasks in the scope over which $|\Sigma_{\mathcal D}|$ is minimized, the number of retained layers, and the treatment of nonlinear capacity. 

For the Chua systems, we retain all candidates attaining the smallest $|\Sigma_{\mathcal D}|$ in the available candidate pool across parents and deletion depths. This gives 141 five-symbol candidates for 3-scroll Chua, 102 seven-symbol candidates for 4-scroll Chua, and 10 nine-symbol candidates for 5-scroll Chua. Candidates with the same quotient-symbol count can therefore retain different numbers of ReLU units.
% The capacity-matched 3-scroll analysis uses only the 126 candidates with $P_{\mathcal D}=3$.

For Dadras, we retain both of the two smallest available symbol-count layers at $P_{\mathcal D}=3$, giving 23 candidates with $|\Sigma_{\mathcal D}|\in\{5,6\}$. We additionally retain the three candidates in the smallest available layer at $P_{\mathcal D}=4$, all with $|\Sigma_{\mathcal D}|=6$.

For memory tasks, selection is performed separately for each parent. We first minimize $|\Sigma_{\mathcal D}|$ and then, among candidates attaining that count, maximize the number of linearized units. We retain distinct linearization sets satisfying these criteria, up to 4 per parent. The protected set contains all admissible zero-input FP symbols of the parent, including those containing unstable fixed points. The main analysis starts from all 30 initial parents. 

To compare guidance conditions when the parent already solves the memory task, the memory guidance ablation includes only parents with accuracy $\geq 0.98$ and a stable fixed point within $0.3$ of every memory corner, leaving 28 parents for 2-bit memory and 26 for 3-bit memory (Appendix~\ref{app:flipflop_ablation}). This screening is applied before retraining and identically across all guidance conditions, with training random seeds paired across conditions.

\paragraph{Selection using a parent-derived fixed-point count.}
For attractor-reconstruction tasks, parent checkpoint selection is label-free with respect to the target fixed-point count. For each of the 30 parent training runs, we select the checkpoint with the smallest state-space discrepancy. We estimate $\widehat Q$ as the most frequent value of $Q_{\mathrm{vis}}$ among these parents and retain parents with this count for reduction. Ties are resolved by choosing the larger count. Retrained checkpoints are selected by minimizing the state-space discrepancy plus the penalty $5|Q_{\mathrm{vis}}-\widehat Q|$. Thus, the fixed-point count used for selection is inferred from the trained parents.

\subsection{Consensus estimation of symbolic granularity}
\label{app:consensus_granularity}
We additionally examine whether the reduction hierarchies can guide the choice of symbolic granularity when the target system's fixed-point structure is unknown. This consensus analysis is separate from the candidate-selection procedures used in our experiments.

% The fixed-point structure of each trained parent provides a data-driven lower bound on the number of symbols that a reduction must retain. For parent $r$, let $Q_{\mathrm{vis}}^{(r)}$ denote the number of analytically admissible fixed points whose symbolic regions are visited during autonomous dynamics. Any reduction preserving these FP symbols must satisfy
The fixed-point structure of each trained parent provides a data-driven lower bound on the number of symbols that a reduction must retain. For parent $r$, let $Q_{\mathrm{prot}}^{(r)}$ denote the number of protected FP symbols. For attractor-reconstruction tasks, these are the admissible FP symbols visited during parent evaluation; for memory tasks, they are the symbols containing the parent's admissible zero-input fixed points, including unstable fixed points. Any reduction preserving these FP symbols as distinct symbols must satisfy
\begin{equation}
    |\Sigma_{\mathcal D}|
    \geq
    Q_{\mathrm{prot}}^{(r)} .
\end{equation}

This bound is parent-dependent and need not be tight, since additional
FP-free regions may be required to organize the global flow.
We therefore consider a simple consensus-based strategy for estimating
the relevant symbolic granularity from reduction hierarchies across
independently trained parents.

% Our consensus procedure is conditioned on a prescribed reduced
% nonlinear capacity \(P_{\mathcal D}=P_{\min}\).
Our consensus procedure is conditioned on a prescribed remaining nonlinear capacity $P_{\mathcal D}$.
For each independently trained parent \(r\), we collect all distinct
admissible reduction nodes at this capacity.
Rejected nodes, including reductions that violate the fixed-point
preservation constraint, are excluded, and duplicate linearization masks are counted only once within each parent.
Let \(n_r(k)\) denote the number of such admissible nodes from parent
\(r\) with symbolic complexity \(|\Sigma_{\mathcal D}|=k\).

% This procedure estimates symbolic granularity conditional on a chosen
% reduced nonlinear capacity \(P_{\min}\); selecting \(P_{\min}\) itself
% is a separate model-selection problem.
This procedure estimates symbolic granularity at the specified capacity; it does not determine the minimum nonlinear capacity required for dynamical realization.
The parent-inferred fixed-point count provides the necessary capacity
lower bound
\begin{equation}
    P_{\mathcal D}
    \geq
    \left\lceil
        \log_2 Q_{\mathrm{prot}}^{(r)}
    \right\rceil ,
\end{equation}
but this bound need not identify the reduced capacity required for
faithful realization of the full dynamics.

Directly pooling all reduction nodes would give disproportionate
weight to parents whose reduction trees happen to contain more
admissible branches.
We therefore first normalize the vote distribution within each parent:
\begin{equation}
    h_r(k)
    =
    \frac{
        n_r(k)
    }{
        \displaystyle\sum_i n_r(i)
    }.
    \label{eq:parent_symbol_vote}
\end{equation}
The parent-balanced consensus support for symbolic granularity \(k\)
is then
\begin{equation}
    H(k)
    =
    \frac{1}{|\mathcal P_{\mathrm{cons}}|}
    \sum_{r\in\mathcal P_{\mathrm{cons}}}
    h_r(k),
    \label{eq:symbol_consensus}
\end{equation}
%
% where \(\mathcal P_{\mathrm{cons}}\) is the set of parent models that
% contribute at least one admissible node to the
% \(P_{\mathcal D}=P_{\min}\) slice.
where $\mathcal P_{\mathrm{cons}}$ is the set of parent models that contribute at least one admissible reduction at the specified capacity $P_{\mathcal D}$.
The consensus estimate is
\begin{equation}
    \widehat k
    =
    \arg\max_k H(k),
    \label{eq:consensus_K}
\end{equation}
and we define the corresponding consensus strength as
\begin{equation}
    C
    =
    H(\widehat k).
    \label{eq:consensus_strength}
\end{equation}

This construction gives each independently trained parent equal total
weight, irrespective of the number of admissible descendants it
generates.
Hence, \(\widehat k\) identifies the symbolic granularity that is most
consistently supported across independently learned parent
representations, rather than the granularity that merely occurs most
often in the pooled reduction tree.

% When the consensus distribution \(H(k)\) is sharply concentrated,
% we retain the dominant layer \(|\Sigma_{\mathcal D}|=\widehat k\) for subsequent
% candidate selection.
% When the distribution is diffuse, we additionally inspect neighboring
% well-supported symbolic layers as a robustness analysis rather than
% interpreting the mode as a sharp estimate of the true minimum.
When the consensus distribution $H(k)$ is sharply concentrated, its dominant layer suggests a symbolic granularity to investigate. When the distribution is diffuse, several well-supported layers may warrant investigation rather than interpreting the mode as a sharp estimate of the minimum.

% The consensus estimate \(\widehat{k}\) should be interpreted as an
% empirical estimate of compact symbolic granularity rather than a
% certificate of global minimality.
% The fixed-point lower bound may be non-tight, and the same symbolic
% granularity can occur at different nonlinear capacities.
% Accordingly, candidate models selected from the consensus layer must
% still be retrained and evaluated for autonomous dynamical fidelity.
The consensus estimate $\widehat{k}$ characterizes the symbolic granularity supported by the parent hierarchies rather than certifying minimality. The fixed-point lower bound may be non-tight, and the same symbolic granularity can occur at different nonlinear capacities. Whether a quotient candidate can be realized by a reduced model must therefore be assessed through retraining and dynamical evaluation.

\paragraph{Consensus estimates by task.}

% \begin{table}[t]
% \centering
% \caption{Parent-balanced consensus $H_{\mathrm{macro}}(k)$ at the specified remaining nonlinear capacity $P_{\mathcal D}$. Votes are distinct admissible sets of ReLU units to linearize. Each contributing parent receives equal total weight; parents without admissible configurations at that capacity abstain. The parent column reports contributing parents relative to the parent cohort considered.}
%     \begin{tabular}{lccrclr}
%         \toprule
%         Task & $P_{\mathcal D}$ & Parents & Votes & $\widehat{k}$ & Leading layers $(k:H_{\mathrm{macro}}(k))$ & Top-3 mass \\
%         \midrule
%         Chua 3-scroll & 3 & 28/30 & 433 & 5 & $5:.370,\;7:.320,\;6:.221$ & $91.1\%$ \\
%         Chua 4-scroll & 4 & 11/30 & 67 & 7 & $7:.418,\;8:.247,\;10:.147$ & $81.2\%$ \\
%         Chua 5-scroll & 6 & 13/22 & 246 & 12 & $12:.135,\;9:.110,\;10:.110$ & $35.6\%$ \\
%         Dadras 4-wing & 3 & 20/29 & 338 & 6 & $6:.384,\;7:.247,\;8:.240$ & $87.1\%$ \\
%         \midrule
%         2-bit flip-flop & 4 & 27/30 & 241 & 12 & $12:.330,\;13:.196,\;16:.176$ & $70.3\%$ \\
%         3-bit flip-flop & 6 & 27/30 & 62 & 37 & $37:.120,\;39:.093,\;42:.083$ & $29.6\%$ \\
%         \bottomrule
%     \end{tabular}
%     \label{tab:consensus-tasks}
% \end{table}
\begin{table*}[t]
    \centering
    \caption{Parent-balanced consensus $H_{\mathrm{macro}}(k)$ at the specified remaining nonlinear capacity $P_{\mathcal D}$. Votes are distinct admissible linearization sets. Each contributing parent receives equal total weight; parents without admissible sets at that capacity abstain. The Parents column reports the number of contributing parents. Leading layers are listed in descending order of support. This consensus analysis is separate from the candidate selection used in the retraining experiments.}
    \label{tab:consensus-tasks}
    \small
    \setlength{\tabcolsep}{5pt}
    \begin{tabular}{lrrrrlc}
        \toprule
        Task & $P_{\mathcal D}$ & Parents & Votes & $\widehat{k}$ & Leading layers $(k:H_{\mathrm{macro}}(k))$ & Top-3 mass \\
        \midrule
        Chua 3-scroll & 3 & 28 & 433 & 5 & $5:.370,\;7:.320,\;6:.221$ & $91.0\%$ \\
        Chua 4-scroll & 4 & 11 & 67 & 7 & $7:.419,\;8:.247,\;10:.147$ & $81.2\%$ \\
        Chua 5-scroll & 6 & 13 & 111 & 11 & $11:.168,\;12:.166,\;13:.138$ & $47.3\%$ \\
        Dadras 4-wing & 3 & 20 & 338 & 6 & $6:.384,\;7:.247,\;8:.240$ & $87.1\%$ \\
        \midrule
        2-bit flip-flop & 4 & 27 & 241 & 12 & $12:.330,\;13:.196,\;16:.176$ & $70.3\%$ \\
        3-bit flip-flop & 6 & 27 & 62 & 37 & $37:.120,\;39:.093,\;42:.083$ & $29.6\%$ \\
        \bottomrule
    \end{tabular}
\end{table*}

Table~\ref{tab:consensus-tasks} summarizes the parent-balanced consensus at the specified remaining nonlinear capacities. For 3-scroll Chua, 4-scroll Chua, and Dadras, the modes agree with the observed minimal symbol counts of 5, 7, and 6, respectively, with over $81\%$ of the support concentrated in the three leading layers. For 5-scroll Chua, support is more dispersed: the top-three mass is $47.3\%$, and the mode is 11 rather than 9. The 9-symbol layer nevertheless receives $11.0\%$ support, compared with $16.8\%$ for the mode.

For memory tasks, the modes are 12 and 37, exceeding the reference FP-symbol counts of 9 and 27. Input-driven quotient graphs can include FP-free transition regions beyond the zero-input FP structure. The 9-symbol layer receives $2.8\%$ support for 2-bit memory; for 3-bit memory, the smallest supported count is 29.
% Consensus therefore describes the symbolic granularity supported by the parent hierarchies, while identifying minimal realizations requires retraining and dynamical evaluation.

These results show that consensus characterizes the quotient sizes supported by the parent representations. Both parent participation and support concentration inform its interpretation, while identifying the minimal realizations obtained after retraining requires additional dynamical evaluation.

% ============================================================
% C. TRAINING AND REDUCTION-GUIDED RETRAINING DETAILS
% ============================================================

\clearpage
\section{Training and Reduction-Guided Retraining Details}
\label{app:training_detail}

\subsection{Benchmark systems and data generation}
\label{app:data_generation}

% \paragraph{3-scroll Chua system.}
% The 3-scroll Chua trajectory used in our experiments was generated by
% the continuous-time system
% \begin{align}
%     \dot{x} &= 10\left(y - 0.05x - 0.6\sin x\right),\label{eq:chua-x}\\
%     \dot{y} &= x-y+z, \label{eq:chua-y}\\
%     \dot{z} &= -16y. \label{eq:chua-z}
% \end{align}
% We numerically integrated these equations using a time step of \(\Delta t=0.01\).
% The sinusoidal nonlinearity follows the multi-scroll Chua construction
% of \citet{tang2001generation}.

% The fixed points satisfy
% \[
%     y^*=0,\qquad z^*=-x^*,\qquad
%     0.05x^*+0.6\sin x^*=0.
% \]
% Within the region occupied by the target 3-scroll attractor, this
% equation yields five fixed points with
% \[
%     x^*\approx 0,\ \pm 3.4316,\ \pm 5.7806.
% \]
% These five fixed points constitute the target fixed-point set used in
% our analysis, and hence \(Q=5\).

\paragraph{3-scroll Chua system.}
We use the sinusoidal Chua system following the multiscroll construction of \citet{tang2001generation}:
\begin{align}
    \dot{x} &= 10\left(y-0.05x-0.6\sin x\right), \label{eq:chua-x}\\
    \dot{y} &= x-y+z, \label{eq:chua-y}\\
    \dot{z} &= -16y. \label{eq:chua-z}
\end{align}
We numerically integrate these equations with a time step of $\Delta t=0.01$. The fixed points satisfy
\begin{equation}
    y^*=0,\qquad z^*=-x^*,\qquad 0.05x^*+0.6\sin x^*=0.
    \label{eq:chua3-fp-condition}
\end{equation}
Within the region occupied by the target attractor, the five fixed points have
\begin{equation}
    x^*\approx 0,\ \pm 3.4316,\ \pm 5.7806.
    \label{eq:chua3-fp-locations}
\end{equation}
These constitute the target fixed-point set, giving $Q=5$ and a necessary nonlinear-capacity bound of $P\geq\lceil\log_2 5\rceil=3$ under the assumptions of Section \ref{sec:method}.

\paragraph{4- and 5-scroll Chua systems.}
We use the windowed-sine multiscroll construction of \citet{tang2001generation}:
\begin{align}
    \dot{x} &= \alpha\left(y-f(x)\right), \label{eq:alrnn-rev-chua-flow}\\
    \dot{y} &= x-y+z,\\
    \dot{z} &= -\beta y.
\end{align}
The nonlinearity is sinusoidal within a finite interval and linear outside it. With $g(x)=-b\sin(\pi x/(2a)+d)$ and $w=2ac$, we define
\begin{equation}
    f(x)=
    \begin{cases}
        g(-w)+g'(-w)(x+w), & x<-w,\\
        g(x), & |x|\leq w,\\
        g(w)+g'(w)(x-w), & x>w.
    \end{cases}
    \label{eq:alrnn-rev-chua-window}
\end{equation}
The value and slope are continuous at both interval boundaries. We set $(\alpha,\beta,a,b)=(10.814,14,1.3,0.11)$ for both systems, with $(c,d)=(3,0)$ for 4-scroll Chua and $(c,d)=(4,\pi)$ for 5-scroll Chua. The target fixed points are
\begin{equation}
    (x_k^*,y_k^*,z_k^*)=(x_k,0,-x_k),\qquad x_k=2a(k-d/\pi),\qquad k\in\mathbb Z,\quad |x_k|\leq w.
    \label{eq:chua45-fp-locations}
\end{equation}
Including the boundary roots gives $Q=7$ and $Q=9$ for the 4- and 5-scroll systems, respectively. The corresponding necessary nonlinear-capacity bounds are $P\geq3$ and $P\geq4$ under the assumptions of Section \ref{sec:method}.

\paragraph{Dadras system.}
We use the Dadras system \citep{dadras2009novel}:
\begin{align}
    \dot{x} &= y-3x+2.7yz, \label{eq:alrnn-rev-dadras-flow}\\
    \dot{y} &= 1.7y-xz+z,\\
    \dot{z} &= 2xy-9z.
\end{align}
The target fixed-point set contains the origin and 4 nonzero fixed points, giving $Q=5$ and a necessary nonlinear-capacity bound of $P\geq3$ under the assumptions of Section \ref{sec:method}. For the fixed-point position analysis, we label the origin O and the remaining points W1--W4 in ascending lexicographic order of their standardized coordinates. In particular, W2 has coordinates $(-2.3107,2.6256,-1.3482)$ before standardization and $(-0.9726,1.0871,-1.0419)$ afterward.

\paragraph{Flip-flop memory tasks.}
We consider $b\in\{2,3\}$ memory channels following the flip-flop construction of \citet{sussillo2013opening}. The AL-RNN receives an external input through
\begin{equation}
\label{eq:alrnn-rev-driven-alrnn}
\mathbf z_{t+1}=\mathbf A\mathbf z_t+\mathbf W\bm\phi_P(\mathbf z_t)+\mathbf h+\mathbf C\mathbf u_t,
\qquad
\mathbf o_t=(z_{t+1,1},\ldots,z_{t+1,b})^\top.
\end{equation}
The recurrent parameters have the roles defined in Sec.~2.1, and $\mathbf C\in\mathbb R^{M\times b}$ is a trainable input matrix.

Independently at each time step and for each channel, the input is $+1$ or $-1$ with probability $0.01$ each and zero with probability $0.98$. Pulses last one step, and multiple channels may receive a pulse simultaneously. Starting from a uniformly sampled memory vector $\mathbf y_{-1}\in\{-1,+1\}^b$, the target is
\begin{equation}
\label{eq:alrnn-rev-memory-target}
y_{tj}=
\begin{cases}
u_{tj}, & u_{tj}\ne0,\\
y_{t-1,j}, & u_{tj}=0,
\end{cases}
\qquad j=1,\ldots,b.
\end{equation}
Thus, each channel must retain the value written by its most recent pulse. The first $b$ coordinates of $\mathbf z_0$ are set to $\mathbf y_{-1}$, and the remaining coordinates are drawn independently from $\mathcal N(0,0.1^2)$. The output $\mathbf o_t$ is compared with $\mathbf y_t$, so a pulse updates the target on the same step on which $\mathbf z_{t+1}$ is read out. Memory-specific losses and evaluation procedures are given in Appendices~\ref{app:flipflop_details} and \ref{app:rev-structural-diagnostics}.

\paragraph{Trajectory generation and normalization.}
For these three systems, we integrate the full 3-dimensional state using classical RK4 with step $0.005$ and record every second step, giving a sampling interval of $0.01$. The initial states are $(0.5,0.1,0)$ for the two Chua systems and $(1,1,1)$ for Dadras. After discarding 500 time units, we record one noiseless trajectory. The first $400{,}000$ samples are used for training and the next $100{,}000$ for testing. Each coordinate is standardized using the training mean and standard deviation. We apply the same transformation to the test trajectory and reference fixed points. The three-scroll Chua data described above are used in raw coordinates.

\subsection{Parent and direct-training configurations}
\label{app:parent_training}

For all tasks, we use latent-state dimension $M=20$. The readout dimension is $N=3$ for attractor-reconstruction tasks and $N=n$ for $n$-bit memory tasks. Parent models have $P_{\mathrm{parent}}=10$ ReLU units. The ReLU units occupy the final $P_{\mathrm{parent}}$ latent coordinates; the remaining coordinates are linear. Direct-training baselines retain the same latent dimension and training protocol as the corresponding parents, varying only the number of ReLU units $P$. We train 30 independent seeds per configuration.

\paragraph{Attractor-reconstruction tasks.}
All four systems use the same training settings: $2{,}000$ epochs with 50 optimization steps per epoch, each using a batch of 16 trajectory segments of length 200. We use RAdam \citep{liu2020radam} with an exponentially decaying learning rate from $10^{-3}$ to $10^{-5}$ and clip gradients to a maximum norm of 10. No noise or target-fixed-point residual term is included during training.

Training follows the sparse teacher-forcing protocol of \citet{brenner2024almostlinear}. The readout coordinates are overwritten by the corresponding target values at the beginning of each segment and every 128 steps thereafter, with forcing strength 1.0.

\paragraph{Flip-flop memory tasks.}
For both 2- and 3-bit memory, parent and direct models are trained for $3{,}000$ optimization steps using batches of 64 input sequences of length 200. We use Adam \citep{kingma2015adam} with an initial learning rate of $10^{-3}$, cosine annealing to $10^{-5}$, and gradient clipping at a maximum norm of 1. Models receive the task input sequences without intermediate state synchronization. The final training iterate is used for evaluation.

\paragraph{Checkpoint selection for attractor-reconstruction tasks.}
Every 10 epochs, we evaluate a $10{,}000$-step autonomous rollout. For all parent models, checkpoint selection is label-free with respect to the target fixed-point count: we retain the checkpoint minimizing the state-space discrepancy $E_{\mathrm{stsp}}$ evaluated over the full rollout. Subsequent parent screening and retraining checkpoint selection using the parent-derived count $\widehat Q$ are described in Appendix~\ref{app:selection_experiments}.

\subsection{Reduced-model initialization and optimization}
\label{app:student_initialization}

For each reduction candidate specified by the linearization set $\mathcal D$, we copy all trainable parameters from the selected parent checkpoint:
\begin{equation}
    \theta_{\mathcal D}^{(0)}=\theta_{\mathrm{parent}}.
    \label{eq:student_initialization_appendix}
\end{equation}
The selected ReLU activations are gradually changed to the identity according to \Eqref{eq:alrnn-rev-annealing}, while $\mathcal D$ remains fixed throughout retraining. Candidates derived from the same parent therefore start with identical trainable parameters and differ only in the selected linearization set. The optimizer state is initialized from scratch for each run, and each candidate is retrained once with a fixed random seed.

\paragraph{Attractor-reconstruction tasks.}
Reduced candidates are retrained for $2{,}000$ epochs, with 50 optimization steps per epoch and batches of 16 trajectory segments of length 200. As in parent and direct training, we use RAdam with an exponentially decaying learning rate from $10^{-3}$ to $10^{-5}$ and clip gradients to a maximum norm of 10. Gradual linearization takes place over the first 200 epochs.
% The ablations additionally consider abrupt linearization and continuation over 500 epochs.
The ablations additionally consider abrupt linearization.

Retrained checkpoints are selected by minimizing
\begin{equation}
    S_{\mathrm{ckpt}}^{\mathrm{retrain}}
    =E_{\mathrm{stsp}}+5\left|Q_{\mathrm{vis}}-\widehat Q\right|,
    \label{eq:retrain_checkpoint_selection}
\end{equation}
where $\widehat Q$ is estimated from the selected parents as described in Appendix~\ref{app:selection_experiments}. Only checkpoints at which the selected activations have become fully linear are eligible.

\paragraph{Flip-flop memory tasks.}
Reduced candidates are retrained for $2{,}000$ optimization steps using Adam, with batches of 64 input sequences of length 200. The learning rate decreases from $10^{-3}$ to $10^{-5}$ using cosine annealing, and gradients are clipped to a maximum norm of 1. Gradual linearization takes place over the first 200 optimization steps. The final training iterate is used for evaluation.

\subsection{Autonomous teacher trajectory}
\label{app:teacher_trajectory}

For attractor-reconstruction tasks, teacher targets for all teacher-guided retraining conditions are extracted from a single continuous autonomous rollout of the frozen parent model.
No teacher forcing, raw-data overwrite, or segment reset is applied
during this trajectory.

The teacher rollout is initialized from the first sample of the
training trajectory using the same latent-state initialization
procedure as in standard autonomous evaluation.
The resulting teacher trajectory contains \(400{,}000\) time steps.

For every time step, we store the parent latent state, readout, and
activation pattern of the original nonlinear units.
A single cache is constructed for each source parent and shared across
all reduction candidates, retraining conditions, and retraining seeds
derived from that parent.

The same cache also supplies the initial student state for each
training segment and the reference states used for sparse
synchronization.
Specifically, for a segment beginning at teacher time \(s\), the
student is initialized as
\begin{equation}
    \mathbf{z}^{\mathcal D}_s
    =
    \mathbf{z}^{\mathrm{parent}}_s,
\end{equation}
and subsequent synchronization events use the corresponding parent
states from the same cached trajectory.

\subsection{Phase-aligned student rollouts}
\label{app:phase_alignment}

During teacher-guided retraining, a random segment start \(s\) is
sampled from the autonomous teacher cache and the student is initialized
from the corresponding teacher state:
\begin{equation}
    \mathbf{z}_s^{\mathrm{student}}
    =
    \mathbf{z}_s^{\mathrm{teacher}}.
    \label{eq:phase_alignment_appendix}
\end{equation}
The student is then rolled out autonomously for
$
    L=200
$
steps.
The teacher targets are taken from the same free-running teacher
trajectory at
\begin{equation}
    s+1,\ldots,s+L.
\end{equation}

The phase-alignment procedure is important for chaotic dynamics.
Because nearby chaotic trajectories quickly diverge in phase, teacher
targets must be taken from the same parent trajectory used to initialize the student segment.
Otherwise, the student would be trained to match parent states
corresponding to a different phase of the attractor.

\subsection{Periodic parent-state synchronization}
\label{app:parent_sync}

For attractor-reconstruction tasks, each training segment starts from the corresponding parent state, and the reduced model is periodically synchronized with the parent during retraining. This follows the general idea of sparse teacher forcing in standard AL-RNN training \citep{brenner2024almostlinear}, but uses the autonomous parent trajectory rather than observed data as the reference.

Synchronization is applied every 128 steps. Full-preactivation guidance and the main combined-guidance protocol synchronize the entire latent state:
\begin{equation}
    \mathbf z^{\mathcal D}_{t}\leftarrow\mathbf z^{\mathrm{parent}}_{t}.
    \label{eq:parent-sync-full}
\end{equation}
Each 200-step segment therefore contains one synchronization event, 128 steps after initialization. Between synchronization events, the reduced model evolves autonomously.

For retained-preactivation guidance, synchronization is restricted to the readout coordinates and the retained ReLU coordinates:
\begin{equation}
    \mathcal R_{\mathrm{sync}}=\{1,\ldots,N\}\cup\mathcal R,
    \label{eq:parent-sync-retained-set}
\end{equation}
\begin{equation}
    \mathbf z^{\mathcal D}_{t,\mathcal R_{\mathrm{sync}}}\leftarrow\mathbf z^{\mathrm{parent}}_{t,\mathcal R_{\mathrm{sync}}},
    \label{eq:parent-sync-retained}
\end{equation}
where $\mathcal R$ denotes the retained ReLU coordinates. The remaining coordinates continue to evolve freely. Symbol guidance and the output-matching-only control synchronize only the readout coordinates.

Within each objective, abrupt and gradual linearization use the same synchronization rule. No observed-data overwrite is applied during retraining. For memory tasks, the parent and reduced model receive the same initial state and input sequence, without intermediate state synchronization.

\subsection{Retraining objectives}
\label{app:retraining_detail}

We consider five retraining conditions:
\begin{enumerate}
    \item no auxiliary guidance;
    \item full-preactivation guidance;
    \item retained-preactivation guidance;
    \item symbol guidance;
    \item full-preactivation + symbol guidance.
\end{enumerate}

% The full-preactivation condition with fixed coordinate-equal weighting and full-state synchronization is used as the main
% reduction-guided retraining protocol.
For the autonomous reconstruction results in the main text,
we combine full-preactivation and symbol guidance with fixed
coordinate-equal weighting of the squared-error terms.
The individual guidance conditions are evaluated as ablations in
Appendix~\ref{app:chua_guidance_ablation}.

% \paragraph{No teacher guidance.}

% The no-guidance condition retains the parent-derived parameter initialization, the fixed linearization set, and the autonomous parent trajectory during retraining.
% % Instead, it is trained on the original observed trajectory using the same output objective and sparse teacher-forcing procedure as standard AL-RNN training.
% % This condition therefore isolates the effect of parent-trajectory guidance from that of parent initialization and structural reduction.
\paragraph{No auxiliary guidance.}
This condition retains the parent-derived parameter initialization, the fixed linearization set, and the autonomous parent trajectory during retraining. It minimizes only the parent-readout loss $\mathcal L_{\mathrm{out}}$ below, without auxiliary state or symbol losses.

\paragraph{Output matching.}

All teacher-guided conditions use the readout-matching loss
\begin{equation}
    \mathcal{L}_{\mathrm{out}}
    =
    \frac{1}{TN}
    \sum_{\ell=1}^{T}
    \left\|
        \mathbf{x}^{\mathcal D}_{s+\ell}
        -
        \mathbf{x}^{\mathrm{parent}}_{s+\ell}
    \right\|_2^2,
    \label{eq:lout_appendix}
\end{equation}
where \(s\) denotes the starting time of the sampled segment in the
cached autonomous parent trajectory.

\paragraph{Full-preactivation guidance.}

The full-preactivation objective additionally matches all non-readout
latent coordinates:
\begin{equation}
    \mathcal{L}_{\mathrm{full}}
    =
    \frac{1}{T(M-N)}
    \sum_{\ell=1}^{T}
    \left\|
        \mathbf{z}^{\mathcal D}_{s+\ell,N+1:M}
        -
        \mathbf{z}^{\mathrm{parent}}_{s+\ell,N+1:M}
    \right\|_2^2.
    \label{eq:lfull_appendix}
\end{equation}
The total loss is
\begin{equation}
    \mathcal{L}
    =
    \mathcal{L}_{\mathrm{out}}
    +
    \lambda_{\mathrm{full}}
    \mathcal{L}_{\mathrm{full}}.
    \label{eq:full_total_loss}
\end{equation}
For the main protocol, we use the fixed coordinate-equal weighting
\begin{equation}
    \lambda_{\mathrm{full}}
    =
    \frac{M-N}{N}.
    \label{eq:lambda_full}
\end{equation}
Because both losses are normalized by their respective numbers of
coordinates, this choice gives each readout and non-readout coordinate
the same weight in the total squared-error objective.
For \(M=20\) and \(N=3\),
\(\lambda_{\mathrm{full}}=17/3\).

\paragraph{Retained-preactivation guidance.}

Let \(\mathcal R\) denote the retained ReLU coordinates after reduction.
The retained-preactivation loss is
\begin{equation}
    \mathcal{L}_{\mathrm{ret}}
    =
    \frac{1}{T|\mathcal{R}|}
    \sum_{\ell=1}^{T}
    \left\|
        \mathbf{z}^{\mathcal D}_{s+\ell,\mathcal R}
        -
        \mathbf{z}^{\mathrm{parent}}_{s+\ell,\mathcal R}
    \right\|_2^2,
    \label{eq:lret_appendix}
\end{equation}
with total loss
\begin{equation}
    \mathcal{L}
    =
    \mathcal{L}_{\mathrm{out}}
    +
    \lambda_{\mathrm{ret}}
    \mathcal{L}_{\mathrm{ret}}.
\end{equation}
Coordinates corresponding to linearized ReLU units and originally
linear latent variables receive no auxiliary state-matching loss.
Thus, this objective constrains only the nonlinear coordinates retained
after reduction while allowing the remaining linear latent coordinates to reorganize during retraining.

\paragraph{Symbol guidance.}

For each retained ReLU coordinate \(j\in\mathcal R\), define the parent
activation label
\begin{equation}
    b^{\mathrm{parent}}_{t,j}
    =
    \mathbb{I}
    \left[
        z^{\mathrm{parent}}_{t,j}>0
    \right],
    \qquad
    y_{t,j}
    =
    2b^{\mathrm{parent}}_{t,j}-1.
\end{equation}
We use the smooth margin loss
\begin{equation}
    \mathcal{L}_{\mathrm{sym}}
    =
    \frac{1}{T|\mathcal{R}|}
    \sum_{\ell=1}^{T}
    \sum_{j\in\mathcal R}
    \operatorname{softplus}
    \left(
        m
        -
        y_{s+\ell,j}
        z^{\mathcal D}_{s+\ell,j}
    \right),
    \label{eq:lsym_appendix}
\end{equation}
with margin
\begin{equation}
    m=1.
\end{equation}
The total loss is
\begin{equation}
    \mathcal{L}
    =
    \mathcal{L}_{\mathrm{out}}
    +
    \lambda_{\mathrm{sym}}
    \mathcal{L}_{\mathrm{sym}}.
\end{equation}

\paragraph{Full-preactivation + symbol guidance.}
Full-preactivation guidance penalizes differences in the non-readout
latent coordinates. We add symbol guidance to place explicit emphasis
on agreement in the signs of the retained ReLU coordinates.
The combined objective is
\begin{equation}
\label{eq:alrnn-rev-combined-guidance}
 \mathcal L=\mathcal L_{\rm out}
   +\lambda_{\rm full}\mathcal L_{\rm full}+\lambda_{\rm sym}\mathcal L_{\rm sym}.
\end{equation}

A systematic comparison of these retraining variants is reported in
Section~\ref{app:chua_guidance_ablation}.

\subsection{Memory-specific training and evaluation details}
\label{app:flipflop_details}

The driven model and target sequence are defined in Equations \ref{eq:alrnn-rev-driven-alrnn}-\ref{eq:alrnn-rev-memory-target}. Parent and direct-training settings are given in Appendix~\ref{app:parent_training}, and reduced-model initialization and optimization in Appendix~\ref{app:student_initialization}.

\paragraph{Parameter initialization and task loss.}
For parent and direct models, the diagonal entries of $\mathbf A$ are initialized independently from $\mathcal U(0.7,0.9)$. Entries of $\mathbf W$ are drawn from $\mathcal N(0,1)/\sqrt M$, and the matrix is iteratively rescaled to a spectral radius at most $0.95$. We initialize $\mathbf h=0$ and the entries of $\mathbf C$ from $\mathcal N(0,0.5^2)$. Direct-training comparisons use $P=1,\ldots,10$. A new input batch and initial states are generated at every optimization step.

To allow a response to each pulse, we exclude the pulse step and the next three steps on that channel from task loss and accuracy. Let $m_{tj}=0$ on these steps and $m_{tj}=1$ otherwise; overlapping response intervals remain excluded. For a batch of $B$ sequences of length $T$, the task loss is
\begin{equation}
\label{eq:rev-ff-task-loss}
\mathcal L_{\mathrm{task}}
=
\frac{\sum_{n=1}^{B}\sum_{t=0}^{T-1}\sum_{j=1}^{b}m_{ntj}(o_{ntj}-y_{ntj})^2}
{\sum_{n=1}^{B}\sum_{t=0}^{T-1}\sum_{j=1}^{b}m_{ntj}}.
\end{equation}

\paragraph{Reduction and retraining objectives.}
Reduction protects all admissible zero-input FP symbols of the parent, including those containing unstable fixed points or lying outside the visited regions. Candidate linearization sets must keep these FP symbols distinct. This constraint concerns the parent's symbolic structure and does not guarantee preservation of fixed-point locations or stability after retraining.

During retraining, the inherited input matrix $\mathbf C$ remains trainable, while the parent parameters are frozen. Teacher states are generated using the same initial state and input sequence as the reduced model, without intermediate synchronization (Appendix~\ref{app:parent_sync}). The objective combines the task loss with parent-derived guidance.

Let $\mathcal R$ denote the retained ReLU coordinates. The state and symbol losses are
\begin{equation}
\label{eq:rev-ff-full-loss}
\mathcal L_{\mathrm{state}}^{\mathrm{FF}}
=
\frac{1}{BTM}
\sum_{n=1}^{B}\sum_{t=0}^{T-1}\sum_{j=1}^{M}
\left(z^{\mathcal D}_{n,t+1,j}-z^{\mathrm{parent}}_{n,t+1,j}\right)^2
\end{equation}
and
\begin{equation}
\label{eq:rev-ff-symbol-loss}
\mathcal L_{\mathrm{sym}}^{\mathrm{FF}}
=
\frac{1}{BT|\mathcal R|}
\sum_{n=1}^{B}\sum_{t=0}^{T-1}\sum_{j\in\mathcal R}
\max\!\left(0,1-\operatorname{sign}(z^{\mathrm{parent}}_{n,t+1,j})z^{\mathcal D}_{n,t+1,j}\right).
\end{equation}
The main combined objective is
\begin{equation}
\label{eq:rev-ff-combined-loss}
\mathcal L^{\mathrm{FF}}
=
\mathcal L_{\mathrm{task}}
+\frac{M-b}{b}\mathcal L_{\mathrm{state}}^{\mathrm{FF}}
+\mathcal L_{\mathrm{sym}}^{\mathrm{FF}}.
\end{equation}
Unlike the attractor-reconstruction objective in Appendix~\ref{app:retraining_detail}, the state loss averages over all $M$ coordinates, including the readout, and the symbol term is a hinge loss. The coefficient $(M-b)/b$ therefore does not give coordinate-equal weighting between the task and state losses. The response exclusions $m_{ntj}$ apply only to the task loss; both auxiliary terms use the full sequence.

\paragraph{Zero-input fixed points.}
We set $\mathbf u=0$ and enumerate all activation regions, including unvisited ones. For a region $\boldsymbol{\sigma}$, let $\mathbf D_{\boldsymbol{\sigma}}$ have diagonal entry one on linear coordinates and the corresponding activation bit on ReLU coordinates. The regional matrix and candidate fixed point satisfy
\begin{equation}
\label{eq:rev-ff-fp}
\mathbf F_{\boldsymbol{\sigma}}=\mathbf A+\mathbf W\mathbf D_{\boldsymbol{\sigma}},
\qquad
(\mathbf I-\mathbf F_{\boldsymbol{\sigma}})\mathbf z^*=\mathbf h.
\end{equation}
A candidate is retained only if the system is nonsingular and the solution has the assumed activation pattern. A coordinate is active when strictly positive; zero is assigned to the inactive side. Singular systems are skipped. Reduced models enumerate $2^{P_{\mathcal D}}$ regions, with linearized coordinates assigned diagonal entry one throughout. A retained fixed point is classified as stable when
\begin{equation}
\label{eq:rev-ff-stability}
\rho(\mathbf F_{\boldsymbol{\sigma}})<1.
\end{equation}
These implementation rules do not by themselves verify the interiority assumptions required for the theoretical fixed-point bound.

Task accuracy, fixed-point recovery metrics, and success criteria are defined in Appendix~\ref{app:rev-structural-diagnostics}.

\clearpage
\section{Evaluation Protocol and Statistics}
\label{app:evaluation}

% \subsection{Autonomous free-run evaluation}
% \label{app:free_run_evaluation}

% Each selected model is evaluated using an autonomous rollout of
% \(10{,}000\) time steps without teacher forcing or external state
% correction.
% The first \(1{,}000\) steps are discarded as transient for the
% structural evaluation.
% Both the visited fixed-point count \(Q_{\mathrm{vis}}\) and the
% post-transient symbolic complexity \(|\widehat{\Sigma}|\) are computed
% from the remaining \(9{,}000\) steps.
% The fidelity metrics \(E_{\mathrm{stsp}}\) and \(E_H\) are computed
% from the full \(10{,}000\)-step rollout.

% We define the post-transient visited-symbol set as
% %
% \begin{equation}
%     \widehat{\Sigma}
%     =
%     \left\{
%         \boldsymbol{\sigma}(\mathbf z_t)
%         :
%         1001 \leq t \leq 10000
%     \right\},
%     \label{eq:posttransient_symbols_appendix}
% \end{equation}
% %
% whose cardinality
% \(|\widehat{\Sigma}|\)
% measures the symbolic complexity realized by the long-run autonomous
% dynamics.

% The quantity \(Q_{\mathrm{vis}}\) denotes the number of admissible
% fixed points of the learned model whose symbolic regions are visited
% during the full \(10{,}000\)-step autonomous rollout.

\subsection{Autonomous free-run evaluation}
\label{app:free_run_evaluation}

Each selected model is evaluated using an autonomous rollout of
\(10{,}000\) time steps without teacher forcing or external state
correction.
The first \(1{,}000\) steps are discarded as transient for all final
evaluation quantities, including \(Q_{\mathrm{vis}}\),
\(|\widehat{\Sigma}|\), \(E_{\mathrm{stsp}}\), and \(E_{\mathrm H}\).

We define the post-transient visited-symbol set as
\begin{equation}
    \widehat{\Sigma}
    =
    \left\{
        \boldsymbol{\sigma}(\mathbf z_t)
        :
        1001 \leq t \leq 10000
    \right\},
    \label{eq:posttransient_symbols_appendix}
\end{equation}
whose cardinality
\(|\widehat{\Sigma}|\)
measures the symbolic complexity realized by the long-run autonomous
dynamics.

The visited fixed-point count \(Q_{\mathrm{vis}}\) is computed over the
same post-transient interval and counts admissible fixed points of the
learned model whose symbolic regions are visited during that interval.

\subsection{State-space and spectral fidelity}
\label{app:fidelity_metrics}

We use two complementary measures of autonomous dynamical fidelity, following the prior AL-RNN work \citep{brenner2024almostlinear}.
Both metrics are computed from the post-transient portion of the
autonomous rollout, after discarding the first \(1{,}000\) steps.

The state-space error is the Kullback--Leibler divergence
\begin{equation}
    E_{\mathrm{stsp}}
    =
    D_{\mathrm{KL}}
    \left(
        \widehat p_{\mathrm{target}}
        \,\|\,
        \widehat p_{\mathrm{gen}}
    \right),
    \label{eq:estsp_appendix}
\end{equation}
where \(\widehat p_{\mathrm{target}}\) and
\(\widehat p_{\mathrm{gen}}\) are histogram-based empirical
state-space distributions in the observed coordinates.
We use \(30\) bins per dimension over the coordinate ranges of the
target trajectory, with additive smoothing of \(10^{-5}\) applied to
each histogram cell before normalization.
Generated samples outside the target histogram range are excluded.
The target distribution is estimated from the full training
trajectory, whereas the generated distribution is estimated from the
post-transient \(9{,}000\)-step autonomous trajectory.

% As a complementary measure of temporal structure, we compute the
% spectral Hellinger distance \(E_{\mathrm H}\).
% For each observed coordinate, the post-transient time series is
% standardized, and its power spectrum is smoothed with a Gaussian
% kernel with a standard deviation of \(20\) frequency bins and
% normalized to unit mass.
% The Hellinger distance between the generated and target spectra is
% computed as
% %
% \begin{equation}
%     D_{\mathrm H}(p,q)
%     =
%     \sqrt{
%         1-
%         \sum_f
%         \sqrt{p_f q_f}
%     },
% \end{equation}
% %
% where \(p_f\) and \(q_f\) are the normalized power-spectrum values of
% the target and generated trajectories, respectively, at frequency bin
% \(f\), and the sum runs over all frequency bins.
% The spectral error \(E_{\mathrm H}\) is then obtained by averaging
% \(D_{\mathrm H}(p,q)\) across observed coordinates.

% For each observed coordinate, the generated post-transient segment
% (9{,}000 steps) and a target segment of equal length, taken as the
% first \(9{,}000\) samples of the training trajectory, are
% standardized to zero mean and unit variance, and their power spectra
% are computed by a real FFT.
% Using segments of identical length ensures that both spectra live on
% the same frequency grid, so no interpolation or bin matching is
% required.
% Each spectrum is smoothed with a Gaussian kernel of width
% \(\sigma = 20\) frequency bins and normalized to unit sum, the
% Hellinger distance between the two normalized spectra is computed per
% coordinate, and \(E_{\mathrm{H}}\) is the average over the \(N\)
% observed coordinates.

As a complementary measure of temporal structure, we compute the
spectral Hellinger distance \(E_{\mathrm H}\).
For each observed coordinate, we compare the generated post-transient
segment of \(9{,}000\) steps with a target segment of equal length,
taken as the first \(9{,}000\) samples of the training trajectory.
Each segment is standardized to zero mean and unit variance, and its
power spectrum is computed using a real FFT.
Because the generated and target segments have identical lengths, their
spectra share the same frequency grid and require no interpolation or
bin matching.
Each spectrum is smoothed with a Gaussian kernel with standard
deviation \(\sigma=20\) frequency bins and normalized to unit mass.

The Hellinger distance between the normalized target and generated
spectra is
\begin{equation}
    D_{\mathrm H}(p,q)
    =
    \sqrt{
        1-
        \sum_f
        \sqrt{p_f q_f}
    },
\end{equation}
where \(p_f\) and \(q_f\) denote the normalized target and generated
power spectra, respectively, at frequency bin \(f\).
The spectral error \(E_{\mathrm H}\) is obtained by averaging
\(D_{\mathrm H}(p,q)\) over the \(N\) observed coordinates.

\subsection{Minimal realization criteria}
\label{app:minimal_eval}

In the empirical results, minimality refers to the smallest realizations observed under the stated experimental and evaluation conditions.

For attractor-reconstruction tasks, we distinguish fixed-point-count recovery from minimal-realization success:
\begin{equation}
\label{eq:rev-fp-success}
\mathcal S_{\mathrm{FP}}=\mathbb I[Q_{\mathrm{vis}}=Q_{\star}],
\end{equation}
\begin{equation}
\label{eq:minimal_success}
% \label{eq:rev-count-success}
\mathcal S_{\mathrm{min}}=\mathbb I[Q_{\mathrm{vis}}=Q_{\star}\;\land\;|\widehat{\Sigma}|=K_{\star}],
\end{equation}
where $Q_{\star}$ is the target fixed-point count and $K_{\star}$ is the reference symbol count. We use $(Q_{\star},K_{\star})=(5,5)$, $(7,7)$, $(9,9)$, and $(5,6)$ for 3-, 4-, and 5-scroll Chua and Dadras, respectively. The Dadras criterion assesses the observed six-symbol realization without asserting that six is the theoretical minimum.

These criteria use the counting conventions in Appendix~\ref{app:free_run_evaluation}. Correspondence with the target fixed-point locations is assessed separately, as are the state-space error $E_{\mathrm{stsp}}$ and spectral Hellinger distance $E_{\mathrm H}$; neither error enters the minimal-realization criterion. Interpreting a realization as attaining the theoretical lower bound additionally requires the assumptions of Section \ref{sec:method}, including target fixed-point preservation and a common visited-region partition.

For memory tasks, we use the memory-state recovery criterion defined in Appendix~\ref{app:rev-structural-diagnostics}, requiring exactly $2^n$ stable zero-input fixed points, with one within distance $0.3$ of each prescribed memory corner. Task accuracy and recovery of the reference saddle structure are evaluated separately.

\subsection{Parent-balanced statistical aggregation}
\label{app:seed_statistics}

Different source parents can generate different numbers of selected
reduction candidates.
Pooling all candidates would therefore give greater weight to parent
models that happen to generate more descendants.
We instead treat the independently trained source parent as the primary
statistical unit.

Let \(\mathcal C_r\) denote the set of candidates selected
\emph{before retraining} from parent seed \(r\).
For a binary success indicator \(S\), we define the within-parent
success fraction as
\begin{equation}
    q_r(S)
    =
    \frac{1}{|\mathcal C_r|}
    \sum_{c\in\mathcal C_r}
    S(c),
    \qquad
    |\mathcal C_r|>0.
    \label{eq:q_seed_appendix}
\end{equation}
Candidates that fail during retraining or evaluation
remain in \(\mathcal C_r\) and are assigned \(S(c)=0\).
If a parent produces no candidate satisfying the selection criterion,
we set
\begin{equation}
    q_r(S)=0.
\end{equation}

Let \(\mathcal P\) denote the set of the \(30\) independently trained
parent seeds.
The parent-balanced seed-macro success rate is
\begin{equation}
    \bar q(S)
    =
    \frac{1}{|\mathcal P|}
    \sum_{r\in\mathcal P}
    q_r(S).
    \label{eq:q_macro_appendix}
\end{equation}
Thus, each independently trained parent contributes equal total weight,
irrespective of the number of reduction candidates it generates.

For continuous candidate-level quantities such as
\(E_{\mathrm{stsp}}\), each successfully evaluated candidate retains
the weight determined by the number of candidates originally selected
from its parent.
Let
\(\mathcal C_r^{\mathrm{eval}}\subseteq\mathcal C_r\)
denote the candidates from parent \(r\) for which a valid continuous
measurement is available.
The corresponding parent-balanced empirical cumulative distribution is
\begin{equation}
    F_{\mathrm{bal}}(e)
    =
    \frac{
        \displaystyle
        \sum_{r\in\mathcal P}
        \sum_{c\in\mathcal C_r^{\mathrm{eval}}}
        \frac{1}{|\mathcal C_r|}
        \mathbb I[E_{r,c}\leq e]
    }{
        \displaystyle
        \sum_{r\in\mathcal P}
        \frac{
            |\mathcal C_r^{\mathrm{eval}}|
        }{
            |\mathcal C_r|
        }
    }.
    \label{eq:parent_balanced_ecdf}
\end{equation}
%
% Candidates without a valid continuous measurement are omitted without
% renormalizing the weights of their successfully evaluated siblings.
% Consequently, a parent's realized weight is proportional to the
% fraction of its originally selected candidates that yield valid
% measurements.
Candidates with missing or non-finite continuous measurements are
excluded from both the numerator and denominator, without
renormalizing the weights of their successfully evaluated siblings.

Directly trained models and higher-capacity parent models contribute
one observation per independent seed and therefore require no
additional candidate-level weighting.

For a capacity-specific comparison, let $\mathcal C_r(P)$ be the selected candidates from parent $r$ with exactly $P$ remaining ReLUs, and let $S_{rc}$ be the relevant success indicator. The within-parent fraction at this capacity is
\begin{equation}
\label{eq:rev-capacity-parent}
 f_r(P)=
 \begin{cases}
 |\mathcal C_r(P)|^{-1}\sum_{c\in\mathcal C_r(P)}S_{rc},
       &|\mathcal C_r(P)|>0,\\
 0,&|\mathcal C_r(P)|=0 .
 \end{cases}
\end{equation}
We report the parent-balanced mean $30^{-1}\sum_{r=1}^{30}f_r(P)$. This recomputes the weights within the specified capacity. It differs from retaining weights calculated over all capacities and reporting only the contributions from candidates at $P$.

For memory, we also report the number of initial parents yielding at least one successful candidate at capacity $P$,
\begin{equation}
\label{eq:rev-any-success-parent}
 N_{\rm success}(P)
 =\sum_{r=1}^{30}
 \mathbf 1\!\left[\exists c\in\mathcal C_r(P):S_{rc}=1\right].
\end{equation}
This counts each parent at most once but allows multiple candidate retraining attempts. The candidate-level fraction pools individual models and uses a different denominator. Candidates from one parent are not independent parent seeds.

The guidance ablations report within-parent success fractions averaged over all 30 initial parents, including zero for parents yielding no selected candidate. The off-tree controls instead report raw per-run success: the number of successful retraining runs divided by the number evaluated in that arm. These summaries need not coincide when parents contribute different numbers of candidates. Positional-recovery rates measure the fraction of reference fixed points recovered, rather than the fraction of models recovering every reference point. Each figure specifies whether the positional and error diagnostics include all produced models or a success-conditioned subset.

\subsection{Structural and positional diagnostics}
\label{app:rev-structural-diagnostics}

The count-based fixed-point recovery rate in \Figref{fig:appendix-direct-count-position_attractor}(a) is the fraction of runs satisfying $Q_{\mathrm{vis}}=Q$, where $Q_{\mathrm{vis}}$ counts admissible fixed points of the learned model whose regions are visited during the evaluation rollout, and $Q$ is the target fixed-point count. This criterion does not require positional matching or a particular number of visited symbols. Agreement in the counts therefore does not by itself establish correspondence with the target fixed-point locations.
% The visitation conventions are specified in Appendix~\ref{app:rev-counts}.

Let $\{\mathbf x_q^\star\}_{q=1}^{Q}$ denote the target fixed points in the observation coordinates used for evaluation. We define their minimum pairwise separation as
\begin{equation}
\label{eq:target-fp-separation}
d_{\min}=\min_{q\ne q'}\|\mathbf x_q^\star-\mathbf x_{q'}^\star\|_2.
\end{equation}
The target and model fixed points use the same coordinate transformation as the evaluated trajectories.

For the positional analysis in Figures \ref{fig:attractor_results}(c) and \ref{fig:appendix-direct-count-position_attractor}(b), a reference fixed point is recovered when it is matched one-to-one to an admissible model fixed point within $0.25\,d_{\min}$. If $m$ of the $Q$ reference fixed points are recovered, the model-level positional score is
\begin{equation}
\label{eq:rev-position-score}
R_{\mathrm{FP}}=100\,m/Q.
\end{equation}
A produced model with no admissible fixed point scores zero. 

For memory tasks, bit accuracy measures agreement between the output signs and the target memory values. Using the response-exclusion mask defined in Appendix~\ref{app:flipflop_details}, we compute
\begin{equation}
\label{eq:rev-ff-accuracy}
\mathrm{Acc}=\frac{\sum_{i=1}^{B}\sum_{t=0}^{T-1}\sum_{j=1}^{n}m_{itj}\mathbb I[\operatorname{sign}(o_{itj})=y_{itj}]}{\sum_{i=1}^{B}\sum_{t=0}^{T-1}\sum_{j=1}^{n}m_{itj}},
\end{equation}
where $B$ is the number of evaluation sequences, $T$ is their length, and $n$ is the number of memory bits. The mask excludes the pulse step and the following three steps on the affected channel. Task error is reported as $1-\mathrm{Acc}$.

% The memory-state recovery rate in \Figref{fig:appendix-direct-count-position-memory}(a) uses both stable fixed-point counts and positions. For $n$-bit memory, where $n=b$ in the notation of Appendix~\ref{app:flipflop_details}, structural success requires exactly $2^n$ stable zero-input fixed points, with one within distance $0.3$ of each prescribed memory corner. 
For $n$-bit memory, where $n=b$ in the notation of Appendix~\ref{app:flipflop_details}, memory-state recovery requires exactly $2^n$ stable zero-input fixed points, with one within distance $0.3$ of each prescribed memory corner.
% Using the corner distances defined there, the structural indicator is
% \begin{equation}
% \label{eq:rev-ff-structural}
% S_{\mathrm{struct}}=\mathbf 1\!\left[|\mathcal F_{\mathrm{stable}}|=2^n\;\land\;\max_{\mathbf c\in\{-1,+1\}^n}d(\mathbf c)<0.3\right].
% \end{equation}
Let $\mathcal F_{\mathrm{stable}}$ denote the set of admissible stable zero-input fixed points of the model, and let $\Pi_n$ extract its $n$ readout coordinates. For each prescribed memory corner $\mathbf c\in\{-1,+1\}^n$, define
\begin{equation}
\label{eq:rev-ff-corner-distance}
d(\mathbf c)=\min_{\mathbf z^\star\in\mathcal F_{\mathrm{stable}}}\|\Pi_n\mathbf z^\star-\mathbf c\|_2,
\end{equation}
with $d(\mathbf c)=+\infty$ when $\mathcal F_{\mathrm{stable}}$ is empty. The memory-state recovery indicator is
\begin{equation}
\mathcal S_{\mathrm{struct}}=\mathbb I\!\left[|\mathcal F_{\mathrm{stable}}|=2^n\;\land\;\max_{\mathbf c\in\{-1,+1\}^n}d(\mathbf c)<0.3\right].
\end{equation}
The memory-state recovery rate is the fraction of runs satisfying $S_{\mathrm{struct}}=1$. This criterion imposes no requirement on saddle fixed points and no task-accuracy threshold.
% The joint accuracy-and-structure indicator in Eq.~\eqref{eq:rev-ff-success} additionally requires $\mathrm{Acc}\ge0.98$. 
The memory ablation and off-tree success panels also use $S_{\mathrm{struct}}$, while task error $1-\mathrm{Acc}$ is reported separately.

For the memory positional diagnostic, the reference set is $\{-1,0,+1\}^n$ in readout space. Recovery requires agreement in position within distance $0.3$ and in local fixed-point type. If $m_{\mathrm{grid}}$ reference points are recovered, the full-grid score is
\begin{equation}
\label{eq:rev-grid-score}
R_{\mathrm{grid}}=100\,m_{\mathrm{grid}}/3^n.
\end{equation}
This score includes the $2^n$ memory corners and the non-corner reference points. It is different from the saddle-only recovery reported separately in \Figref{fig:memory_results}(c). A change in the full-grid score should therefore not be described numerically as the same change in saddle recovery.

% ============================================================
% E. ADDITIONAL CHUA RESULTS
% ============================================================

\clearpage
\section{Additional results: Attractor-reconstruction tasks}
\label{app:additional_attractor_results}

% This section provides additional results for the 3-scroll Chua
% benchmark that complement the main-text analysis.
% We first summarize the structural outcomes of hierarchical reduction,
% then compare alternative retraining protocols, examine the robustness
% of the state-space fidelity criterion, and finally analyze how the
% symbolic transition structure changes after retraining.

This appendix complements the attractor-reconstruction results in \Figref{fig:attractor_results}. We examine the structural outcomes of reduction, the effects of guidance and gradual linearization, and changes in symbolic transitions. We then report fixed-point recovery across capacities and detailed results across the attractor-reconstruction tasks, and compare the selected reductions with off-tree controls.

\subsection{Reduction sequences across tasks}
\label{app:chua_reduction_results}

\Figref{fig:appendix-reduction-attractor} illustrates the symbolic simplification predicted by successive reductions. For each parent, we use a single autonomous trajectory of $40{,}000$ steps following a 500-step teacher-forcing warm-up and discard the first $1{,}000$ autonomous steps. At each reduction stage, we project the parent activation patterns onto the remaining ReLU coordinates and construct their transition graph. All stages use the same parent trajectory; no retraining is involved. The displayed visited-symbol counts decrease from 31 to 5 for 3-scroll Chua, from 41 to 7 for 4-scroll Chua, from 22 to 9 for 5-scroll Chua, and from 61 to 6 for Dadras. The corresponding final nonlinear capacities are $P_{\mathcal D}=3,3,6,$ and $3$.

The final Chua graphs have a chain-like organization with one FP symbol for each of the 5, 7, or 9 protected fixed points, whereas the Dadras graph retains one FP-free symbol in addition to its 5 FP symbols. These examples illustrate how hierarchical reduction simplifies both nonlinear capacity and symbolic complexity before retraining. The two lower bounds need not be attained simultaneously: the displayed 5-scroll Chua reduction attains the 9-symbol bound but retains 6 ReLU units, above the fixed-point-based capacity bound of $\lceil\log_2 9\rceil=4$.

\subsection{Reorganization of symbolic transitions after retraining}
\label{app:transition_topology}

The transition graph realized after retraining need not reproduce the selected quotient graph exactly. \Figref{fig:appendix-nonchain-retraining} shows two examples for 3-scroll Chua and one for 4-scroll Chua, all using the main-text retraining condition: combined full-preactivation and symbol guidance with 200-epoch gradual linearization. The selected candidates have 5 or seven 7 symbols but contain additional connections that make their graphs non-chain-like. After retraining, the autonomous trajectories realize chain-like transition graphs with the same symbol counts.

The two 3-scroll examples attain $Q_{\mathrm{vis}}=|\widehat\Sigma|=5$, with $E_{\mathrm{stsp}}=1.40$ and $1.39$. The 4-scroll example attains $Q_{\mathrm{vis}}=|\widehat\Sigma|=7$, with $E_{\mathrm{stsp}}=3.67$. The chain-like organization is not imposed as a graph-matching objective during retraining. Rather, it emerges in the transitions of the learned autonomous dynamics. These examples show that a quotient candidate provides a structural starting point for retraining without fixing the final transition topology.

\begin{figure*}[p]
    \centering
    \includegraphics[width=\textwidth,height=0.78\textheight,keepaspectratio]{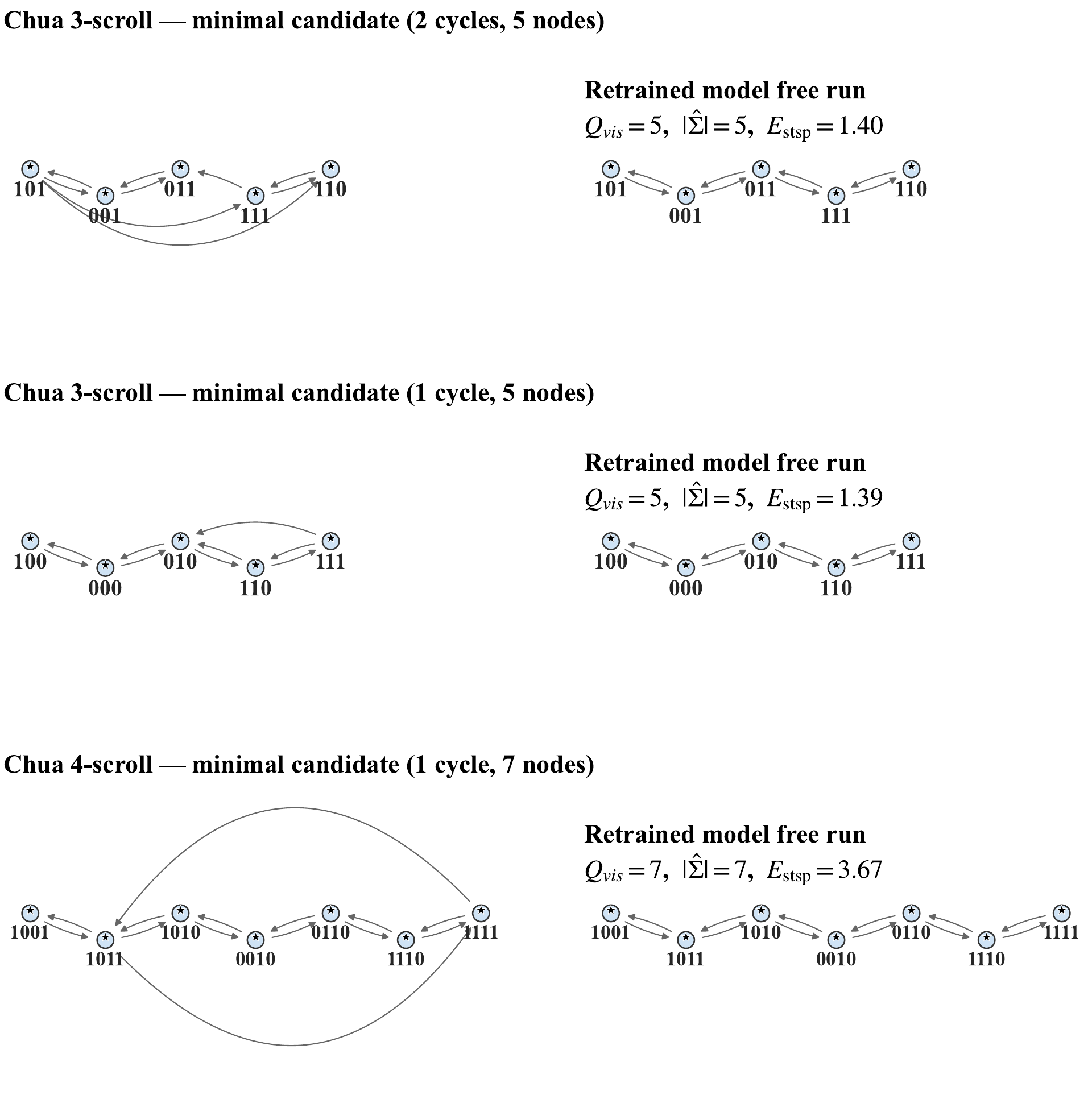}
    \caption{\textbf{Reorganization of non-chain quotient candidates after retraining.} The first two rows show 3-scroll Chua examples and the third shows a 4-scroll Chua example. All retrained models use the main-text condition: combined full-preactivation and symbol guidance with 200-epoch gradual linearization. Left: selected quotient graphs before retraining. Right: symbolic transition graphs from autonomous rollouts of the corresponding retrained models. Node labels are retained-ReLU activation patterns, stars mark FP symbols, and arrows show transitions. Additional non-chain connections in the candidates are absent from the displayed post-retraining graphs, which form five- and seven-node chains. The 3-scroll models attain $Q_{\mathrm{vis}}=|\widehat\Sigma|=5$ with $E_{\mathrm{stsp}}=1.40$ and $1.39$; the 4-scroll model attains $Q_{\mathrm{vis}}=|\widehat\Sigma|=7$ with $E_{\mathrm{stsp}}=3.67$. These examples show that retraining can reorganize transitions while retaining the minimal symbol count.}
    \label{fig:appendix-nonchain-retraining}
\end{figure*}

\subsection{Fixed-point counts and positional recovery}
\label{app:attractor_fp_recovery}

\Figref{fig:appendix-direct-count-position_attractor} compares count-based and position-verified fixed-point recovery under direct training. Both measures use the same 30 models at each nonlinear capacity $P$, including unsuccessful runs. Count-based recovery reports the fraction of runs satisfying $Q_{\mathrm{vis}}=Q$, without requiring positional agreement. Position-verified recovery reports the fraction of target fixed points matched one-to-one within $0.25\,d_{\min}$, averaged across seeds. It therefore credits partial recovery even when a model fails the count criterion.

For 3- and 4-scroll Chua, both measures approach complete recovery as nonlinear capacity increases. Dadras shows a different pattern: count-based recovery is nearly complete over a broad range of capacities, whereas positional recovery levels off at approximately $60\%$ at higher capacities. Thus, reliably reproducing the fixed-point count does not ensure that the fixed points occupy the target locations. This discrepancy is already present under direct training of higher-capacity models, providing context for the incomplete positional recovery of both parents and retrained models in \Figref{fig:attractor_results}(c).

For 5-scroll Chua, positional recovery generally increases with capacity, including at capacities where few runs satisfy the count criterion. At $P=12$, approximately $73\%$ of runs recover the target count, while approximately $89\%$ of target fixed points are recovered by position on average. These percentages have different denominators: the former measures how often a run satisfies the count criterion, whereas the latter measures how much of the target fixed-point set is recovered. Together, the two measures distinguish count agreement from positional recovery and reveal partial reconstruction that a binary count criterion does not capture.

\begin{figure*}[t]
    \centering
    \includegraphics[width=\textwidth]{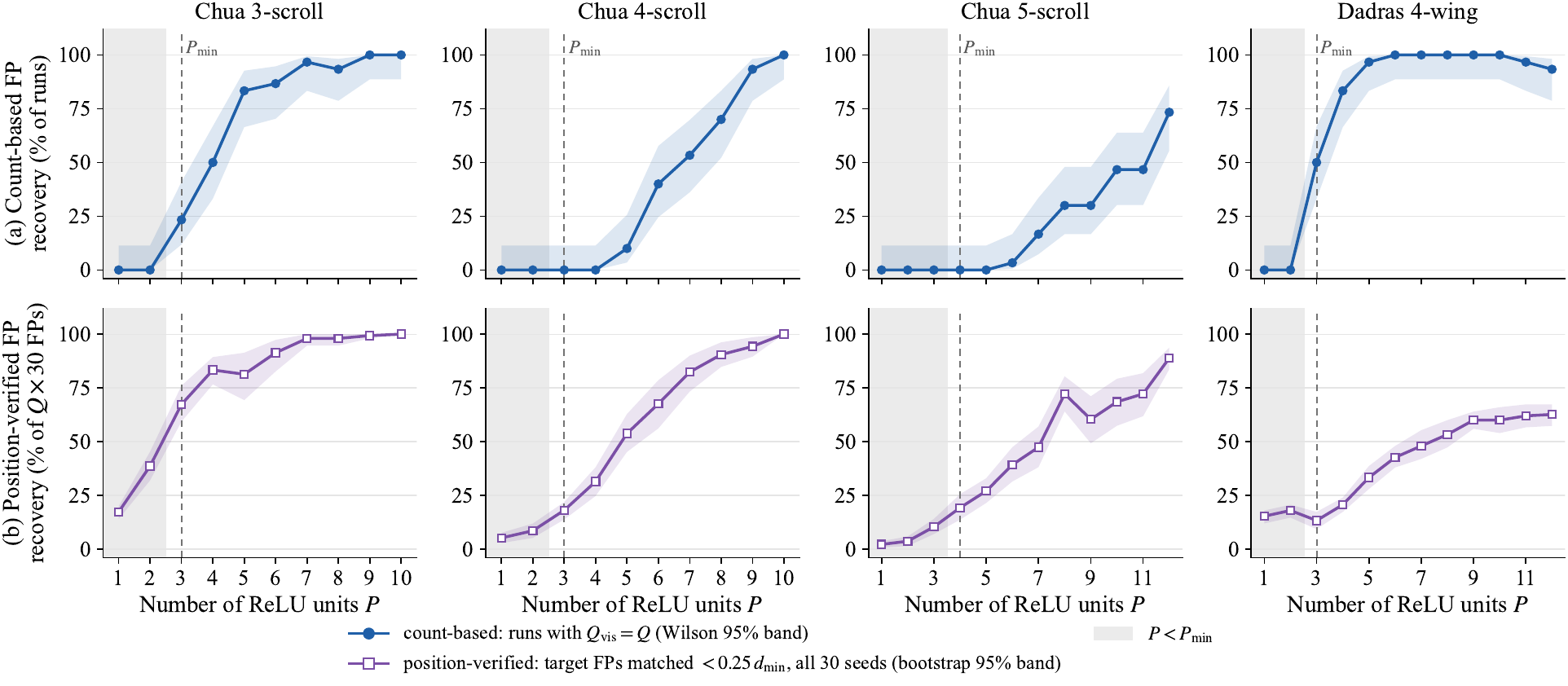}
    \caption{\textbf{Fixed-point counts and positional recovery under direct training.} Columns show the four attractor-reconstruction tasks. Both rows evaluate the same 30 seeds at each capacity $P$, without filtering by success. (a) Fraction of runs satisfying $Q_{\mathrm{vis}}=Q$; bands show Wilson 95\% confidence intervals. (b) Fraction of target fixed points matched one-to-one within $0.25\,d_{\min}$, averaged over all 30 seeds; bands show 95\% bootstrap confidence intervals obtained by resampling seeds. Models with no admissible fixed points score zero. The denominators are 30 runs in (a) and $30Q$ target fixed points in (b), with the seed remaining the resampling unit. Vertical dashed lines mark the fixed-point-based capacity bounds $P_{\min}$; lower capacities are shaded.}
    \label{fig:appendix-direct-count-position_attractor}
\end{figure*}

\subsection{Fixed-point correspondence and local dynamics}
\label{app:fp_local_dynamics}

We extend the positional analysis in Appendix~\ref{app:attractor_fp_recovery} by examining local stability and eigenspectra. Target and admissible model fixed points are matched one-to-one in observation coordinates within $0.25\,d_{\min}$. We then assess local-type agreement through the number of unstable directions of the full latent-state Jacobian. This distinguishes recovery of fixed-point locations from recovery of their local stability structure.

\Figref{fig:appendix-fp-local-recovery} compares positional recovery with joint position-and-local-type recovery. For the retrained cohorts, joint recovery is approximately $99.4\%$, $96.1\%$, $78.9\%$, and $17.8\%$ for 3-, 4-, and 5-scroll Chua and Dadras, respectively. Conditional on positional matching, unstable-dimension agreement is approximately $97$--$100\%$ in the Chua analyses and $78$--$85\%$ for Dadras. Thus, most position-matched Chua fixed points also recover the target local stability. Dadras exhibits additional discrepancies, including excess unstable directions at some matched non-origin fixed points.

\begin{figure*}[t]
    \centering
    \includegraphics[width=\textwidth]{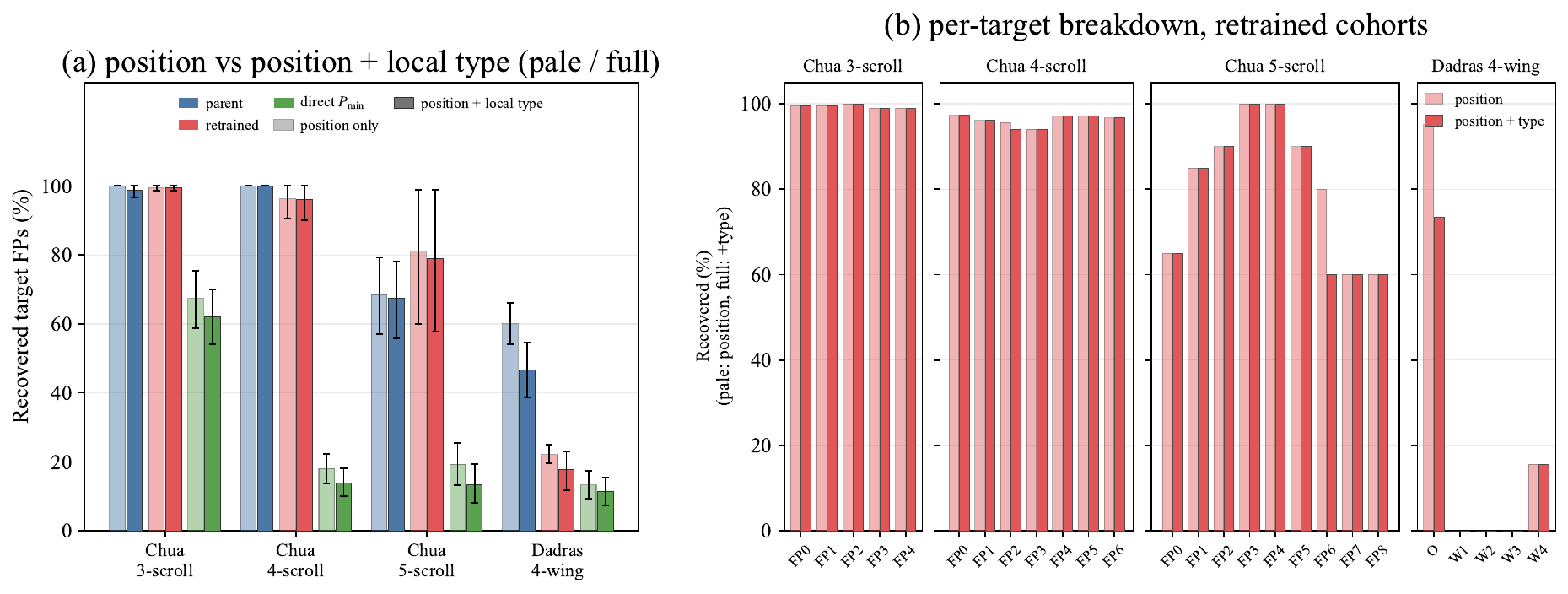}
   \caption{\textbf{Fixed-point correspondence in position and local stability.}
(a) Positional recovery alone (pale bars) and joint position-and-local-type recovery (saturated bars) for parents, retrained models, and direct models at $P_{\min}$. Positional matching is one-to-one within $0.25\,d_{\min}$; local-type agreement additionally requires the correct unstable dimension of the full model Jacobian.
(b) Per-target recovery for the retrained cohorts, using the same positional and local-type criteria.}
    \label{fig:appendix-fp-local-recovery}
\end{figure*}

For the spectral comparison, we map continuous-time target eigenvalues $\lambda_j$ to discrete-time multipliers at the sampling interval $\Delta t$:
\begin{equation}
    \mu^{\mathrm{target}}_j
    =
    \exp(\lambda_j\Delta t).
    \label{eq:target_discrete_multipliers}
\end{equation}
These multipliers are compared with three matched eigenvalues from the 20-dimensional AL-RNN Jacobian. We report the maximum absolute discrepancy among the three matched pairs and separately inspect the remaining 17 eigenvalues and the unstable dimension of the full spectrum. A favorable three-mode match alone does not rule out additional unstable latent directions.

% \Figref{fig:appendix-fp-eigenspectra} shows representative spectra and the distributions of matched-mode discrepancies. In the Chua analyses, median maximum multiplier discrepancies are approximately $0.05$--$0.08$. Because the target multipliers lie near one at the sampling interval, these values do not imply equally small errors in continuous-time dynamics. On the growth-rate scale $\log|\mu|/\Delta t$, discrepancies are approximately $0.3$--$1.2$ per unit time for Chua and of order $5$ for Dadras.
\Figref{fig:appendix-fp-eigenspectra}(a,b) shows the full local spectra at the five matched fixed points of a representative 3-scroll Chua model. Panel (c) summarizes matched-mode discrepancies for parents and retrained models attaining minimal realization, using position-matched fixed points. This success-conditioned spectral comparison is distinct from the fixed-point recovery analysis over the complete cohorts.

The remaining latent modes are predominantly contractive in the Chua models, with median multiplier moduli of approximately $0.73$--$0.84$. Extra unstable directions occur at no more than $1.3\%$ of position-matched Chua fixed points, and no near-neutral Chua modes are identified under the numerical checks. Dadras shows extra unstable directions at approximately $12$--$14\%$ of position-matched fixed points. These results support recovery of local stability in the Chua systems while identifying additional limitations in the learned Dadras dynamics.

\begin{figure*}[t]
    \centering
    \includegraphics[width=\textwidth]{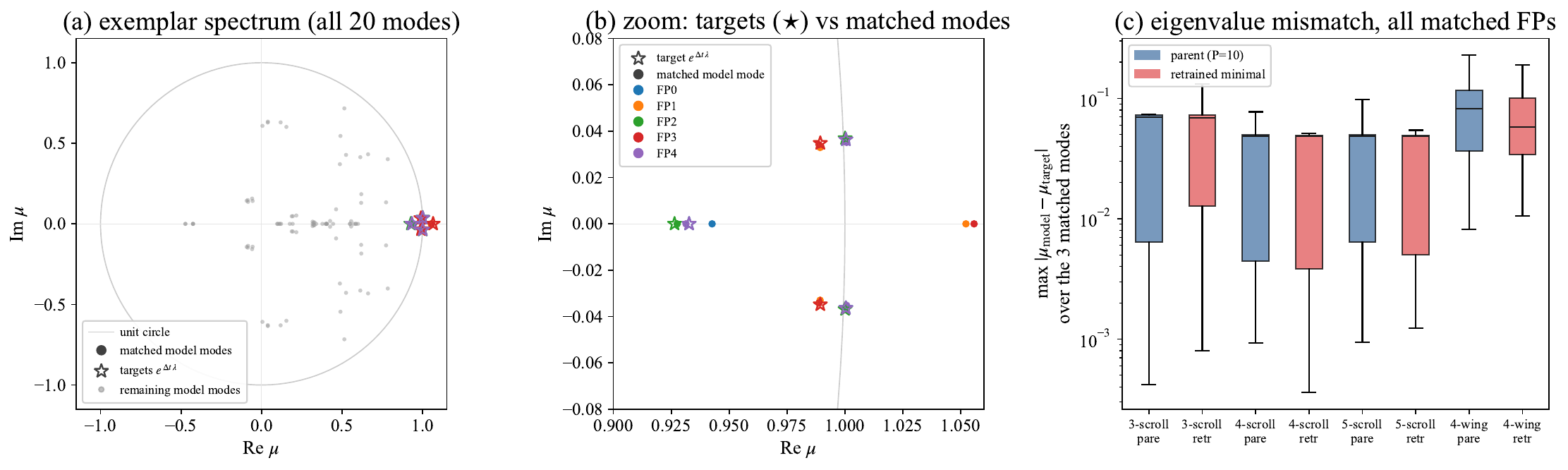}
    \caption{\textbf{Local eigenspectra and additional latent directions.}
(a) Full 20-dimensional local spectra at the five matched fixed points of a representative 3-scroll Chua model. Stars indicate target multipliers, colored points indicate matched model eigenvalues, and gray points indicate the remaining model eigenvalues. The gray circle marks the unit circle.
(b) Enlarged view of the target and matched model multipliers, colored by fixed-point identity.
(c) Maximum absolute discrepancy among the three matched multiplier pairs at position-matched fixed points, comparing $P=10$ parents with retrained models attaining minimal realization. The full unstable dimension and remaining latent modes are assessed separately from this three-mode comparison.}
    \label{fig:appendix-fp-eigenspectra}
\end{figure*}

\subsection{Ablation study: Effects of linearization-set selection}
\label{app:offtree-attractor}

To test whether the choice of ReLU units matters beyond the final nonlinear capacity, we compare tree-selected and off-tree linearization sets. For a given parent, let $\mathcal C$ denote the collection of tree-selected linearization sets used in the comparison. For each $\mathcal D\in\mathcal C$, we sample an off-tree set satisfying
\begin{equation}
\label{eq:offtree-linearization}
\mathcal D_{\mathrm{off}}\notin\mathcal C,
\qquad
|\mathcal D_{\mathrm{off}}|=|\mathcal D|.
\end{equation}
Both models are initialized from the same parent, so the comparison holds the parent initialization and remaining nonlinear capacity fixed while changing which ReLU units are linearized.

We evaluate 141 tree-selected candidates for 3-scroll Chua and 102 for 4-scroll Chua under two retraining protocols: no auxiliary guidance with gradual linearization over 200 epochs, and full-preactivation guidance with abrupt linearization. These conditions provide two ways of obtaining successful retraining, through gradual linearization or parent guidance (Appendix~\ref{app:chua_guidance_ablation}). Within each protocol, tree-selected and off-tree models use the same retraining settings.

The left panel of \Figref{fig:appendix-mask-control-chua} reports raw per-run minimal-realization rates, rather than the parent-balanced rates used in \Figref{fig:attractor_results}(c) and \Figref{fig:appendix-guidance-anneal-chua}. Without auxiliary guidance, gradual linearization achieves approximately $80\%$ and $71\%$ success for tree-selected reductions on 3- and 4-scroll Chua, compared with $21\%$ and $14\%$ for off-tree controls. With full-preactivation guidance and abrupt linearization, the corresponding rates are $82\%$ and $73\%$ for tree-selected reductions versus $16\%$ and $11\%$ for off-tree controls. The substantial advantage of tree-selected sets therefore persists under both retraining protocols.

Position-verified fixed-point recovery also favors tree-selected reductions (\Figref{fig:appendix-mask-control-chua}, right). With unguided gradual linearization, recovery is approximately $96\%$ and $94\%$ for tree-selected reductions, compared with $74\%$ for off-tree controls on both systems. With full-preactivation guidance and abrupt linearization, recovery is $92\%$ and $81\%$ versus $63\%$ and $56\%$, respectively. These averages include all produced models, allowing partial recovery by models that do not achieve minimal realization.

Thus, tree-selected linearization sets improve fixed-point positional recovery and yield a pronounced advantage in minimal-realization success over size-matched off-tree sets from the same parents. These results support the effectiveness of our fixed-point-guided reduction method in identifying which ReLU units to linearize.

\begin{figure*}[t]
    \centering
    \includegraphics[width=\textwidth]{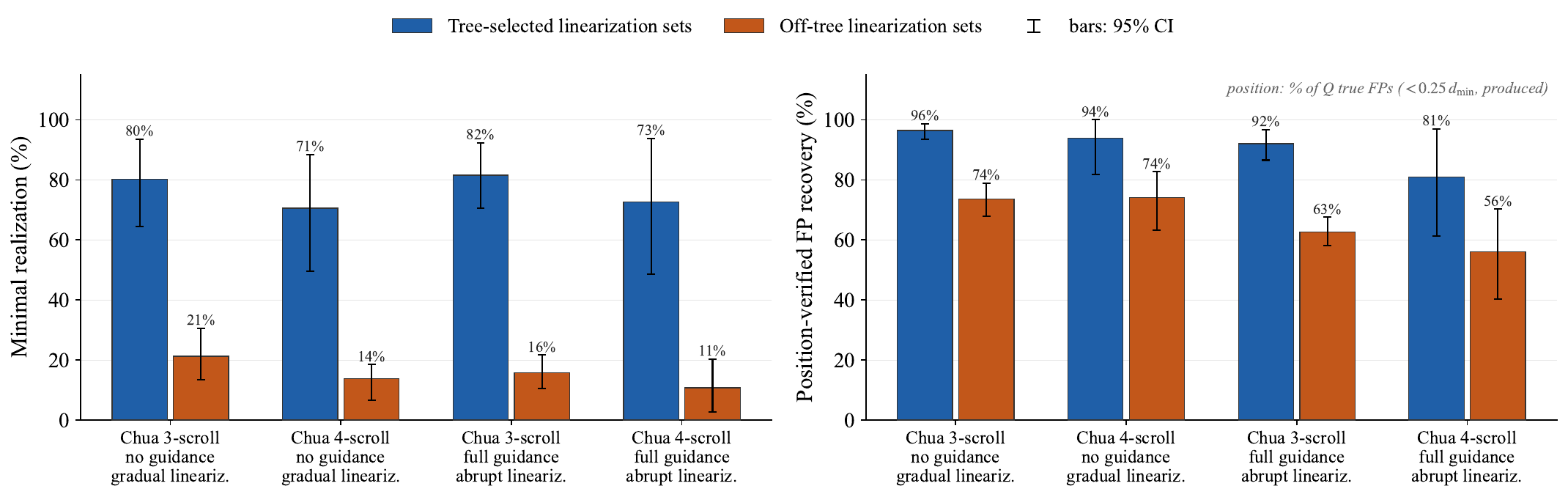}
    \caption{\textbf{Effects of linearization-set selection on the attractor-reconstruction tasks.} Tree-selected and off-tree sets are matched by parent and remaining ReLU count for 141/102 selected candidates on 3-/4-scroll Chua. Comparisons use either no auxiliary guidance with gradual linearization over 200 epochs or full-preactivation guidance with abrupt linearization. Left: raw per-run minimal-realization success. Right: position-verified fixed-point recovery using one-to-one matching within $0.25\,d_{\min}$, averaged over all produced models, including those that fail the minimal-realization criterion. Error bars indicate 95\% confidence intervals.}
    \label{fig:appendix-mask-control-chua}
\end{figure*}

\subsection{Ablation study: Effects of guidance and gradual linearization}
\label{app:chua_guidance_ablation}

We compare retraining objectives and linearization schedules using the same selected reductions for each task: 141 candidates for 3-scroll Chua and 102 for 4-scroll Chua, from cohorts of 30 initial parents with $P=10$. The 3-scroll comparison includes both the 126 candidates with three remaining ReLUs and the 15 with four.

For abrupt linearization, the selected ReLU activations are replaced by the identity before retraining. For gradual linearization, the same transition takes place over the first 200 retraining epochs. 
% The plotted abrupt conditions use no auxiliary guidance, symbol guidance, retained-preactivation guidance, or full-preactivation guidance; the gradual conditions use no auxiliary guidance, full-preactivation guidance, or combined full-preactivation and symbol guidance.
The objectives, fixed weights, and synchronization rules are given in Appendix~\ref{app:retraining_detail}. 
% The no-guidance condition still minimizes the parent-readout loss. 
Because synchronization differs between objectives, these comparisons evaluate the corresponding training protocols rather than isolating only the auxiliary loss.

\Figref{fig:appendix-guidance-anneal-chua} separates structural success from geometric and distributional fidelity. Minimal-realization rates average candidate success within each parent and then across all 30 initial parents, assigning zero to parents with no selected candidate. Positional recovery measures the fraction of target fixed points matched one-to-one within $0.25\,d_{\min}$, averaged over produced models; a model with no admissible fixed point scores zero. State-space and spectral errors are also evaluated over produced models, without restricting to minimal-realization successes.

Without auxiliary guidance, abrupt linearization attains approximately $14\%$ and $9\%$ minimal-realization success for 3- and 4-scroll Chua. Adding symbol, retained, or full guidance increases these rates to approximately $70$--$74\%$ and $37$--$39\%$, respectively. 
% Gradual linearization without auxiliary guidance reaches approximately $74\%$ and $38\%$, comparable to guided abrupt retraining. 
Gradual linearization without auxiliary guidance achieves comparable gains, reaching approximately $74\%$ and $38\%$. Thus, gradual linearization alone improves minimal-realization success about as much as adding parent guidance under abrupt linearization.
Full guidance with gradual linearization reaches $80\%$ and $46\%$, and combined guidance with gradual linearization reaches $85\%$ and $49\%$, the highest minimal-realization rates among the plotted conditions.

The positional results show a similar improvement: recovery rises from approximately $51\%$ and $53\%$ for unguided abrupt linearization to $96\%$ and $93\%$ for unguided gradual linearization, and to $99\%$ and $96\%$ for combined guidance with gradual linearization. The error distributions do not follow a uniform ranking. In particular, the combined protocol does not attain the lowest median $E_{\mathrm{stsp}}$ or $E_{\mathrm H}$ in every comparison. 
% These results support its effectiveness for realizing the selected symbolic structure, while distinguishing that objective from minimizing each trajectory-fidelity metric.
Combined guidance with gradual linearization achieves the highest minimal-realization success rate, although it does not yield the lowest error on every trajectory-fidelity metric.

\begin{figure*}[t]
    \centering
    \includegraphics[width=\textwidth]{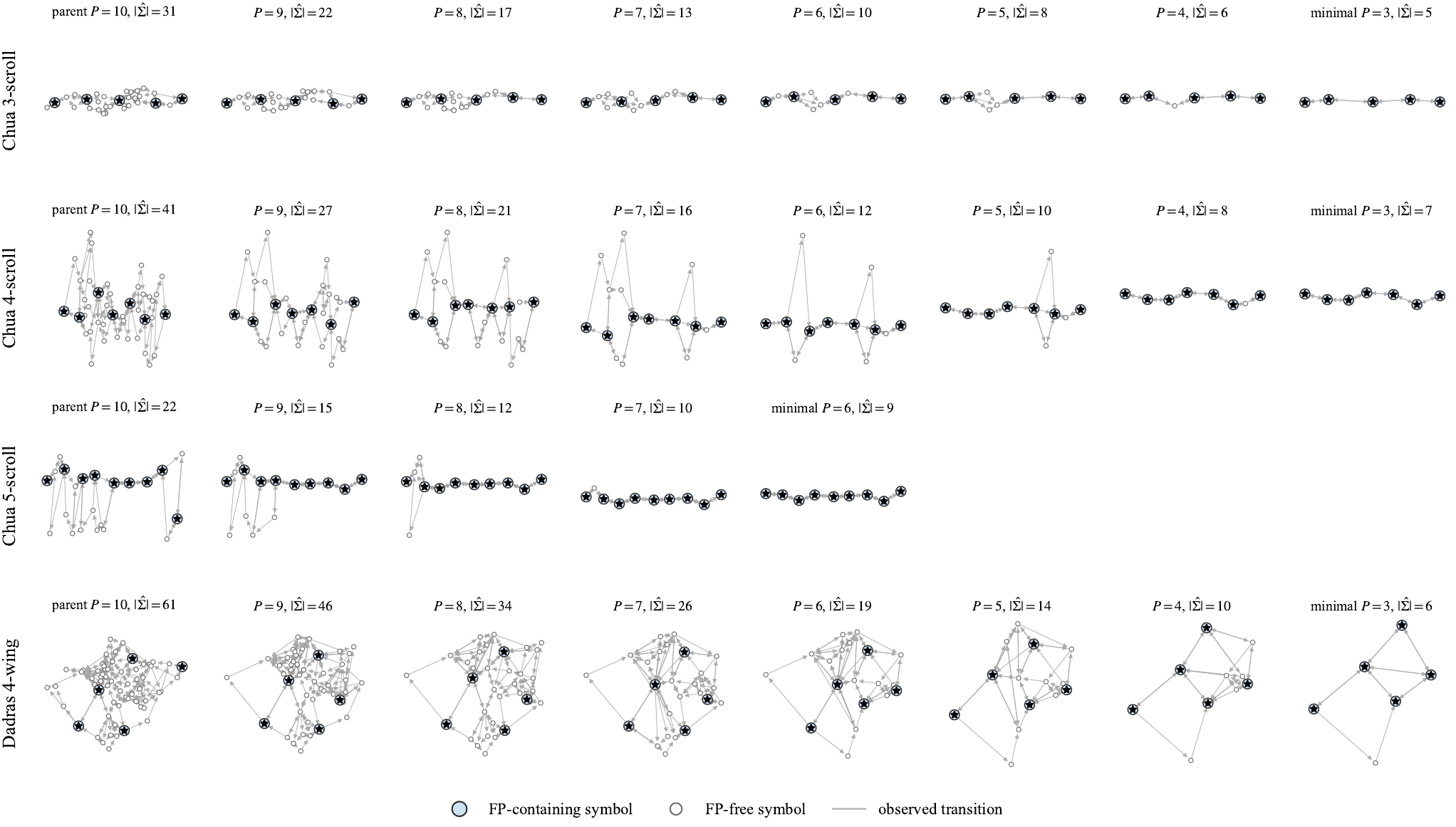}
    \caption{\textbf{Example reduction sequences for the attractor-reconstruction tasks.} Rows show 3-, 4-, and 5-scroll Chua and Dadras, with successive reductions arranged from left to right. The graphs are projections of a fixed parent trajectory onto the remaining ReLU coordinates; no retraining is involved. Each parent trajectory contains $40{,}000$ autonomous steps after a 500-step teacher-forcing warm-up, with the first $1{,}000$ autonomous steps discarded. Panel labels report the remaining ReLU count and the number of visited projected symbols. Light-blue starred nodes are FP symbols, open nodes are FP-free symbols, and arrows are projected parent transitions. The sequences end at 5, 7, 9, and 6 symbols with 3, 3, 6, and 3 remaining ReLUs, respectively. The Chua endpoints are chain-like; the Dadras endpoint retains one additional FP-free symbol.}
    \label{fig:appendix-reduction-attractor}
\end{figure*}

\begin{figure*}[t]
    \centering
    \includegraphics[width=\textwidth]{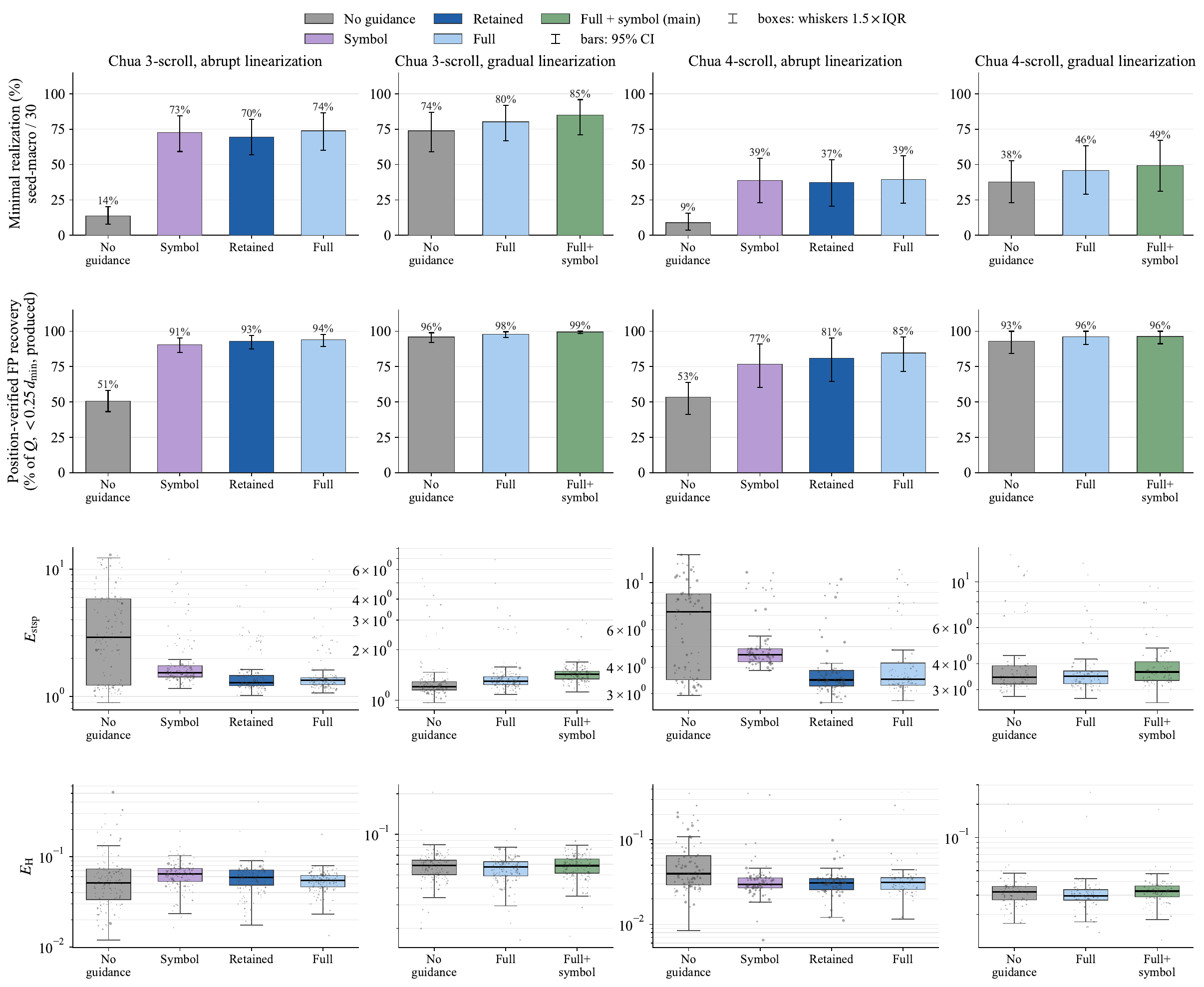}
    \caption{\textbf{Guidance and gradual-linearization ablations on the attractor-reconstruction tasks.} The comparisons use 141/102 selected candidates for 3-/4-scroll Chua from cohorts of 30 initial parents with $P=10$. Columns contrast abrupt linearization with gradual linearization over 200 epochs. Guidance uses fixed auxiliary weights, with $\lambda_{\mathrm{sym}}=1$ and preactivation weights specified in Appendix~\ref{app:retraining_detail}; full+symbol with gradual linearization is the main protocol. Top row: minimal-realization success, averaged within each parent and across all 30 parents, with parents yielding no candidate contributing zero. Second row: the fraction of target fixed points matched one-to-one within $0.25\,d_{\min}$, evaluated over produced models, with no admissible fixed points scoring zero. Bottom rows: $E_{\mathrm{stsp}}$ and $E_{\mathrm H}$ over produced models; dot size is proportional to candidate weight.
    Guidance substantially improves abrupt retraining, and gradual linearization also improves success without auxiliary guidance. The combined protocol achieves the highest minimal-realization rates, but does not uniformly minimize the state-space and spectral errors.}
    \label{fig:appendix-guidance-anneal-chua}
\end{figure*}

\FloatBarrier

\clearpage

% % ============================================================
% % F. ADDITIONAL BENCHMARK RESULTS
% % ============================================================

\section{Additional results: Flip-flop memory tasks}
\label{app:additional_memory_results}

\subsection{Reduction sequences across tasks}
\label{app:memory_reduction_results}

\Figref{fig:appendix-reduction-memory} illustrates successive reductions of the input-driven symbolic graphs. For each parent, we use a fixed input sequence of $6{,}000$ steps and discard the first 100 steps. At each stage, we project the parent activation patterns onto the remaining ReLU coordinates and construct their transition graph. All stages use the same parent trajectory; no retraining is involved. The displayed visited-symbol counts decrease from 85 to 9 for 2-bit memory and from 141 to 27 for 3-bit memory, while the nonlinear capacities decrease from $P=10$ to $P_{\mathcal D}=4$ and $6$, respectively.

The final graphs contain only FP symbols, associated with stable memory states and intervening unstable fixed points. The fixed points characterize the zero-input dynamics, whereas the graph edges describe transitions under the input sequence. Under the fixed-point separation assumptions, preserving 9 and 27 distinct fixed points requires
\begin{equation}
\label{eq:rev-memory-capacity-bounds}
P_{\mathcal D}\geq\lceil\log_2 9\rceil=4,
\qquad
P_{\mathcal D}\geq\lceil\log_2 27\rceil=5,
\end{equation}
respectively. The displayed 2-bit reduction attains both the symbol and nonlinear-capacity bounds; the 3-bit reduction attains 27 visited symbols while retaining 6 ReLUs. These capacity bounds concern preservation of the specified fixed-point structure, rather than memory-task performance alone.

The 3-bit endpoint uses the same parent and linearization set as the 29-symbol candidate in \Figref{fig:memory_results}(a). The counts differ because the reduction tree uses four construction sequences of $5{,}000$ steps each and includes the parent's admissible zero-input FP regions as graph nodes. Its quotient graph contains two additional transient symbols that are not visited in the separate $6{,}000$-step evaluation trajectory. The difference therefore reflects the graph-construction and observation conventions, not a change caused by retraining.

\subsection{Fixed-point counts and positional recovery}
\label{app:memory_fp_recovery}
\label{app:rev-memory-results}

\Figref{fig:appendix-direct-count-position-memory} compares two measures of fixed-point recovery under direct training, using the same 30 models at each nonlinear capacity $P$. Count-based recovery measures the fraction of runs with exactly $2^n$ stable zero-input fixed points, without imposing a positional criterion. Position- and type-verified recovery measures the fraction of the $3^n$ reference grid points in $\{-1,0,+1\}^n$ matched one-to-one within distance $0.3$ to model fixed points of the corresponding local type. All seeds are included, regardless of count-based success. This comparison is distinct from memory-state recovery, which requires both the correct stable count and a stable fixed point near every prescribed memory corner.

Increasing nonlinear capacity generally improves both measures, but their outcomes differ substantially. At higher capacities, the stable-count criterion is satisfied by approximately $90\%$ or more of the runs, while reference-grid recovery reaches only about two-thirds for 2-bit memory and remains below one-half for 3-bit memory. Thus, obtaining the expected number of stable fixed points does not ensure recovery of the full reference arrangement, including its unstable fixed points.

The comparison after retraining in \Figref{fig:memory_results}(c) further separates memory-corner recovery from saddle recovery. Among the minimal-achieving models and corresponding parents, memory-corner recovery is complete, whereas retrained models recover approximately $55\%$ and $41\%$ of the reference saddles for 2- and 3-bit memory, compared with $37\%$ and $28\%$ for their parents. Restricting the parent comparison to structurally successful parents leaves the latter values unchanged at $37.1\%$ and $27.7\%$. One-to-one matching allows additional model fixed points to remain unmatched, so their presence does not directly reduce the recovery score. These results show that reduction and retraining can improve recovery of the reference saddle arrangement even when the stable memory states are already recovered.

\begin{figure*}[t]
    \centering
    \includegraphics[width=\textwidth]{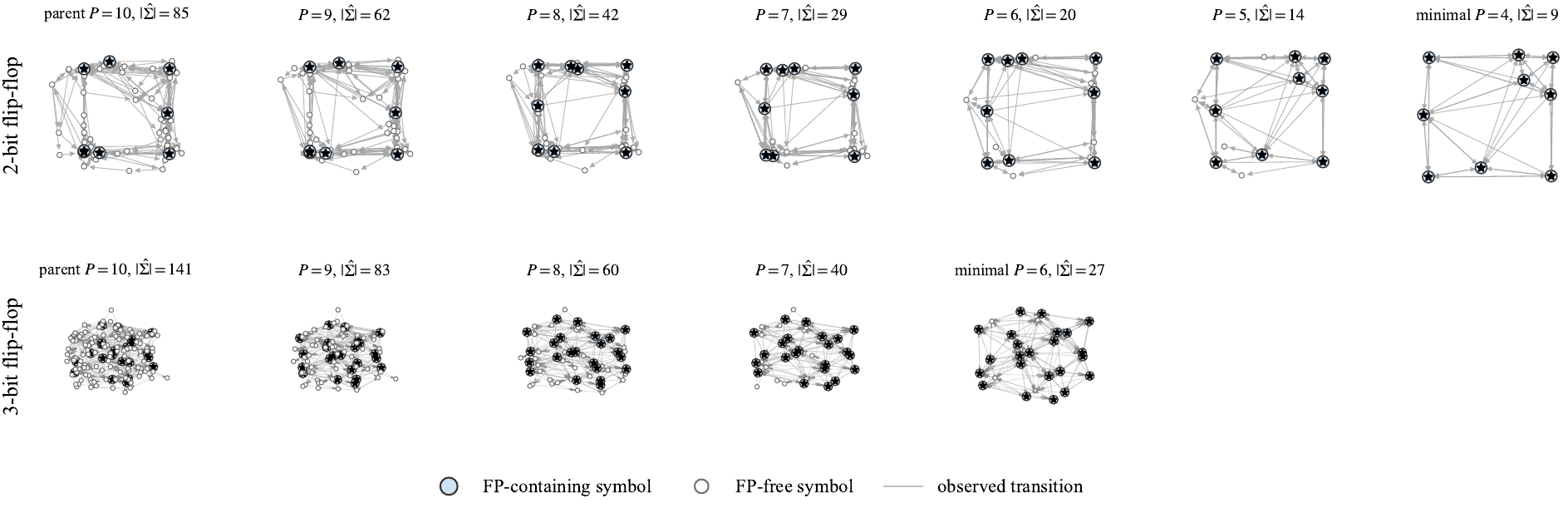}
    \caption{\textbf{Example reduction sequences for the flip-flop memory tasks.} Rows show 2- and 3-bit memory, with successive reductions arranged from left to right. Graphs are projections of a fixed input-driven parent trajectory onto the remaining ReLU coordinates; no retraining is involved. Panel labels report the remaining ReLU count and the number of projected symbols visited over $6{,}000$ steps, with the first 100 discarded. Input sequences are generated with random seed zero. Starred nodes are zero-input FP symbols, open nodes are FP-free symbols, and arrows are projected parent transitions. The sequences end at 9 and 27 visited symbols with 4 and 6 remaining ReLUs. The 3-bit endpoint and the 29-symbol candidate in \Figref{fig:memory_results}(a) use the same linearization set but different counting conventions, as explained in the text.}
    \label{fig:appendix-reduction-memory}
\end{figure*}

\begin{figure*}[t]
    \centering
    \includegraphics[width=0.8\textwidth]{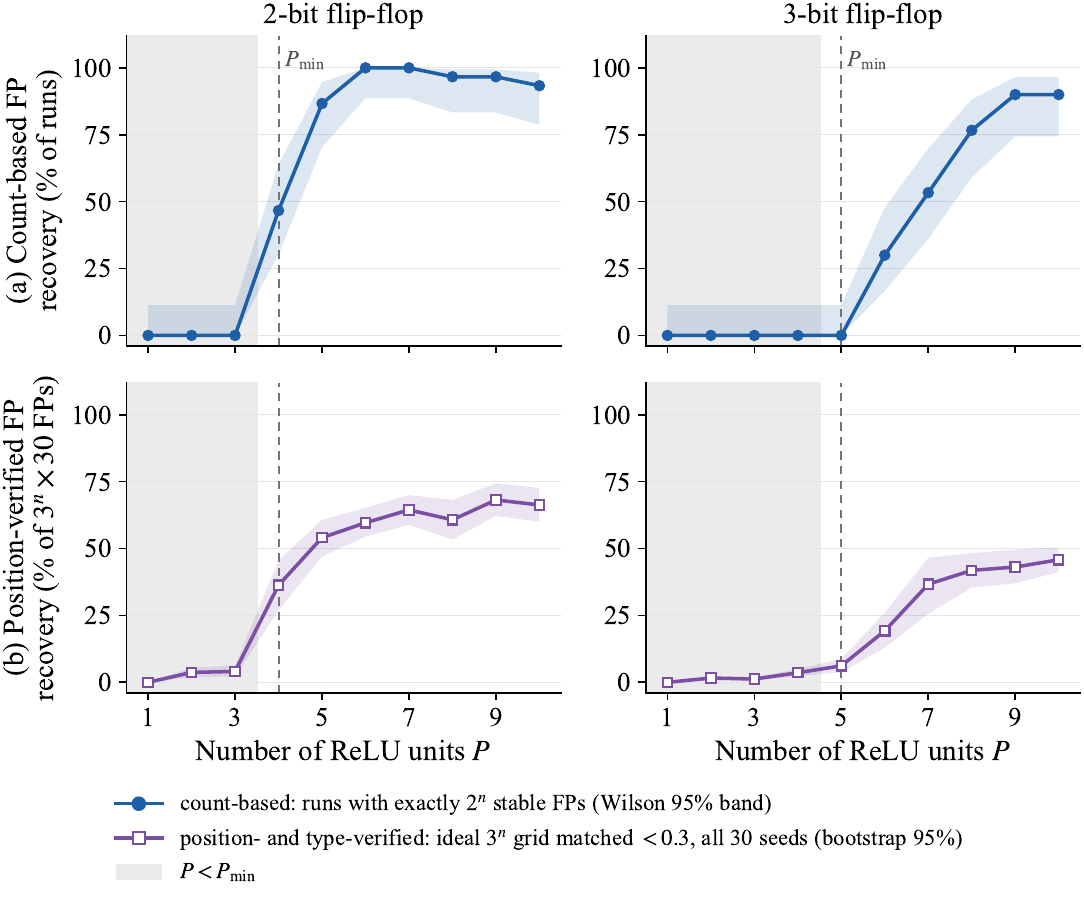}
    \caption{\textbf{Fixed-point counts and positional recovery under direct training on the flip-flop memory tasks.} Columns show 2- and 3-bit memory. Both rows evaluate the same 30 seeds at each capacity $P$, without filtering by success. (a) Fraction of runs with exactly $2^n$ stable zero-input fixed points, without a positional requirement; bands show Wilson 95\% confidence intervals. (b) Fraction of reference grid points matched one-to-one within distance $0.3$ and with the corresponding local type, averaged over all seeds; bands show 95\% bootstrap confidence intervals obtained by resampling seeds. The denominators are 30 runs in (a) and $30\times 3^n$ reference fixed points in (b). Dashed lines and shading mark the capacity bounds for preserving the full $3^n$-point reference structure and capacities below those bounds.}
    \label{fig:appendix-direct-count-position-memory}
\end{figure*}

\subsection{Ablation study: Effects of linearization-set selection}
\label{app:offtree-memory}
\label{app:rev-offtree}

We apply the off-tree control described in Appendix~\ref{app:offtree-attractor} to the memory tasks. For each tree-selected candidate, we sample two off-tree linearization sets from the same parent, matching the number of linearized ReLU units and excluding sets in that parent's tree-selected candidate collection. Tree-selected and off-tree runs use paired training random seeds, holding the parent initialization and remaining nonlinear capacity fixed while changing which units are linearized.

As in the attractor-reconstruction comparison, we evaluate two retraining protocols: no auxiliary guidance with gradual linearization, and full-preactivation guidance with abrupt linearization. For memory tasks, gradual linearization takes place over the first 200 optimization steps. Both protocols minimize the task loss, with the full-preactivation loss added in the guided condition. Within each protocol, tree-selected and off-tree models use the same retraining settings.

The left panel of \Figref{fig:appendix-mask-control-ff} reports raw per-run memory-state recovery rates, rather than the parent-balanced rates used in \Figref{fig:memory_results}(c) and \Figref{fig:appendix-guidance-anneal-ff}. Recovery requires exactly $2^n$ stable zero-input fixed points, with one within distance $0.3$ of each memory corner. Without auxiliary guidance, gradual linearization achieves approximately $92\%$ and $97\%$ recovery for tree-selected sets on 2- and 3-bit memory, compared with $45\%$ and $23\%$ for off-tree sets. With full-preactivation guidance and abrupt linearization, the corresponding rates are $95\%$ and $88\%$ versus $30\%$ and $14\%$. The advantage of tree-selected sets therefore persists under both retraining protocols.

Position- and type-verified recovery of the reference grid also favors tree-selected sets (\Figref{fig:appendix-mask-control-ff}, right). This metric uses one-to-one matching to the $3^n$ reference points within distance $0.3$, requiring agreement in local fixed-point type. It is averaged equally over all produced models, including those that fail the memory-state recovery criterion. With unguided gradual linearization, recovery is approximately $58\%$ and $47\%$ for tree-selected sets on 2- and 3-bit memory, compared with $36\%$ and $18\%$ for off-tree sets. With full-preactivation guidance and abrupt linearization, the corresponding rates are $72\%$ and $56\%$ versus $26\%$ and $12\%$.

Tree-selected linearization sets thus improve both recovery of the stable memory states and agreement with the reference fixed-point structure. Together with the attractor-reconstruction results, these findings support the effectiveness of our fixed-point-guided reduction method in identifying which ReLU units to linearize.

\begin{figure*}[t]
    \centering
    \includegraphics[width=\textwidth]{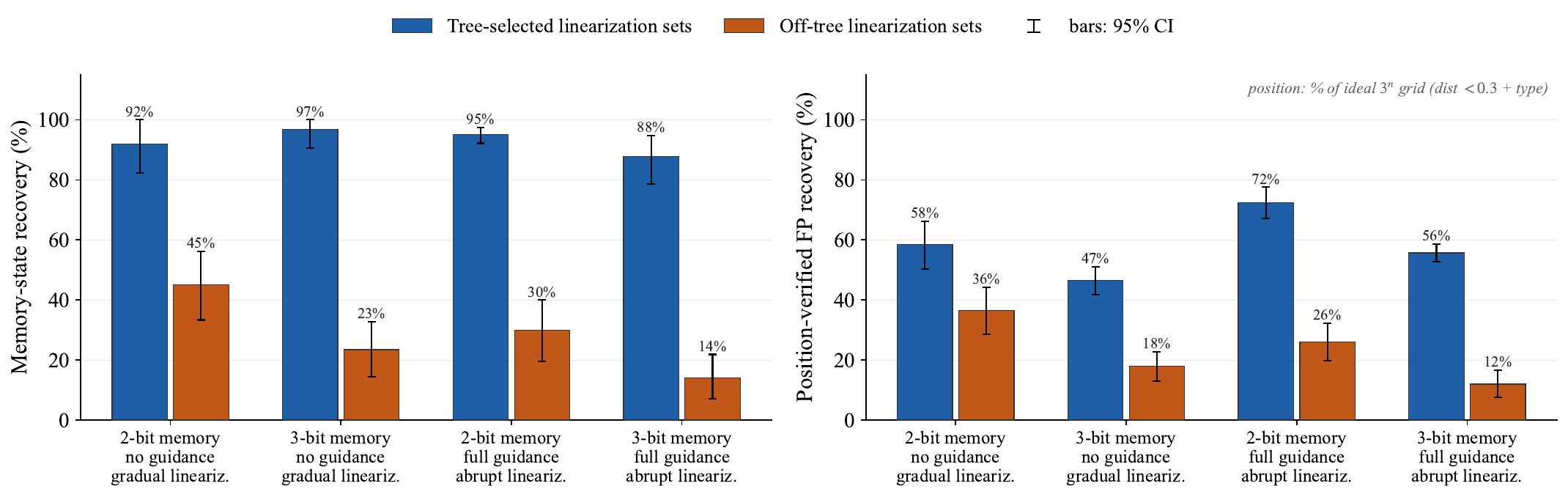}
    \caption{\textbf{Effects of linearization-set selection on the flip-flop memory tasks.} Two off-tree sets are sampled for each tree-selected candidate, matching the parent and number of remaining ReLUs, with paired training random seeds. Comparisons use either no auxiliary guidance with gradual linearization over 200 optimization steps or full-preactivation guidance with abrupt linearization. Left: raw per-run memory-state recovery, requiring exactly $2^n$ stable zero-input fixed points with one within distance $0.3$ of each memory corner. Right: position- and type-verified recovery of the $3^n$-point reference grid, averaged equally over all produced models, including those that fail the memory-state recovery criterion. Error bars indicate 95\% confidence intervals.}
    \label{fig:appendix-mask-control-ff}
\end{figure*}

\subsection{Ablation study: Effects of guidance and gradual linearization}
\label{app:flipflop_ablation}

We compare retraining objectives and linearization schedules using the same selected reductions: 50 candidates for 2-bit memory and 32 for 3-bit memory, from cohorts of 30 initial parents with $P=10$. Parents are screened before retraining for accuracy of at least $0.98$ and a stable fixed point within distance $0.3$ of each memory corner, leaving 28 and 26 eligible parents, respectively. The screening and candidate sets are identical across guidance conditions, and training random seeds are paired.

For abrupt linearization, the selected ReLU activations are replaced by the identity before retraining. For gradual linearization, this transition takes place over the first 200 optimization steps. The task-specific objectives and weights are given in Appendix~\ref{app:flipflop_details}. Without auxiliary guidance, the reduced model minimizes the task loss alone. Parent and reduced models receive the same input sequence and initial state, without intermediate synchronization.

\Figref{fig:appendix-guidance-anneal-ff} separates memory-state recovery, reference-grid recovery, and task error. The top panel reports memory-state recovery, requiring exactly $2^n$ stable zero-input fixed points with one within distance $0.3$ of each memory corner, without a task-accuracy threshold. Rates are averaged first within each parent and then over all 30 initial parents, assigning zero to parents contributing no candidate. The middle panel evaluates position- and type-verified reference-grid recovery only among models satisfying this memory-state recovery criterion under gradual linearization, using parent-balanced candidate weights. The bottom panel reports task error over all produced models under gradual linearization.

Unlike the Chua tasks, changing from abrupt to gradual linearization has little effect on memory-state recovery at a fixed guidance objective. Without auxiliary guidance, recovery is approximately $91\%$ and $88\%$ for 2-bit memory under abrupt and gradual linearization, respectively, and $85\%$ under both schedules for 3-bit memory. Combined guidance achieves approximately $92\%$ and $78\%$ for the two tasks under either schedule. Symbol guidance gives lower memory-state recovery, at approximately $82\%$ and $67\%$.

The effects of guidance are clearer in the reference fixed-point geometry. Among models satisfying the memory-state recovery criterion under gradual linearization, symbol guidance yields the highest position- and type-verified recovery of the full reference grid: approximately $83\%$ for 2-bit memory and $62\%$ for 3-bit memory, compared with $63\%$ and $49\%$ without auxiliary guidance. Combined guidance reaches $75\%$ and $58\%$. These percentages include both memory corners and saddles; they are not saddle-only recovery rates.

Task error follows a different ordering. Retraining without auxiliary guidance gives the lowest median error, whereas symbol guidance gives the highest. Combined guidance lies between these conditions in both task error and reference-grid recovery. Thus, the principal benefit of symbol guidance in these experiments is improved recovery of the reference fixed-point arrangement, rather than higher memory-state recovery or lower task error.

\begin{figure*}[t]
    \centering
    \includegraphics[width=\textwidth]{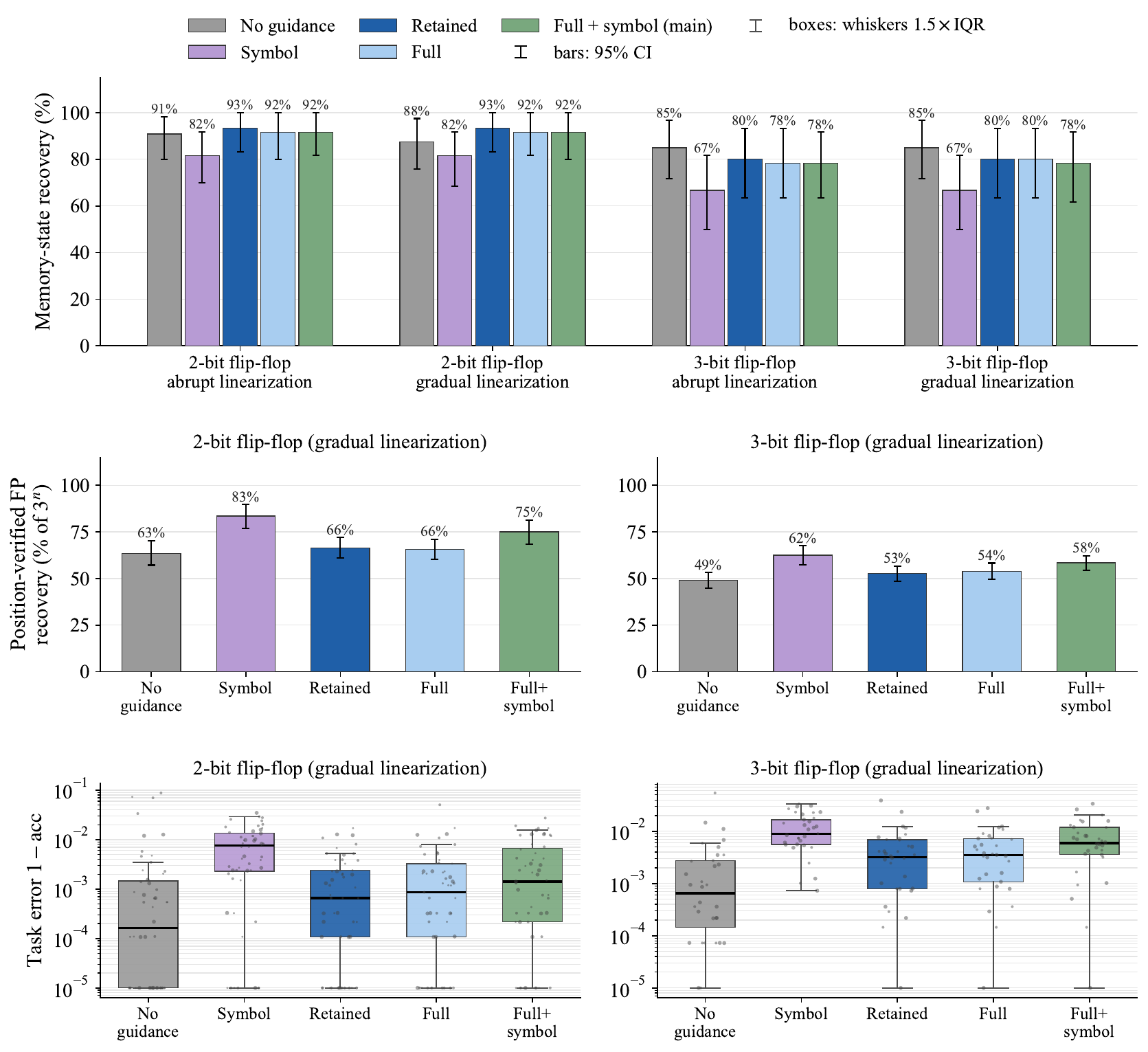}
    \caption{\textbf{Effects of guidance and gradual linearization on the flip-flop memory tasks.} Comparisons use 50 and 32 selected candidates for 2- and 3-bit memory from 28 and 26 eligible parents, respectively, out of 30 initial parents per task. Top: memory-state recovery under abrupt linearization and gradual linearization over 200 optimization steps, averaged within parents and over all 30 initial parents, assigning zero to parents without candidates. Middle: parent-balanced position- and type-verified recovery of the $3^n$-point reference grid among models satisfying the memory-state recovery criterion under gradual linearization; confidence intervals are obtained by bootstrapping parents while retaining this success filter. Bottom: task error $1-\mathrm{acc}$ over all produced models under gradual linearization. Error bars indicate 95\% confidence intervals; boxes show medians and interquartile ranges, with whiskers extending to the most extreme observations within $1.5$ interquartile ranges of the box. Objectives and weights are defined in Appendix~\ref{app:flipflop_details}.}
    \label{fig:appendix-guidance-anneal-ff}
\end{figure*}

\FloatBarrier
\clearpage

\section{Computational cost}
\label{app:computational_cost}

We examine computational cost on the four attractor-reconstruction tasks and the two flip-flop memory tasks. We first report representative stage-level runtimes, then estimate the cost of finding a minimal realization under the task-specific criteria below when the required nonlinear capacity is not supplied to the search procedure. Search costs are reconstructed from the evaluated cohorts and include the preparation required by each protocol.

\subsection{Stage-level runtimes}
\label{app:cost_stage_runtime}

For 4-scroll Chua, we measured wall-clock times on the same machine under sequential execution, with one experimental process at a time. Direct training uses $P=5$, while retraining starts from a $P=10$ parent and retains five ReLU units. Both use $M=20$, $2{,}000$ epochs with 50 optimization steps per epoch, batches of 16 segments of length 200, and checkpoint evaluation every 10 epochs. Retraining uses combined guidance with gradual linearization over the first 200 epochs.

The two direct-training runs took $5{,}826$ and $5{,}650$ s, while retraining took $7{,}304$ and $7{,}228$ s, approximately $25$--$28\%$ longer. Parent training at $P=10$ took $5{,}703$ s. Reduction took 87, 47, and 49 s for three parents, including trajectory generation, symbolic-graph construction, hierarchy exploration, and analysis-figure generation. Constructing a reusable teacher cache took approximately 34 s per parent; retraining times include cache loading but exclude cache construction.

For memory, we performed separate uncontended measurements after the original experiments. Two runs were measured at each of $P=1$, $P=10$, and the evaluated minimal capacity ($P=4$ for 2-bit memory and $P=6$ for 3-bit memory). Median parent-training times were $225.35$ and $218.05$ s for 2- and 3-bit memory, respectively. Median direct-training times at $P=4$ and $P=6$ were $210.8$ and $217.7$ s. Reduction-only runs took $10.4$ and $8.3$ s. Subtracting these from runs performing reduction followed by one candidate retraining gives incremental retraining estimates of $240.4$ and $248.4$ s, which we use in the cost accounting. Separate combined-guidance runs with gradual linearization over 200 optimization steps took $252.4$ and $254.9$ s as a cross-check. Memory retraining generates parent states from the same input sequences as the student and does not require a separately stored teacher-trajectory cache.

Reduction is small relative to training in these measurements. The retraining overhead reflects the complete protocol and implementation, rather than isolating the cost of auxiliary guidance. These are stage-level measurements, not timings of a single complete search; subsequent success-verification analyses are excluded.

\subsection{Search protocols}
\label{app:cost_search_protocols}

\Figref{fig:appendix-cost-protocols} compares three search procedures. Each direct-training capacity has 30 evaluated seeds, and each parent-based procedure starts from 30 parents with $P_{\mathrm{parent}}=10$. These cohort sizes follow the executed experiments.

\paragraph{Protocol D1: Direct training with a naive capacity sweep.}
Starting from $P=1$, Protocol D1 trains and evaluates models individually, stopping at the first success. If all 30 models at a capacity fail, the search proceeds to the next capacity.

\paragraph{Protocol D2: Direct training with a parent-informed capacity sweep.}
Protocol D2 first trains 30 parents and estimates a fixed-point count $\widehat Q$ from their most frequent count. For attractor reconstruction, each parent contributes the number of admissible fixed points whose symbolic regions are visited during autonomous evaluation. For memory, it contributes the number of analytically admissible, stable zero-input fixed points, irrespective of whether their regions are visited during an input-driven rollout. The resulting memory estimates are $\widehat Q=4$ and $8$, supported by 28 and 27 of the 30 parents, respectively.

The direct-training sweep begins at
\begin{equation}
\label{eq:cost_starting_capacity}
P_{\mathrm{start}}=\left\lceil\log_2\widehat Q\right\rceil,
\end{equation}
and otherwise follows Protocol D1. This is a necessary capacity lower bound for representing the inferred fixed-point count, not a guarantee of sufficient capacity. In particular, memory searches start at $P=2$ and $3$; these estimates account for stable memory states, not the full reference saddle structure. Parent-training costs are included. If the sweep reaches $P=10$, the already trained parents are reused without charging their training costs again.

\paragraph{Protocol O: Reduction-guided retraining.}
Protocol O trains the 30 parents, constructs their reduction hierarchies, and applies the parent screening and candidate-selection rules used in the main experiments. Selected candidates are considered in ascending order of their quotient-symbol count $|\Sigma_{\mathcal D}|$. Each candidate is initialized from its parent and retrained using the main protocol, stopping at the first success. The accounting includes preparation for all 30 parents, including those yielding no selected candidate.

\paragraph{Shared stopping criterion.}
For attractor reconstruction, success uses the joint fixed-point-count and visited-symbol criterion in Appendix~\ref{app:minimal_eval}. For memory, success requires memory-state recovery---exactly $2^n$ stable zero-input fixed points, with one within distance $0.3$ of each memory corner---at the smallest successful remaining capacity in the evaluated cohorts: $P_\star=4$ and $6$ for 2- and 3-bit memory. It does not require a prescribed visited-symbol count or recovery of every reference saddle. In particular, the observed 3-bit capacity $P_\star=6$ exceeds the reference-grid bound $\lceil\log_2 27\rceil=5$.

A shared external evaluator supplies this stopping verdict to all three protocols. The reference counts and, for memory, the reference capacity $P_\star$ are used only for this test, not to set the search's starting capacity or to select parents or reduction candidates. The comparison therefore concerns finding a model that meets an externally specified goal without supplying its required capacity to the search procedure.

\begin{figure*}[t]
    \centering
    \includegraphics[width=\textwidth]{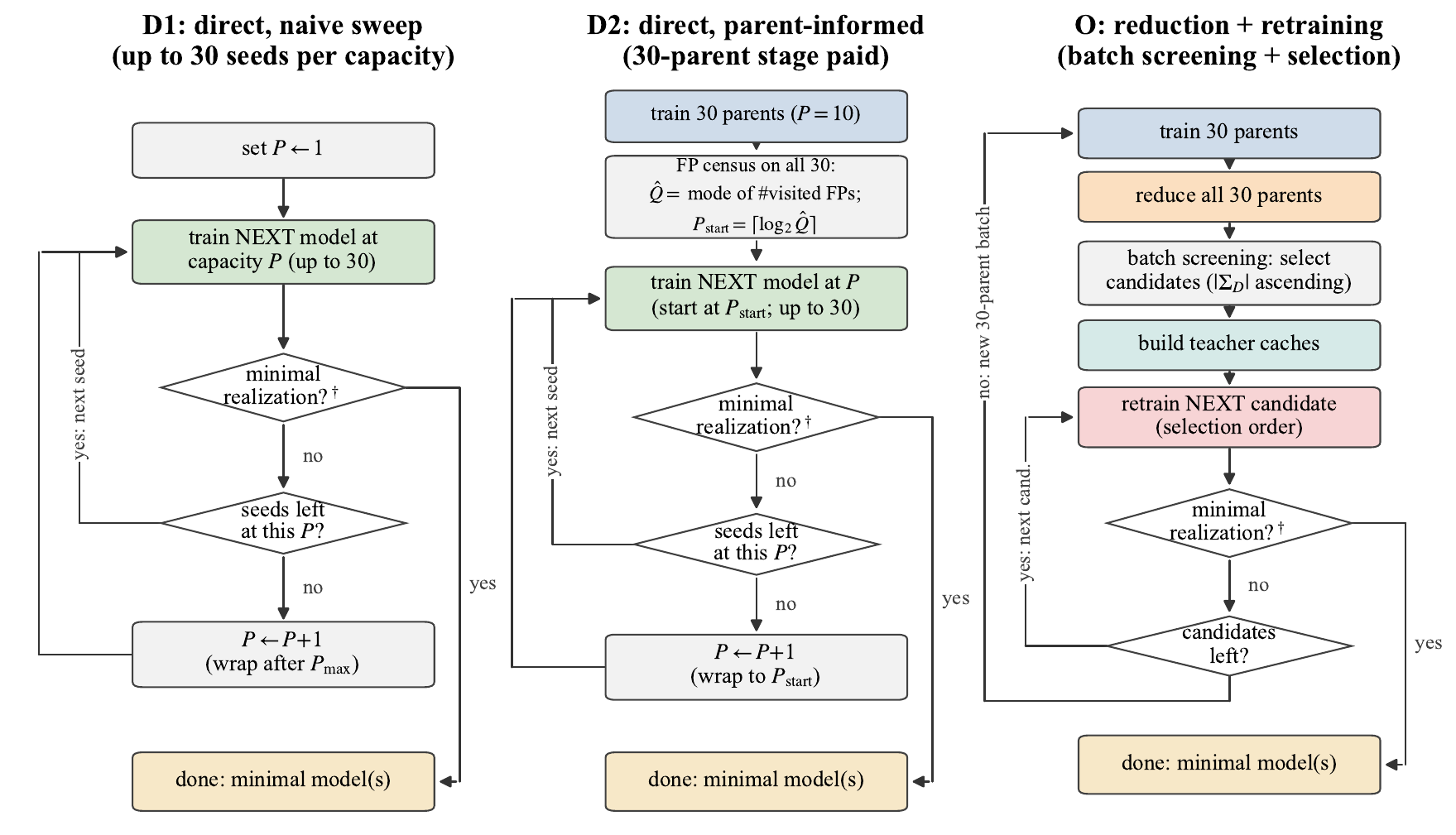}
    \caption{\textbf{Search protocols for obtaining a minimal model.} Protocol D1 searches capacities from $P=1$. Protocol D2 first trains 30 parents and estimates a starting capacity from their fixed-point counts. Protocol O reduces the trained parents and retrains selected candidates in ascending quotient-symbol count. Each procedure evaluates models individually and stops at the first success. The diagram shows the attractor-reconstruction implementation: for memory, the parent census counts stable zero-input fixed points rather than visited FP regions, and the teacher-cache construction step is omitted. $\dagger$ All protocols use the same task-specific external stopping test described in the text. Restart branches indicate possible continuation; the reported accounting uses only the evaluated cohorts and does not extrapolate to additional sweeps or parent batches.}
    \label{fig:appendix-cost-protocols}
\end{figure*}

\subsection{Cost accounting and conditional quantiles}
\label{app:cost_accounting}

Search costs combine the observed outcomes with estimated per-run durations. For attractor reconstruction, unit costs are the median wall-clock times of the corresponding production cohorts, which were run concurrently. Where matching direct-training timings are unavailable, we use the same system's parent-capacity training time as a proxy. For memory, all unit costs come from the separate uncontended measurements above, rather than recorded durations of the original cohorts. Direct-training costs at unmeasured capacities are imputed using the median of the measured per-capacity unit costs within that task. Costs are thus estimated aggregate run-hours, not elapsed calendar time or CPU-hours; protocol comparisons use the same timing basis within each task.

Protocol O is charged for training and reducing all 30 initial parents. For attractor reconstruction, one teacher cache is additionally charged for each parent contributing selected candidates, using an estimated construction time of 36 s; its cost is shared across that parent's candidates. Memory has no separate cache charge. All preparation is paid before retraining begins. The external success-verification analysis is excluded for every protocol.

The comparison stops at the first successful model rather than charging all selected candidates or all models at the successful capacity. Protocols D1 and D2 retain ascending capacity order but randomize the order of the 30 models within each capacity. Protocol O retains ascending quotient-symbol-count order but randomizes candidate order within each layer. All model outcomes remain fixed to their observed values.

Consider the first capacity or candidate layer containing a success. Let $n$ be its number of evaluated models, $s>0$ its number of successes, and $J$ the position of the first success under a uniformly random ordering without replacement. Then
\begin{equation}
\label{eq:cost_first_success_order}
\Pr(J\leq m)=1-\frac{\binom{n-s}{m}}{\binom{n}{m}},
\qquad
\mathbb E[J]=\frac{n+1}{s+1},
\end{equation}
where $\binom{n-s}{m}=0$ for $m>n-s$. Let $C_0$ include preparation and all preceding unsuccessful capacities or layers, and let $c$ be the unit cost in the first success-containing layer. The conditional expected cost and cost quantiles are
\begin{equation}
\label{eq:cost_expectation_quantiles}
\mathbb E[C]=C_0+c\frac{n+1}{s+1},
\qquad
B_q=C_0+c\min\{m:\Pr(J\leq m)\geq q\},
\quad q\in\{0.5,0.9\}.
\end{equation}

Thus, $B_{50}$ and $B_{90}$ describe variation in the order of trying the same evaluated models. They are not confidence intervals or success guarantees for new parents or training seeds. The calculation does not require independence between candidate outcomes. 
% If no evaluated model succeeds, we report the explored cost as a lower bound on the cost of finding a first success.
If no evaluated model succeeds, we report the estimated cost of exhausting the evaluated cohorts and indicate that no success was observed.

\subsection{Search-cost comparison}
\label{app:cost_results}

\Figref{fig:appendix-search-cost} shows conditional expected costs $\mathbb E[C]$ and their stage-wise contributions for all six tasks. Table~\ref{tab:conditional_search_cost} additionally reports $B_{50}$ and $B_{90}$. All parent ensembles use $P=10$, including 5-scroll Chua.

\begin{table*}[t]
    \centering
    \caption{\textbf{Conditional search costs in hours.} Entries report $\mathbb E[C]\,/\,B_{50}\,/\,B_{90}$ under random ordering within capacities for Protocols D1 and D2 and within quotient-symbol-count layers for Protocol O. Success uses the task-specific criteria in Appendix~\ref{app:cost_search_protocols}. Memory costs use post-hoc uncontended measurements; attractor costs use production-cohort unit times. 
    % For Dadras, the direct-training entries are explored-cost lower bounds, not estimated means or quantiles.
    For Dadras, the direct-training entries report the estimated costs of the evaluated searches with no observed success; no first-success mean or quantile is assigned.}
    \label{tab:conditional_search_cost}
    \begin{tabular}{lccc}
        \toprule
        Task & D1 & D2 & \textbf{O} \\
        \midrule
        Chua 3-scroll & $209.0\,/\,206.2\,/\,222.5$ & $110.9\,/\,108.0\,/\,124.4$ & $\mathbf{104.9\,/\,104.9\,/\,104.9}$ \\
        Chua 4-scroll & $428.6\,/\,426.9\,/\,467.3$ & $330.8\,/\,329.1\,/\,369.5$ & $\mathbf{94.1\,/\,93.8\,/\,93.8}$ \\
        Chua 5-scroll & $611.7\,/\,610.0\,/\,648.8$ & $429.3\,/\,427.6\,/\,466.4$ & $\mathbf{100.3\,/\,98.5\,/\,102.9}$ \\
        % Dadras & $>934.5$ & $>734.0$ & $\mathbf{102.5\,/\,100.9\,/\,104.2}$ \\
        Dadras
        & No success ($934.5$)
        & No success ($734.0$)
        & $\mathbf{102.5\,/\,100.9\,/\,104.2}$ \\
        \midrule
        2-bit flip-flop & $5.38\,/\,5.38\,/\,5.50$ & $5.51\,/\,5.51\,/\,5.63$ & $\mathbf{2.03\,/\,2.03\,/\,2.03}$ \\
        3-bit flip-flop & $9.25\,/\,9.23\,/\,9.47$ & $7.47\,/\,7.44\,/\,7.68$ & $\mathbf{1.96\,/\,1.96\,/\,1.96}$ \\
        \bottomrule
    \end{tabular}
\end{table*}

\begin{figure*}[t]
    \centering
    \includegraphics[width=\textwidth]{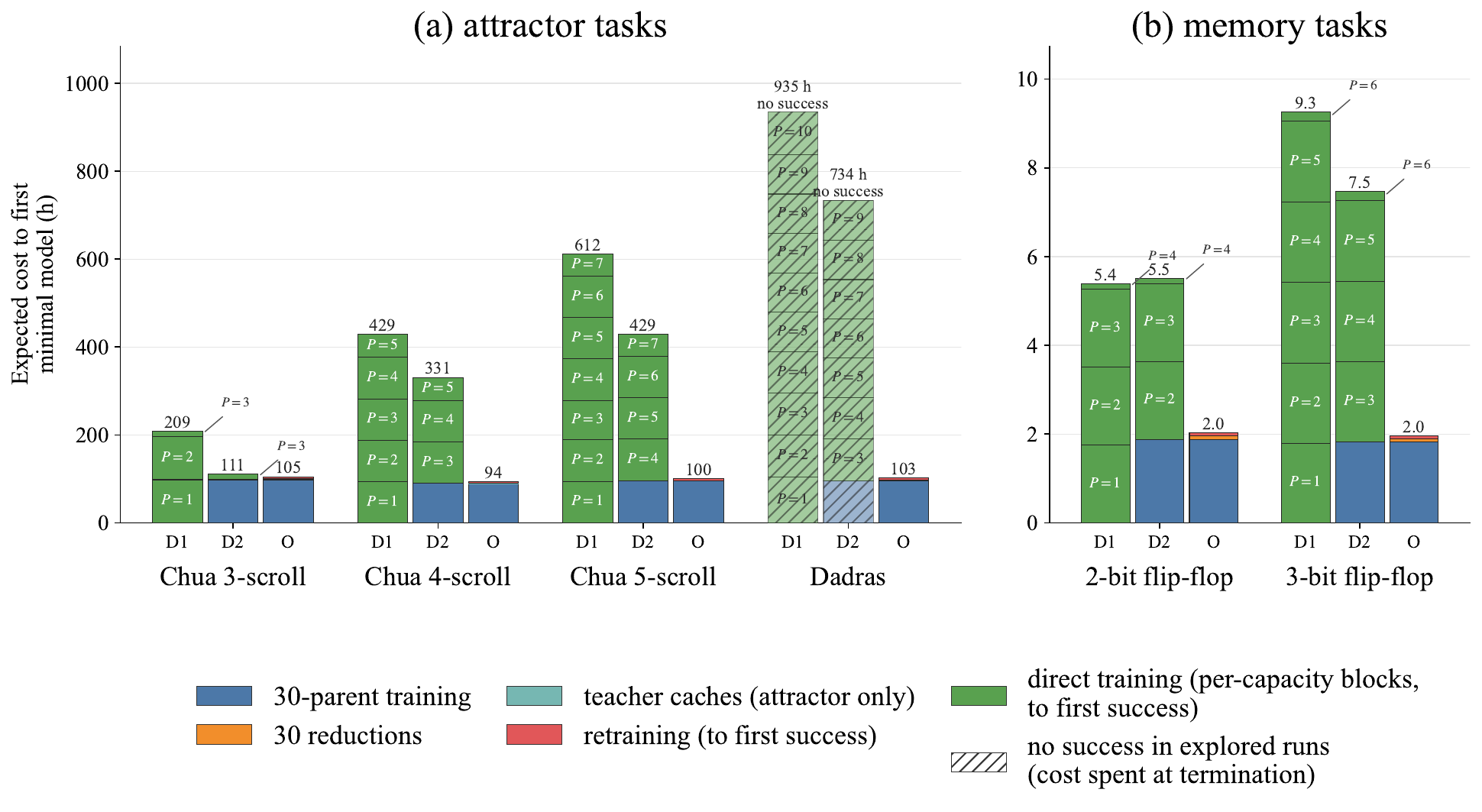}
    \caption{\textbf{Estimated cost of finding the first minimal model.} Solid bars show the conditional expected cost $\mathbb E[C]$ defined in \Eqref{eq:cost_expectation_quantiles}
 for (a) attractor-reconstruction and (b) memory tasks, using separate vertical scales and the task-specific success criteria in Appendix~\ref{app:cost_search_protocols}. Stacked components separate parent training, reduction, teacher-cache construction (attractor tasks only), retraining, and direct training. Unsuccessful direct-training capacities are charged in full; the first successful capacity is charged only through the expected first-success position. Protocol O includes preparation for all 30 parents and retraining through the expected first success. 
    % Hatched bars and arrows show explored-cost lower bounds when no success was observed.
    Hatched bars report the estimated costs of the evaluated direct searches for which no success was observed.
    Attractor costs use production-cohort timings; memory costs use separate uncontended measurements. Values are estimated aggregate run-hours and exclude final success-verification time.}
    \label{fig:appendix-search-cost}
\end{figure*}

For 3-scroll Chua, Protocols O and D2 have similar expected costs: approximately 105 and 111 h, respectively. The difference is larger for the more difficult Chua tasks: approximately 94 versus 331 h for 4-scroll Chua and 100 versus 429 h for 5-scroll Chua.
% For Dadras, Protocol O reaches a minimal realization at an expected cost of approximately 103 h, whereas neither direct procedure succeeds within the evaluated capacity range. The explored costs are approximately 734 h for Protocol D2 and 935 h for Protocol D1.
For Dadras, Protocol O finds a minimal realization at an estimated expected cost of approximately 103 h. Neither direct protocol finds one in the evaluated capacity sweeps, whose estimated costs are approximately 734 h for Protocol D2 and 935 h for Protocol D1.

The memory tasks show the same benefit of reduction-guided search. Expected costs for Protocol O are $2.03$ and $1.96$ h for 2- and 3-bit memory, compared with $5.51$ and $7.47$ h for Protocol D2 and $5.38$ and $9.25$ h for Protocol D1. The lowest quotient-symbol-count layers contain seven successful candidates out of seven at $|\Sigma_{\mathcal D}|=9$ for 2-bit memory, and one successful candidate out of one at $|\Sigma_{\mathcal D}|=29$ for 3-bit memory. Consequently, the first candidate succeeds under every within-layer ordering of these observed outcomes, so $\mathbb E[C]$, $B_{50}$, and $B_{90}$ coincide at the preparation cost plus one retraining run.

% Most of Protocol O's expected cost is incurred in learning the parent ensemble. Once candidates are available, few retraining attempts are needed in these cohorts, whereas direct search may spend substantial computation at several capacities before reaching a success. Within each task family, Protocol O therefore shows less variation in expected search cost across the evaluated tasks. These results show how the cost of learning and reducing a parent ensemble can be offset by avoiding unsuccessful direct-training attempts at multiple capacities.
Protocols D1 and D2 incur the cost of training models at successive capacities. The parent-derived fixed-point count in Protocol D2 supplies a capacity lower bound, but direct training first succeeds above this bound on 4- and 5-scroll Chua and both memory tasks. Protocol O uses the learned parent dynamics to construct reduced candidates and retrains them in ascending quotient-symbol count, avoiding a separate capacity sweep. Although parent training accounts for most of Protocol O's expected cost, few subsequent retraining attempts are needed in the evaluated cohorts. The cost advantage therefore comes from using the learned dynamics to identify reduced models to retrain, avoiding unsuccessful training at multiple capacities.

\FloatBarrier
\clearpage

\section{Robustness of the state-space error}
\label{app:histogram_robustness}

We examine how histogram boundary handling and the exclusion of
out-of-range generated samples affect the state-space error.
The analysis separates corrections to the bin boundaries from
explicit accounting for generated probability mass outside the
target histogram range.

\paragraph{Histogram corrections.}
We distinguish three versions of the error.
The original implementation, denoted by $E_{\mathrm{legacy}}$,
excludes generated samples outside the target histogram range and
normalizes the retained histogram.
An implementation audit additionally identified that its
bin-indexing rule excludes the lowest-index bin along each
coordinate and samples at the upper boundary, affecting both
target and generated samples.
For the 3-scroll Chua reference trajectory, this boundary
convention excludes approximately $2.7\%$ of the samples.
This fraction is distinct from the generated probability mass
outside the target range.

We first correct the boundary handling while retaining the
original histogram resolution and smoothing parameter, obtaining
$E_{\mathrm{in}}$.
We then account explicitly for the excluded generated mass.
Let $f_{\mathrm{out}}$ denote the fraction of the $9{,}000$
post-transient generated samples outside the target histogram
range.
Rescaling the conditional in-range generated distribution by
$1-f_{\mathrm{out}}$ gives
\begin{equation}
    E_{\mathrm{mass}}
    =
    E_{\mathrm{in}}
    -
    \log(1-f_{\mathrm{out}}).
    \label{eq:histmass-correction}
\end{equation}
We set $E_{\mathrm{mass}}=+\infty$ when
$f_{\mathrm{out}}=1$.
This correction accounts for excluded probability mass without
expanding the target histogram range.

\paragraph{Effect on the reported errors.}
The audit reproduces the recorded original errors within
$2\times10^{-6}$ across $2{,}040$ evaluated models.
\Figref{fig:histogram_robustness}(a) reports the true out-of-range
fractions for parents, retrained models, and directly trained
models across the four attractor-reconstruction tasks.
Panel~(b) compares the original and mass-corrected errors.

For the 3-scroll retrained cohort, correcting the bin boundaries
increases the median error by approximately $0.014$, and accounting
for the out-of-range mass adds approximately $0.007$ further.
The corrections therefore have a small effect on the median error
in this cohort, while making the contribution of excluded
generated mass explicit.

\begin{figure*}[p]
    \centering
    \includegraphics[width=0.6\textwidth]{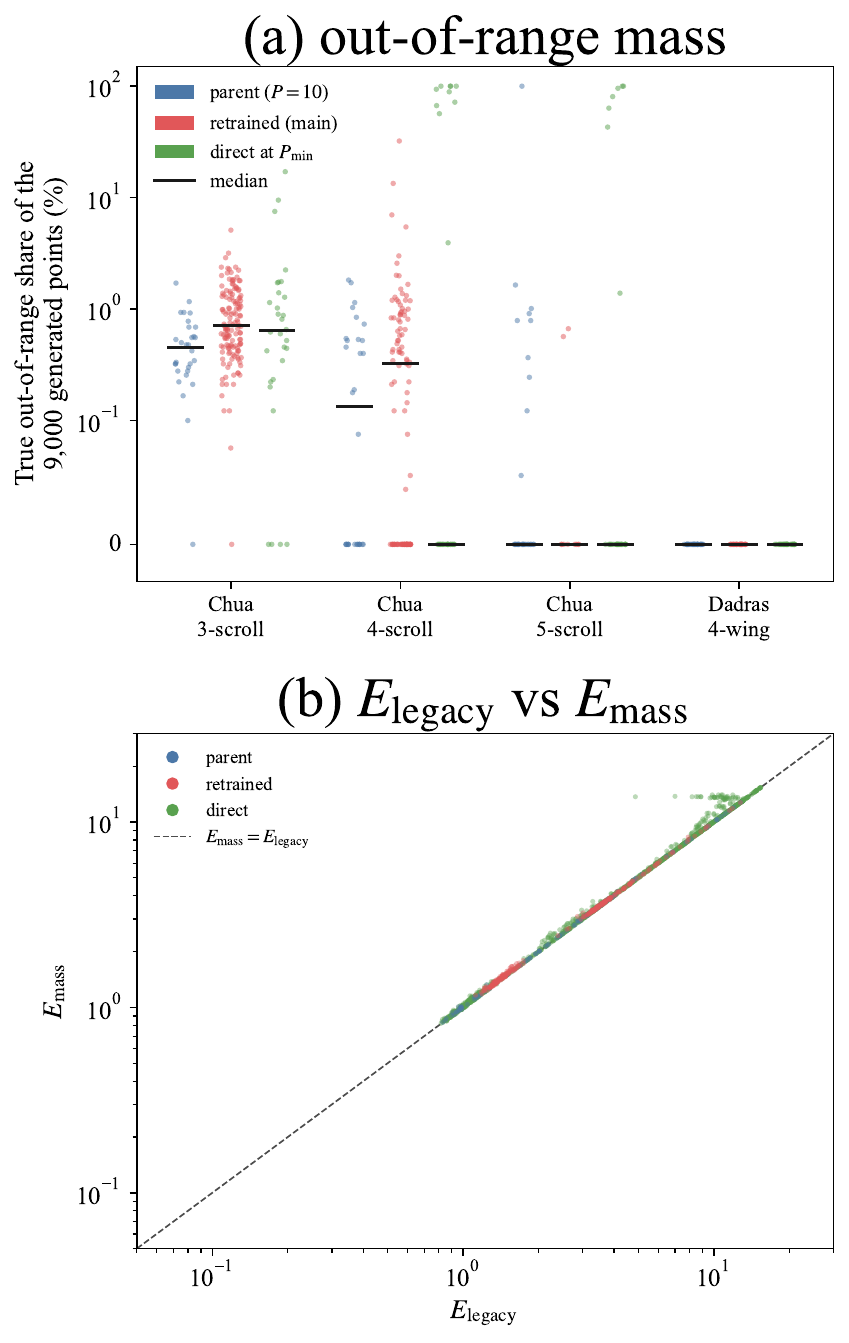}
    \caption{\textbf{Effect of histogram boundary handling and
    out-of-range probability mass on the state-space error.}
    (a) Fraction of the $9{,}000$ post-transient generated samples
    outside the target histogram range for parents, retrained
    models, and direct models at $P_{\min}$ on the four
    attractor-reconstruction tasks. Black lines mark medians.
    These fractions measure true out-of-range samples, separately
    from samples omitted by the original bin-boundary convention.
    (b) Original error $E_{\mathrm{legacy}}$ versus
    $E_{\mathrm{mass}}$, which corrects the bin boundaries and
    accounts for out-of-range probability mass.
    The dashed line indicates equality.}
    \label{fig:histogram_robustness}
\end{figure*}

\end{document}